\documentclass{article}

\usepackage[numbers,sort&compress]{natbib}
\setcitestyle{square,comma}

\usepackage[main, final]{neurips_2026}

\usepackage[utf8]{inputenc} 
\usepackage[T1]{fontenc}    
\usepackage{hyperref}       
\usepackage{url}            
\usepackage{booktabs}       
\usepackage{amsfonts}       
\usepackage{nicefrac}       
\usepackage{microtype}      
\usepackage{xcolor}         
\usepackage{wrapfig}

\usepackage{graphicx}
\usepackage{subcaption}
\usepackage{multirow}
\usepackage{amsmath}
\usepackage{amssymb}
\usepackage{mathtools}
\usepackage{amsthm}
\usepackage{enumitem}
\usepackage{thm-restate}
\usepackage{cleveref}
\usepackage{algorithm}
\usepackage{algorithmic}
\usepackage{soul}
\usepackage{titletoc}

\hypersetup{
  colorlinks=true,
  citecolor=blue,
  linkcolor=blue,
  urlcolor=blue
}

\theoremstyle{plain}
\newtheorem{theorem}{Theorem}[section]
\newtheorem{proposition}[theorem]{Proposition}
\newtheorem{lemma}[theorem]{Lemma}

\theoremstyle{definition}
\newtheorem{definition}[theorem]{Definition}
\newtheorem{assumption}[theorem]{Assumption}
\theoremstyle{remark}

\title{Stability and Diversity of Networked Self-Consuming Generative Ecosystems}

\author{%
Xiukun Wei$^1$, \quad Yang Zhang$^1$, \quad Xueru Zhang$^1$\\
$^1$The Ohio State University\\
\texttt{\{wei.1418, zhang.15751, zhang.12807\}@osu.edu}\\
}
\begin{document}

\maketitle

\begin{abstract}
    The widespread deployment of generative AI has made it increasingly difficult to distinguish synthetic content from real data. Consequently, synthetic data is inevitably incorporated into the training pipelines of future model generations, forming a \textit{self-consuming training loop}. Prior work has studied the effects of such recursive self-consuming training, but analyses have largely been limited to \textit{isolated} models, where a model consumes only its own synthetic data, or to simplified interactions between two models. 
    This paper takes a first step toward understanding \textit{networked} self-consuming generative models, in which multiple models consume synthetic data generated by one another through complex interaction pathways. We introduce a theoretical framework representing models as nodes in a directed, weighted graph, with edge weights governing the flow of synthetic data among models. Using this framework, we analyze the long-term behavior of networked models under retraining dynamics, establishing conditions for convergence and characterizing the resulting fixed points. We further investigate how the system's long-term stability and diversity are shaped by each model’s access to real data, cross-model data consumption, and the structure of the interaction graph.
\end{abstract}

\section{Introduction}
The proliferation of generative models has fundamentally transformed the nature of data used for model training. As synthetic data increasingly permeates public repositories, online platforms, and downstream training pipelines \citep{SyntheticData}, maintaining datasets that are entirely real or purely human-generated is becoming increasingly challenging. This has motivated a rapidly growing line of research on \textit{self-consuming models}, in which models are iteratively updated using mixtures of real and synthetic data generated by the models themselves. Prior work has shown that models trained under such self-consuming loops often exhibit distributional drift~\citep{AICollapse,TheoreticalPerspective,LLMDistributionShift,StatisticalAnalysisLM}, collapse~\citep{Satbility,CollapseInevitable,GoMad}, loss of tail coverage~\citep{Tails,CurseofRecursion}, and  potentially bias amplification~\citep{amplification,FairnessFeedback,wang2026observations,xie2024automating}.

\begin{figure}[t]
\centering
\includegraphics[width=\textwidth]{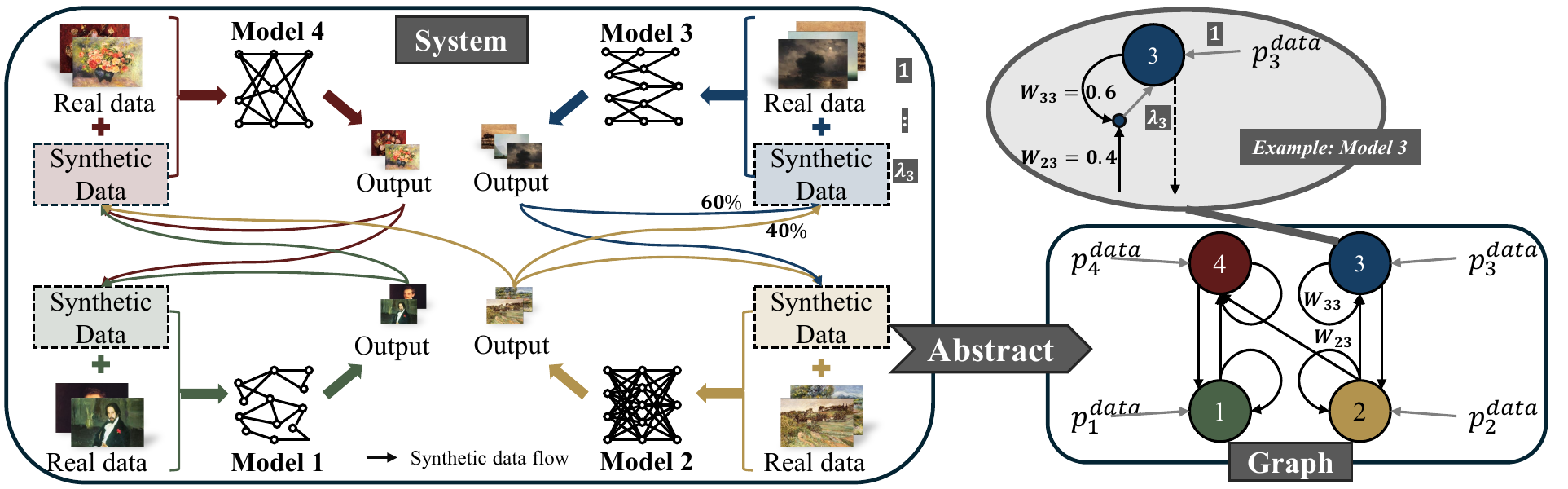}
\caption{From networked self-consuming models to graph abstraction. (1) Left: A generative ecosystem in which each model is trained on its own real-data distribution and a weighted mixture of synthetic data produced by other models. (2) Right: Abstraction of the system as a directed, weighted interaction graph that captures effective synthetic-data propagation among models. (3) Top right:  An example local retraining rule, where \textit{Model~$3$} is trained at iteration $t$ on
$p^{\mathrm{data}}_3 + \lambda_3 \sum_{j=1}^K w_{3j}p_{\boldsymbol{\theta}_j^{t-1}}$.}
\vskip -0.08in
\label{Fig:framework}
\end{figure}

However, existing work on self-consuming models has largely focused on \textbf{isolated} models, where a single model is trained independently and consumes only synthetic data generated by itself. In practice, real-world generative ecosystems are inherently multi-model and interactive: online platforms are populated with outputs from many models, and downstream training pipelines inevitably incorporate synthetic samples produced by peer models~\citep{AreWe}. For example, in the domain of large language models (LLMs), \textit{Orca} is fine-tuned using large-scale synthetic instruction-response pairs generated by ChatGPT and GPT-4~\citep{Orca}. Similarly, \textit{Phi-3} relies heavily on synthetic data generated by larger LLMs to achieve strong performance despite its smaller parameter scale~\citep{Phi3}. This multi-model interaction is also central to modern text-to-image systems: \textit{DALL·E  3} employs a dedicated captioning model to rewrite noisy prompts into high-quality synthetic captions, improving text-image alignment during training~\citep{Dalle3}; \textit{PixArt-$\alpha$} adopts a similar strategy by utilizing LLaVA to refine its training data~\citep{Pixart}.

Although a few recent works have recognized cross-model interactions in self-consuming training~\citep{CoEvolving,Whathappens,WhenandHow}, showing that such interactions can accelerate homogenization or model collapse, these studies are still highly limited. Most are restricted to \textit{two-model} coupled systems~\citep{CoEvolving,WhenandHow} or rely on highly \textit{simplified interactions} and \textit{linearized dynamics assumptions}~\citep{Whathappens}, leaving the behavior of large, heterogeneous networks of interacting self-consuming models largely unexplored.

To fully understand real-world generative ecosystems, it is necessary to move beyond these simplified settings and explicitly account for the complex, networked interactions among many models. Unfortunately, existing analyses fail to explicitly capture the \textit{networked, directed, and weighted} flows of synthetic data among heterogeneous models, nor do they characterize the asymptotic impact of graph structure on stability and diversity. As a result, a general theoretical framework is still lacking, one that (i) accommodates arbitrary interaction graphs, 
(ii) enables systematic analysis of how interaction topology and synthetic-data mixing jointly affect system stability, and (iii) enables quantification of how network connectivity and mixing strength influence system diversity.

In this paper, we initiate a theoretical study of \textbf{networked self-consuming models} by modeling the system as a directed interaction graph (Fig.~\ref{Fig:framework}). Each node represents a generative model, and weighted edges capture the probability that synthetic data produced by one model is consumed by another. Within this framework, we analyze the behavior of networked self-consuming models by characterizing (i) the evolution of networked models under self-consuming iterative training; (ii) the long-term stability of retraining dynamics; (iii) the effects of each model’s access to real data, cross-model data consumption, and graph topology on the system’s long-term stability and diversity.  

Specifically, we show that under standard regularity conditions for iterative retraining, the retraining dynamics can be locally analyzed by the Jacobian around any fixed point of the system. Assuming such a fixed point exists, we derive sufficient conditions for the stability and convergence of the fixed points that separate \textit{local amplification factors}—determined by each model’s synthetic-data mixing weight and local curvature—from \textit{global propagation effects} governed by the interaction graph. This 
condition strictly generalizes the stability threshold of prior single-model analyses~\citep{Satbility, curation}, accommodating arbitrary directed and weighted graphs.

We also study how the system’s long-term diversity is affected by cross-model data consumption, especially when models have access to varying real data distributions. We first derive an upper bound on system diversity at the fixed point, showing that synthetic-data consumption among models induces a \textit{topology-dependent smoothing effect} that contracts inter-model heterogeneity.  More related work is discussed in Appx.~\ref{sec:relatedwork}, and our contributions are summarized below.

\begin{itemize}[leftmargin=*,topsep=0cm,itemsep=-0.0cm]
\item 
    \textbf{A unified theoretical framework for networked self-consuming models}. We introduce a framework that captures directed, weighted interactions among heterogeneous models with arbitrary cross-model synthetic-data flows (Sec.~\ref{sec:problem}).
    
    \item \textbf{Graph-structured convergence \& stability analysis.} We derive explicit conditions for convergence and stability of the networked self-consuming models, showing that how the synthetic-data mixing weight and the interaction graph topology jointly affect long-term stability (Sec.~\ref{sec:HomoData}).

    \item \textbf{Topology-dependent diversity analysis.} We characterize the long-term diversity at the fixed point and show that synthetic-data consumption induces a topology-dependent smoothing effect that contracts inter-model heterogeneity (Sec.~\ref{sec:HeteroData}).

    \item \textbf{Empirical validation.} We validate our theoretical predictions on a synthetic 8-Gaussian dataset and a real image dataset CIFAR-10 \footnote{Please see our code at \url{https://github.com/osu-srml/Gen-System}}. Across multiple generative model classes  and a range of interaction topologies, the experimental observations align with our analysis (Sec.~\ref{sec:exp}).
\end{itemize}

\section{Problem Formulation}
\label{sec:problem}
Consider a set $\mathcal{V} = \{1, \dots, K\}$ of generative models. Each model $i\in\mathcal{V}$ has access to a real data distribution $p_{i}^{\mathrm{data}}(x)$ over a shared data space $\mathcal{X}$ and is parameterized by $\boldsymbol{\theta}_i \in \Theta_i$, inducing a generated data distribution
$p_{\boldsymbol{\theta}_i}$. We allow heterogeneous model architectures by permitting $\Theta_i\neq\Theta_j$ for $i\neq j$ and heterogeneous real data where $p_{i}^{\mathrm{data}}$ varies across models.

\textbf{Data flow among networked models.} If models are trained in isolation using only their own real data sources, each model $i$ converges to its own maximum likelihood estimator (MLE), denoted by $ {\boldsymbol{\theta}}_i^{\star}:= \arg\max_{\boldsymbol{\theta}} \mathbb{E}_{x \sim p^\mathrm{data}_{i}} [\log p_{\boldsymbol{\theta}}(x)]$. We refer to the collection of these local optima, $\Theta^{\star} = ({\boldsymbol{\theta}}_1^{\star}, \dots, {\boldsymbol{\theta}}_K^{\star})$, as the \emph{uncoupled solution}.

However, as synthetic data proliferates and becomes difficult to distinguish from real data, it is often inevitably used to train future generations of models, including both the originating model itself and other models. We formalize such interactions using a directed, weighted graph $(\mathcal{V},\mathcal{E}, \mathbf{W})$, where each node in $\mathcal{V}$ represents a generative model. A directed edge $j \to i\in \mathcal{E}$ indicates that synthetic data generated by model 
$j$ may be used to retrain model $i$. We refer to models with edges pointing to 
$i$ as its \textit{parents}. The adjacency matrix $\mathbf{W} = [w_{ij}] \in \mathbb{R}^{K \times K}$ assigns a weight $w_{ij} \ge 0$ to each edge $j\to i$, representing the probability that any given synthetic sample consumed by model $i$ during retraining is generated by model $j$. Consequently, $\mathbf{W}$ is \emph{row-stochastic}, i.e., $\sum_{j=1}^K w_{ij} = 1$ for each $i \in \mathcal{V}$. \emph{Self-loops} ($w_{ii}\ge 0$) are allowed, meaning that a model may reuse its own generated data.

\textbf{Iterative training loop.} As the environment evolves, models are updated repeatedly to maintain high performance. We formalize this iterative training process as follows. At $t=0$, each model is initialized from a pre-trained state obtained on its own real data $p_i^{\mathrm{data}}$. In subsequent rounds ($t\geq 1$), each model $i$ updates its parameters ${\boldsymbol{\theta}}_i^{t}$ using a mixture of its real data and synthetic data generated by itself and its parents in the previous round, i.e.,
\begin{equation}
\label{eq:update_rule}
\boldsymbol{\theta}_i^{t}
\in ~\arg\max_{\boldsymbol{\theta} \in \Theta_i}~\mathbb{E}_{x \sim p_{i}^{\mathrm{data}}} [\log p_{\boldsymbol{\theta}}(x)] + \lambda_i \sum_{j=1}^K w_{ij} \mathbb{E}_{x \sim p_{\boldsymbol{\theta}_j^{t-1}}}[\log p_{\boldsymbol{\theta}}(x)] 
\end{equation}
where $\lambda_i \ge 0$ controls the relative contribution of synthetic data.\footnote{When the maximizer is not unique, we select the solution from the set of maximizers $\widehat{\Theta}_i^{t}$ that is closest to the previous iterate, i.e., $\boldsymbol{\theta}_i^{t}:=\arg\min_{\boldsymbol{\theta}' \in {\widehat{\Theta}}_i^{t}} \| \boldsymbol{\theta}' - \boldsymbol{\theta}_i^{t-1}\|$, reflecting common practice in which models are \textit{fine-tuned} from their prior state using gradient-based training paradigms rather than retrained from scratch.} Stacking the parameters across all models, we denote the global parameter vector at round $t$ as  $ \Theta^{t}:= (\boldsymbol{\theta}_i^{t}, \dots, \boldsymbol{\theta}_K^{t})\in \prod_{i=1}^K \Theta_i$. We then define the global iteration map  $\mathcal{G}:\prod_{i=1}^K \Theta_i \to \prod_{i=1}^K \Theta_i$ by
\begin{equation}\label{eq:dynamics}
    \Theta^{t+1} = \mathcal{G}(\Theta^t) := \left(\mathcal{G}_1(\Theta^t), \dots, \mathcal{G}_K(\Theta^t)\right)
\end{equation}
where $\mathcal{G}_i$ is the update operator for model $i$ as in Eq.~\eqref{eq:update_rule}.

\textbf{Objectives.} Under the dynamics~\eqref{eq:dynamics}, the network of generative models evolves over time, yet its long-term behavior remains poorly understood. While prior work has investigated the effects of self-consuming retraining paradigms, existing studies are largely limited to a single model in isolation \cite{Satbility,GoMad,AICollapse} or to two-model settings with highly simplified interaction structures \cite{CoEvolving}. In contrast, real-world deployments involve multiple generative models interacting through complex and heterogeneous pathways, in which synthetic data produced by one model can propagate throughout the ecosystem and influence the training and behavior of many others.

This motivates a system-level investigation of how synthetic data and self-consuming iterative training loops shape the long-term dynamics of a network of generative models. In this paper, we seek to address the following fundamental questions regarding the asymptotic behavior of the network:
\begin{itemize}[leftmargin=*,topsep=-0.cm,itemsep=-0.0cm]
    \item How do networked generative models evolve under self-consuming iterative training loops, and under what conditions do the dynamics converge and admit fixed points?

    \item How do structural properties of the interaction graph $\mathbf{W}$ (e.g., connectivity, length and configuration of feedback cycles) and the strength of synthetic-data contributions $\{\lambda_i\}_{i=1}^K$ affect the long-term stability and diversity of individual models and of the network as a whole?  

    \item How does the diversity of model outputs evolve under cross-model interactions, and to what extent is it preserved at equilibrium?
\end{itemize}

Sec.~\ref{sec:HomoData} addresses the first two questions through a local stability analysis around an arbitrary fixed point. Sec.~\ref{sec:HeteroData} addresses the third question by analyzing how synthetic data interactions and graph topology shape inter-model diversity at the system's fixed point.
\section{Convergence \& Stability Analysis}
\label{sec:HomoData}
\subsection{Fixed Point and Technical Assumptions}
We begin by formalizing the notion of \textit{fixed point}, a state in which the models reach equilibrium and remain unchanged under iterative training loop.
\begin{definition}[Fixed point]The global parameter vector $\bar{\Theta}$ is a fixed point of the dynamics~\eqref{eq:dynamics} if $\bar{\Theta} = \mathcal{G}(\bar{\Theta})$.\end{definition}

As shown in Appx.~\ref{subsec:profs_chara_fp}, any fixed point $\bar\Theta$ satisfies the first-order optimality condition:
\begin{equation}\label{eq:f}
    \nabla_{\boldsymbol{\theta}} \mathbb{E}_{x \sim p_i^{\mathrm{data}}} [\log p_{\bar{\boldsymbol{\theta}}_i}(x)]
    + \lambda_i \sum_{j=1}^K w_{ij} \nabla_{\boldsymbol{\theta}} \mathbb{E}_{x \sim p_{\bar{\boldsymbol{\theta}}_j}} [\log p_{\bar{\boldsymbol{\theta}}_i}(x)]
    = 0, \forall i \in \mathcal{V}
\end{equation}

In the single-model setting ($|\mathcal{V}|=1$), the MLE $\boldsymbol{\theta}_i^\star$ satisfies Eq.~\eqref{eq:f} under Assumptions~\ref{asp:strong_concave}-\ref{asp:hessian_stability} and is therefore a fixed point \citep{Satbility}. In the networked multi-model setting, however, the uncoupled solution $\Theta^{\star} = ({\boldsymbol{\theta}}_1^{\star}, \dots, {\boldsymbol{\theta}}_K^{\star})$ is generally \emph{not} a fixed point of $\mathcal{G}$; Appx.~\ref{subsec:example} gives a two-model example exhibiting a fixed point distinct from $\Theta^\star$. Nevertheless, we can identify conditions under which $\Theta^\star$ remains a fixed point in multi-model systems, as established in Prop.~\ref{prop} (see Appx.~\ref{subsec:profs_uncoupled}).

\begin{proposition}[Uncoupled solution as a fixed point]\label{prop}
 When \emph{(i)} the real data distributions are homogeneous ($p_i^{\mathrm{data}} \equiv p^{\mathrm{data}}$ for all $i$) and \emph{(ii)} the per-model MLEs induce a common generative distribution  ($p_{\boldsymbol{\theta}_i^\star} \equiv p^\star$ for all $i \in \mathcal{V}$, for some $p^\star$), uncoupled solution $\Theta^\star$ is a fixed point. 
\end{proposition}

In general multi-model settings, establishing the existence of fixed points satisfying Eq.~\eqref{eq:f} and characterizing them is substantially more challenging than in the single-model case, due to the high-dimensional coupled structure of $\mathcal{G}$ and the flexibility of the interaction graph $\mathbf{W}$. Their existence typically requires global conditions ensuring that the retraining map is a continuous self-map on a compact parameter space. Instead of pursuing such a global fixed-point analysis, we study the local dynamics around an arbitrary fixed point $\bar{\Theta}$ whenever one exists. In particular, we aim to characterize the stability of a given equilibrium and how this stability is influenced by synthetic-data mixing and the interaction graph.

We now introduce the assumptions imposed at every fixed point $\bar\Theta$ under consideration. Let $U_i$ denote a neighborhood of $\bar{\boldsymbol{\theta}}_i$, and $\mathcal{U} := \prod_i U_i$.
These assumptions are standard and follow those of \citep{Satbility}.
\begin{assumption}[Local strong concavity]
\label{asp:strong_concave}
For each model $i$, the map $\boldsymbol{\theta} \mapsto \mathbb{E}_{x \sim p_i^{\mathrm{data}}}[\log p_{\boldsymbol{\theta}}(x)]$ is twice continuously differentiable on $U_i$, and there exists $\alpha_i > 0$ such that $-\nabla_{\boldsymbol{\theta}}^2 
\mathbb{E}_{x \sim p_i^{\mathrm{data}}}[\log p_{\boldsymbol{\theta}}(x)] 
\succeq \alpha_i I_{d_i}$ for all $\boldsymbol{\theta} \in U_i$, where $d_i = \dim(\Theta_i)$.
\end{assumption}
\begin{assumption}[Hessian stability under distribution shift]
\label{asp:hessian_stability}
Denote $d_{\mathcal{W}}(\cdot,\cdot)$ as  1-Wasserstein distance \citep{Wasserstein}. For each model $i$, there exists $L_i \ge 0$ such that for all $\boldsymbol{\theta} \in U_i$ and any probability measures $q, q'$ over $\mathcal{X}$,
$\left\|
\mathbb{E}_{x\sim q}\nabla_{\boldsymbol{\theta}}^2\log p_{\boldsymbol{\theta}}(x)
-
\mathbb{E}_{x\sim q'}\nabla_{\boldsymbol{\theta}}^2\log p_{\boldsymbol{\theta}}(x)
\right\| \le L_i d_{\mathcal{W}}(q,q')$.
\end{assumption}

The above assumptions characterize the local regularity of each model around the fixed point $\bar{\Theta}$. To analyze the networked dynamics, we further introduce two quantities.

\begin{definition}[Distributional gap]
\label{def:eps}
For each $i \in \mathcal{V}$, the \emph{distributional gap} between model $i$'s real data distribution and its parent output distribution at $\bar\Theta$ is
$
\varepsilon_i := \max_{j: w_{ij} > 0} 
d_W\bigl(p_{\bar{\boldsymbol{\theta}}_j}, p_i^{\mathrm{data}}\bigr).
$
\end{definition}
\begin{definition}[Local amplification factor]
\label{def:beta}
The \emph{local amplification factor} at $\bar\Theta$ for model 
$i \in \mathcal{V}$ is
\[
\beta_i := \frac{\lambda_i\, M_i}{\alpha_i + \lambda_i(\alpha_i - L_i \varepsilon_i)}.
\]
where 
$
M_i
:=
\max_{j:w_{ij}>0}
\left\|
\mathbb{E}_{x\sim p_{\bar{\boldsymbol{\theta}}_j}}
\big[
\nabla_{\boldsymbol{\theta}} \log p_{\bar{\boldsymbol{\theta}}_i}(x)\,
\nabla_{\boldsymbol{\theta}_j}\log p_{\bar{\boldsymbol{\theta}}_j}(x)^\top
\big]
\right\|_2
$
captures the strongest cross-model sensitivity between the model $i$ and its parent models at the fixed point.
\end{definition}
Throughout the paper, $\|\cdot\|$ denotes the Euclidean norm for vectors, and $\|\cdot\|_2$ denotes the induced operator (spectral) norm for matrices. Local amplification factor $\beta_i$ quantifies how strongly a perturbation arriving at model $i$ through synthetic data is amplified by its local retraining dynamics. Stacking these factors with the interaction graph, we define the nonnegative \emph{comparison matrix} 
$$\mathbf{D} := \mathrm{diag}(\beta_1, \dots, \beta_K)\mathbf{W}.$$

\subsection{Convergence of networked self-consuming models.}
Next, we analyze the evolution of networked models and identify conditions under which the dynamics converge. Prior studies have shown that a \textit{single} self-consuming model can fail to converge when the synthetic-data contribution $\lambda_i$ is sufficiently large \citep{Satbility}. When multiple models coexist and interact over a network, analyzing the stability of the entire system becomes substantially more challenging. In this case, stability is no longer solely a local property determined by $\lambda_i$, rather, it emerges as a \textit{system-level} property that also depends on the structure of the interaction network encoded by $\mathbf W$.

\begin{restatable}[Local asymptotic stability]{theorem}{Stability}
\label{thm:stability}
Suppose $\alpha_i > L_i \varepsilon_i$, $\forall i\in\mathcal{V}$. If $\rho(\mathbf{D}) < 1$, then $\bar\Theta$ is locally asymptotically stable. Moreover, there exists a neighborhood $\mathcal{U}$ of $\bar\Theta$ such that for any initialization $\Theta^0 \in \mathcal{U}$, the iterates $\Theta^{t+1} = \mathcal{G}(\Theta^t)$  converge to $\bar\Theta$ at asymptotic linear rate
\begin{equation}
\limsup_{t \to \infty} \|\Theta^t - \bar\Theta\|_2^{1/t} 
\le \rho(\mathbf{J}) \le \rho(\mathbf{D}).
\end{equation}
 where $\mathbf{J}:=\nabla_{\Theta}\mathcal{G}(\bar{\Theta})$ denotes the Jacobian of $\mathcal{G}$ at $\bar{\Theta}$, and $\rho(\cdot)$ is the spectral radius of a matrix.   
\end{restatable}

The proof is in Appx.~\ref{subsec:profs_stability}. Thm.~\ref{thm:stability} establishes conditions for the convergence of networked self-consuming models.  Since $\mathbf{D}$ is nonnegative and encodes the interaction graph $\mathbf{W}$ directly, $\rho(\mathbf{D})$ serves as a topology-aware upper bound on the spectral radius of $\mathbf{J}$.  The stability holds whenever $\rho(\mathbf{D}) < 1$, i.e., when local 
amplification and network propagation together do not compound to an expanding fixed point. Since 
$\rho(\mathbf{D}) \le \|\mathbf{D}\|_2 \le \beta_{\max} \|\mathbf{W}\|_2$ with $\beta_{\max} := \max_{i \in \mathcal{V}} \beta_i$, we have the following corollary with its proof in Appx.~\ref{subsec:profs_norm-certificate}.

\begin{restatable}[A norm-based sufficient condition]{corollary}{Normcertificate}
\label{cor:norm-certificate}
If $\beta_{\max} \|\mathbf{W}\|_2 < 1$, then $\rho(\mathbf{D}) < 1$, and 
$\bar\Theta$ is locally asymptotically stable. Equivalently, when
$M_i\|\mathbf{W}\|_2 > \alpha_i - L_i\varepsilon_i$ for all $i\in\mathcal V$, a sufficient condition for local asymptotic stability is
\begin{equation}
\lambda_i < \frac{\alpha_i}
{M_i \|\mathbf{W}\|_2 - (\alpha_i - L_i \varepsilon_i)}, 
\quad \forall i \in \mathcal{V}.
\label{eq:norm-condition}
\end{equation}
\end{restatable}

As in isolated model settings, stability is guaranteed when the synthetic-data contribution $\lambda_i$ is sufficiently small. In the presence of cross-model interactions, the key difference is that the threshold now depends on two network-level quantities: the cross-model sensitivity $M_i$ which measures how parameter perturbations in parent models propagate through their synthetic data to model $i$'s gradient signal; and the spectral norm of the interaction graph $\|\mathbf{W}\|_2$, which measures the graph's worst-case amplification of perturbations across nodes. As $\|\mathbf{W}\|_2$ and $M_i$ increases, local instabilities propagate and compound across models, and the maximum allowable synthetic-data weight $\lambda_i$ decreases.

While this norm-based condition is convenient and analytically tractable, it relies on a global worst-case characterization of the interaction graph and can therefore be conservative. In particular, the spectral norm aggregates all interactions across the network into a single scalar quantity, which can obscure the specific topological features that contribute to instability. As a result, different interaction graphs may share the same value of $\|\mathbf{W}\|_2$ yet exhibit qualitatively different dynamical behaviors.

\textbf{Impacts of graph topology.} To move beyond global norm-based characterizations and identify which 
structural features of the interaction graph drive instability, we analyze how $\rho(\mathbf{D})$ is shaped by the graph's cycle structure.

\begin{definition}[Directed cycle and cycle gain]\label{def:cycle_gain}
A \emph{simple directed cycle} is an ordered sequence of distinct model nodes 
$C=(i_\ell  \to \cdots \to i_2 \to i_1 \to i_\ell)$
with $i_r \in \mathcal{V}$ for all $r$.
The \emph{cycle gain} of $C$ is defined as the geometric mean of the corresponding entries in $\mathbf{D}$:
$$
\gamma(C)
:= \left(\prod_{r=1}^{\ell} D_{i_r i_{r+1}}\right)^{1/\ell}
= \left(\prod_{r=1}^{\ell} \beta_{i_r} w_{i_r i_{r+1}}\right)^{1/\ell}
$$
where $i_{\ell+1}=i_1$.

\end{definition}

Intuitively, cycle gain $\gamma(C)$ captures per-iteration amplification accumulated along cycle $C$, where moving along each edge $i_{r+1}\to i_{r}$ contributes a factor $\beta_{i_r} w_{i_r i_{r+1}}$, and geometric mean reflects the average amplification per retraining step around the cycle.

\begin{restatable}[Cycle gain as a lower bound of $\rho(\mathbf{D})$]{proposition}{cyclegain}
\label{prop:cycle_gain}
Let $\mathbf{D}\ge 0$ be a nonnegative matrix and let $\{C_m\}_{m=1}^M$ denote the set of all simple directed cycles in the corresponding interaction graph. Then the spectral radius of $\mathbf{D}$ satisfies
$$
\rho(\mathbf{D}) \ge \max_{m\in[M]} \gamma(C_m)
$$
\end{restatable}
Prop.~\ref{prop:cycle_gain} indicates that in a network with multiple cycles, $\rho(D)$ is primarily governed by the cycle with the largest geometric mean gain. In other words, a single “dominant cycle” can bottleneck the entire system, regardless of the total number of cycles or their lengths. Cycle length affects stability only through the geometric mean: if every model along a cycle contributes the same per-step gain (i.e., $\beta_i w_{ij}=b$), then the overall cycle gain $\gamma(C)=b$ is independent of the cycle’s length.
Combining this with Thm.~\ref{thm:stability} yields the following necessary condition on cycle gains; see Appx.~\ref{subsec:profs_cycle_gain} for proofs.

\begin{restatable}[Necessary condition on cycles]{corollary}{Necessarycycle}
\label{cor:necessary-cycle}
If $\rho(\mathbf{D}) < 1$, then every simple directed cycle $C$ satisfies $\prod_{(j \to i) \in C} \beta_i w_{ij} < 1$. In particular, every self-loop must satisfy $\beta_i w_{ii} < 1$.
\end{restatable}
\section{Diversity Analysis}
\label{sec:HeteroData}

When $p_i^{\mathrm{data}}$ differ across models, it is natural to expect the resulting generative models to produce diverse outputs. A central question, however, is how such diversity evolves in the long run under cross-model interactions. In the following, we assume that models converge to a fixed point $\bar{\Theta}$ under the iterative retraining dynamics $\mathcal{G}$, and we study how model diversity manifests at $\bar{\Theta}$. 

At any fixed point $\bar\Theta$, each model $i$ is trained on a composite dataset drawn from the mixture distribution  $\bar p_i \propto p_i^{\mathrm{data}} + \lambda_i \sum_j w_{ij} p_{\bar{\boldsymbol{\theta}}_j}$. Because models may differ in architecture and parameter dimensionality, directly comparing parameters is not meaningful. We therefore measure diversity in a common representation space.

\begin{definition}[Output representations]
\label{def:rep_state}
Let $\phi: \mathcal{X} \to \mathbb{R}^D$ be a fixed $1$-Lipschitz feature map\footnote{$\phi$ could be the identity function or a mapping that captures semantic or perceptual properties of the data.}. For each model $i$, we define its \textit{output representation} 
at a fixed point $\bar{\Theta}$ as
$
\bar{\boldsymbol{f}}_i:= \mathbb{E}_{x \sim p_{\bar{\boldsymbol{\theta}}_i}}[\phi(x)]
$. The \textit{real-data representation centroid} for model $i$ is defined as
$
\boldsymbol{f}_{i}^{\mathrm{data}} := \mathbb{E}_{x \sim p_{i}^{\mathrm{data}}}[\phi(x)]
$.
\end{definition}

\begin{definition}[System diversity]
\label{def:system_diversity}
Define \emph{system output diversity} $\mathcal{D}_{\mathrm{out}}$ and \emph{system real-data diversity} $\mathcal{D}_{\mathrm{data}}$ as the empirical variances of $\{ \bar{\boldsymbol{f}}_i\}_{i \in \mathcal{V}}$ and of $\{\boldsymbol{f}_i^{\mathrm{data}}\}_{i \in \mathcal{V}}$ across the network, respectively:
\begin{align*}
\textstyle \mathcal{D}_{\mathrm{out}} &:= \mathrm{Var}_{i \in \mathcal{V}}\bigl[\bar{\boldsymbol{f}}_i\bigr] 
= \textstyle \frac{1}{|\mathcal{V}|} \sum_{i \in \mathcal{V}} 
\|\bar{\boldsymbol{f}}_i - \bar{\boldsymbol{f}}_{\mathrm{avg}}\|^2 \textstyle, ~\text{with }~\bar{\boldsymbol{f}}_{\mathrm{avg}} := \frac{1}{|\mathcal{V}|} \sum_{i \in \mathcal{V}} \bar{\boldsymbol{f}}_i\\
\textstyle \mathcal{D}_{\mathrm{data}} &:= \mathrm{Var}_{i \in \mathcal{V}}\bigl[\boldsymbol{f}_i^{\mathrm{data}}\bigr] 
= \textstyle \frac{1}{|\mathcal{V}|} \sum_{i \in \mathcal{V}} 
\|\boldsymbol{f}_i^{\mathrm{data}} - \boldsymbol{f}^{\mathrm{data}}_{\mathrm{avg}}\|^2, ~\text{with }~\boldsymbol{f}^{\mathrm{data}}_{\mathrm{avg}} := \frac{1}{|\mathcal{V}|} \sum_{i \in \mathcal{V}} \boldsymbol{f}_i^{\mathrm{data}}    
\end{align*}
\end{definition}

\begin{definition}[Fixed-point feature discrepancy]
\label{def:discrepancy}
For each model $i \in \mathcal{V}$, define the \emph{fixed-point feature 
discrepancy} as
$
\bar{\boldsymbol{r}}_i 
:= \bar{\boldsymbol{f}}_i - \mathbb{E}_{x \sim \bar{p}_i}[\phi(x)] 
\in \mathbb{R}^D,
$
which measures the mismatch between the output representation of model $i$ and the representation of its training mixture distribution $\bar{p}_i$.
\end{definition}
A nonzero $\bar{\boldsymbol{r}}_i $ arises when the model cannot exactly realize the mixture distribution, due to limited model expressivity, approximation error, or imperfect optimization during training.

With these definitions, we can characterize the fixed-point output representations $\{\bar{\boldsymbol{f}}_i\}_i$ as the solution of a graph-structured linear system, which sets the stage for analyzing system diversity.

\begin{restatable}[Fixed-point output representation of the system]{proposition}{globalsolution}
\label{prop:global_solution}
Let $\mathbf{\bar{F}}=[\bar{\boldsymbol{f}}_1,\ldots,\bar{\boldsymbol{f}}_K]^\top$ and $\mathbf{F}_{\mathrm{data}}=[{\boldsymbol{f}}_{1}^{\mathrm{data}},\ldots,{\boldsymbol{f}}_{K}^{\mathrm{data}}]^\top$ denote the fixed-point output representations and the real-data representations of all models, respectively.
Let $\mathbf{\Lambda}=\mathrm{diag}(\lambda_1,\ldots,\lambda_K)$ and $\mathbf{L}:=\mathbf{I}-\mathbf{W}$. We have
\begin{equation}
\label{eq:global_solution}
\bar{\mathbf{F}} 
= (\mathbf{I} + \mathbf{\Lambda}\mathbf{L})^{-1}
\bigl[\mathbf{F}_{\mathrm{data}} + (\mathbf{I} + \mathbf{\Lambda})\bar{\mathbf{R}}\bigr],
\end{equation}
where $\bar{\mathbf{R}} := [\bar{\boldsymbol{r}}_1, \ldots,\bar{\boldsymbol{r}}_K]^\top$ and $\bar{\boldsymbol{r}}_i$ satisfies$\|\bar{\boldsymbol{r}}_i\| \le \frac{1 + 2\lambda_i}{1+\lambda_i}\, \varepsilon_i$.
\end{restatable}

The proof is given in Appx.~\ref{subsec:proof_prop:global_solution}. With this solution, we can analyze how the system’s output diversity at the fixed point deviates from the real-data diversity, and the following theorem provides insight into how iterative self-consuming training dynamics shape system diversity.

\begin{restatable}[Upper bound on system output diversity]{theorem}{generaldiversity}
\label{thm:general_diversity}
Under Prop.~\ref{prop:global_solution}, the system output diversity at any fixed point $\bar\Theta$ satisfies
\begin{equation}
\label{eq:div_bound_blackbox}
\mathcal{D}_{\mathrm{out}} 
\le 2\bigl\|(\mathbf{I} + \mathbf{\Lambda}\mathbf{L})^{-1}\bigr\|_2^2 
\biggl[\mathcal{D}_{\mathrm{data}} 
+ \frac{1}{|\mathcal{V}|} \sum_{i \in \mathcal{V}} 
(1+\lambda_i)^2\, \|\bar{\boldsymbol{r}}_i\|^2\biggr].
\end{equation}
\end{restatable}
Eq.~\eqref{eq:global_solution} shows that the output representations at the fixed point are obtained by applying a \textit{linear smoothing} operator $(\mathbf{I} + \Lambda \mathbf{L})^{-1}$ to real-data centroids.
The diversity bound in Thm.~\ref{thm:general_diversity} (see proof in Appx.~\ref{subsec:proof_thm:diversity_bound}) quantifies the worst-case attenuation of heterogeneity along the non-consensus subspace, i.e., the subspace orthogonal to the all-ones vector where models disagree.
The factor $\|(\mathbf{I}+\Lambda \mathbf{L})^{-1}\|_2^2$ acts as a \textit{collapse coefficient}, mapping the initial data diversity
$\mathcal D_{\mathrm{data}}$ to the final output diversity $\mathcal D_{\mathrm{out}}$. This bound highlights three key insights: 
\begin{itemize}[leftmargin=*,topsep=-0.cm,itemsep=-0.0cm]
    \item \textit{No cross-model interaction preserves diversity.} If there is no synthetic data ($\Lambda=\mathbf{0}$), or interaction graph has no cross edges ($w_{ij}=0$, $\forall i\neq j$), then operator $(\mathbf{I}+\Lambda \mathbf{L})^{-1}$ does not average across models. Moreover, when each model's output matches its training mixture in feature space ($\bar{\boldsymbol{r}}_i = \mathbf{0}$ 
for all $i$), the bound reduces to $\mathcal{D}_{\mathrm{out}} = \mathcal{D}_{\mathrm{data}}$.
    \item \textit{Stronger synthetic-data mixing has competing effects.} Larger $\lambda_i$ increases the averaging effect of $(\mathbf{I}+\Lambda \mathbf{L})^{-1}$, pulling model representations toward a common value and reducing differences. However, larger $\lambda_i$ also amplifies the term $(1+\lambda_i)^2\, \|\bar{\boldsymbol{r}}_i\|^2$, since each model relies more heavily on imperfect synthetic training data.
    \item  \textit{Graph topology shapes which disagreements are reduced.} The operator $(\mathbf{I}+\Lambda \mathbf{L})^{-1}$ averages out disagreements that change rapidly across strongly connected regions of the graph, while differences across weakly connected regions are diminished much less.
    \item \textit{Imperfect representation introduces a topology-independent diversity term.} The discrepancy term $\frac{1}{|\mathcal{V}|}\sum_i (1+\lambda_i)^2 \|\bar{\boldsymbol{r}}_i\|^2$ captures, in the feature space, the gap between each model's fixed-point output and its training mixture; it contributes to $\mathcal{D}_{\mathrm{out}}$ regardless of $\mathbf{W}$. Such diversity reflects model error rather than meaningful inter-model disagreement.
\end{itemize}

\textbf{Special Case: undirected graph.} While Thm.~\ref{thm:general_diversity} provides a general bound on diversity collapse, it is expressed in terms of the global operator norm $ \big\|(\mathbf{I} + \Lambda \mathbf{L})^{-1}\big\|_2^2$, which obscures how specific structural features of the interaction graph affect diversity. To gain more interpretable, topology-dependent insights, we consider the instructive case of symmetric (i.e., undirected) interactions.

\begin{restatable}[Diversity Bound for Symmetric Interaction]{corollary}{symmetricbound}
\label{cor:symmetric_bound}
Suppose the interaction graph is undirected ($\mathbf{W} = \mathbf{W}^\top$), the synthetic weight is homogeneous ($\lambda_i \equiv \lambda$) and the feature discrepancy vanishes ($\bar{\boldsymbol{r}}_i = \mathbf{0}$). Let $\mu_2$ denote the algebraic connectivity (the second smallest eigenvalue of $\mathbf{L}$). Then:
\begin{equation}
\label{eq:symmetric_bound}
\mathcal{D}_{\mathrm{out}} \le \frac{1}{(1+\lambda \mu_2)^2} \mathcal{D}_{\mathrm{data}}
\end{equation}
\end{restatable}
Cor.~\ref{cor:symmetric_bound} provides a clean, topology-dependent characterization of diversity collapse, explicitly linking it to the algebraic connectivity $\mu_2$. Intuitively, $\mu_2$ captures the “bottleneckedness” of the network: a small $\mu_2$ indicates weak connections or community structures, whereas a large $\mu_2$ corresponds to strong connectivity with few bottlenecks. The corollary implies that higher connectivity (larger $\mu_2$) leads to a smaller upper bound and thus stronger collapse of diversity. Conversely, even in a densely connected network, a small $\mu_2$ (e.g., when communities are only weakly connected) can significantly slow down diversity collapse.

\section{Experiments}
\label{sec:exp}
Our theoretical analysis is under simplifying assumptions and abstracts away sources of error arising from finite data and imperfect optimization. To understand how self-consuming retraining operates beyond these idealized settings, we complement our theory with experiments, evaluating whether the qualitative behaviors predicted by the theory persist under realistic training conditions.

\begin{figure}[thbp!]
    \centering
    \begin{subfigure}[c]{0.08\linewidth}
        \centering
    \includegraphics[width=\linewidth]{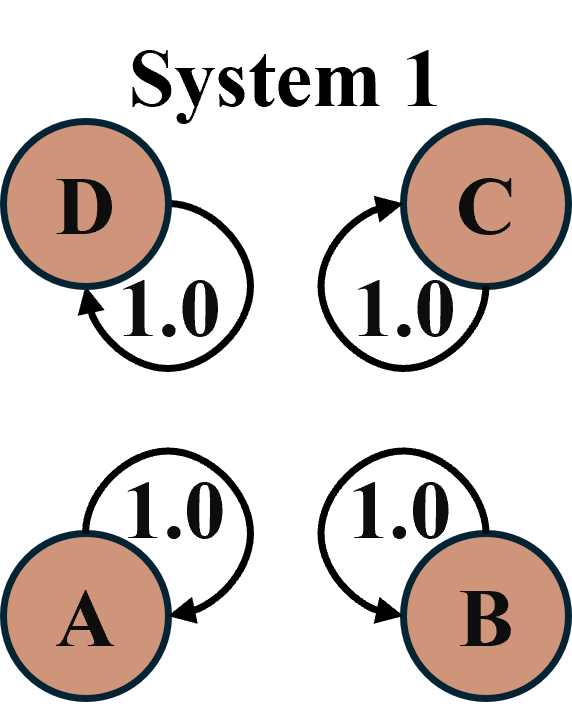}
    \end{subfigure}
    \hfill
    \begin{subfigure}[c]{0.89\linewidth}
        \centering
        \includegraphics[width=\linewidth]{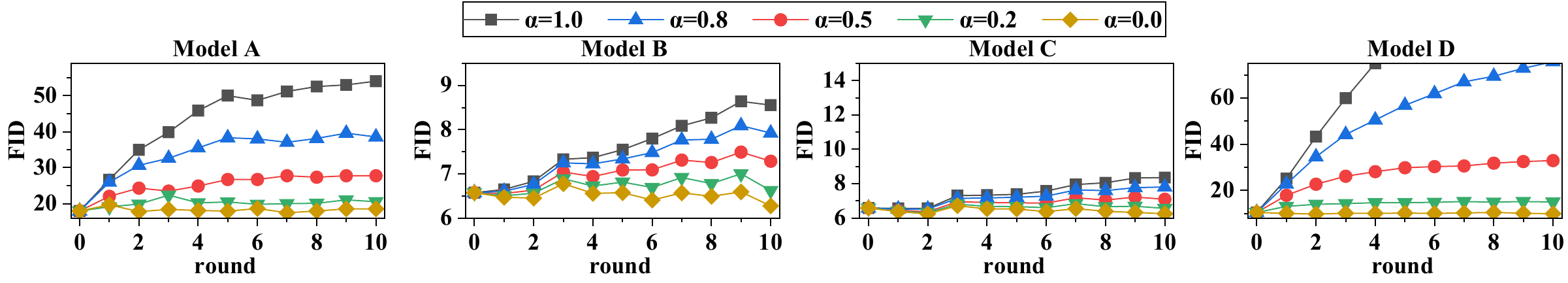}
    \end{subfigure}

    \begin{subfigure}[c]{0.08\linewidth}
        \centering
    \includegraphics[width=\linewidth]{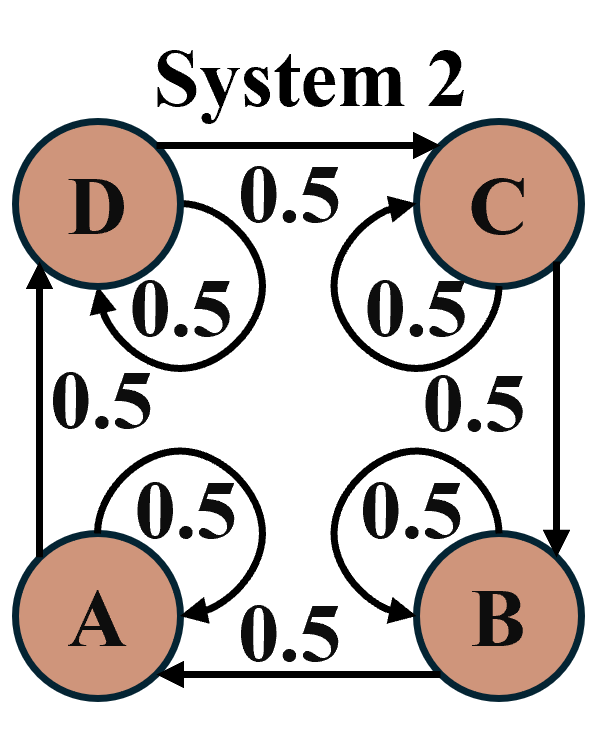}
    \end{subfigure}
    \hfill
    \begin{subfigure}[c]{0.89\linewidth}
        \centering
    \includegraphics[width=\linewidth]{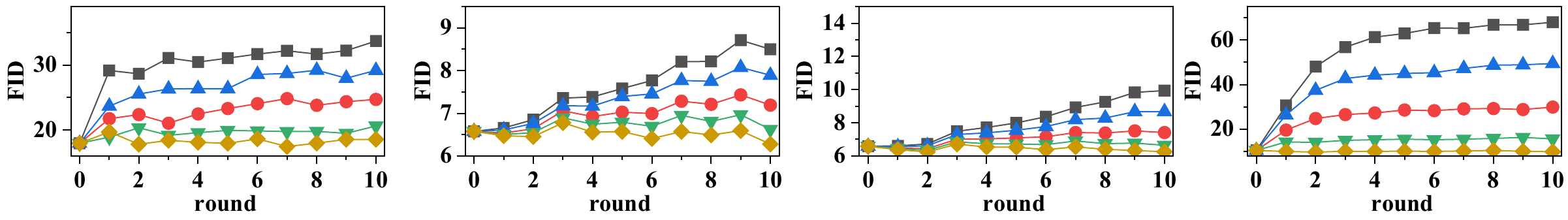}
    \end{subfigure}

    \begin{subfigure}[c]{0.08\linewidth}
        \centering
    \includegraphics[width=\linewidth]{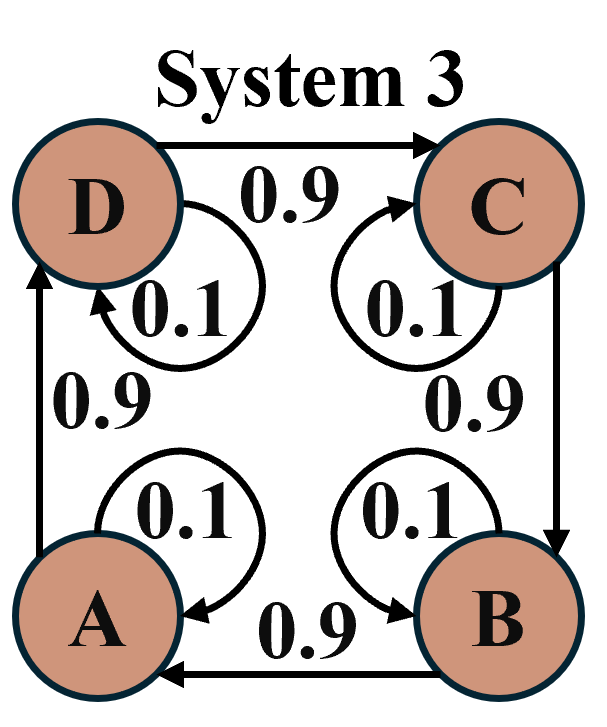}
    \end{subfigure}
    \hfill
    \begin{subfigure}[c]{0.89\linewidth}
        \centering
    \includegraphics[width=\linewidth]{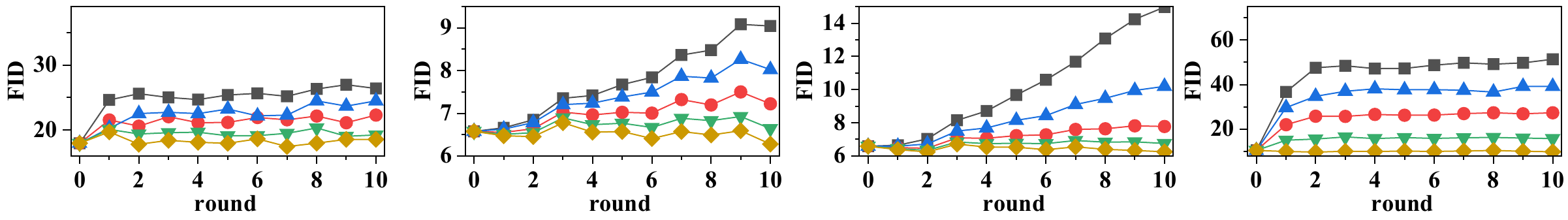}
    \end{subfigure}

    \begin{subfigure}[c]{0.08\linewidth}
        \centering
        \includegraphics[width=\linewidth]{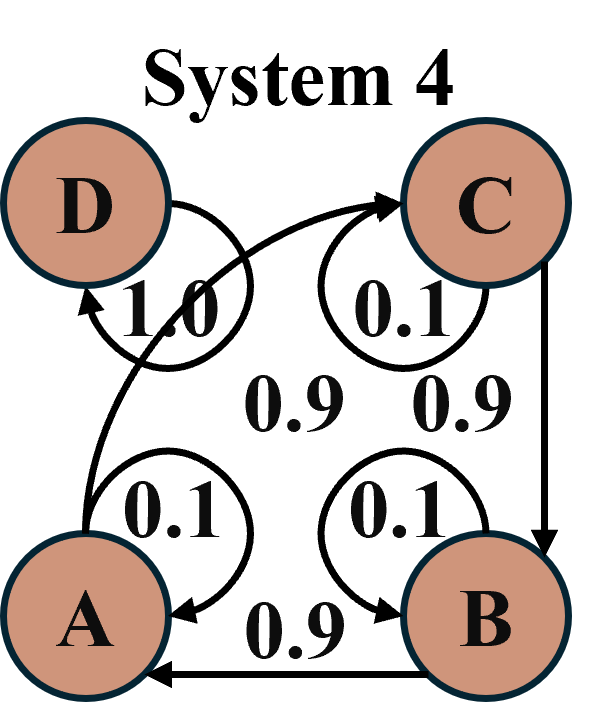}
    \end{subfigure}
    \hfill
    \begin{subfigure}[c]{0.89\linewidth}
        \centering
        \includegraphics[width=\linewidth]{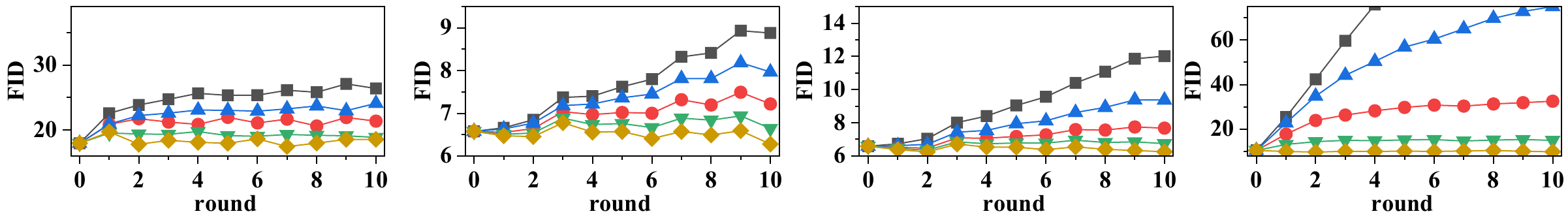}
    \end{subfigure}

    \caption{Stability experiments when all models trained on full CIFAR-10. Each row corresponds to a different interaction graph. For each row, the left figure illustrates the direction and strength of synthetic data flow, while the right figures report the FID between generated samples and real data for each model. Results are shown under different synthetic mixing strengths, with $\alpha_i=\alpha$ shared across models.}
    \label{Fig:Homo}
\end{figure}

\textbf{Datasets.} We conduct experiments on two datasets: (i) \textbf{Synthetic Gaussian}: A dataset following 8-mode Gaussian mixture model, the details are in Appx.~\ref{sec:synex}.
(ii) \textbf{CIFAR-10} \cite{cifar10}: It contains 60,000 images from 10 classes $\{\text{airplane}:=0, \text{automobile}:=1, \text{bird}:=2, \text{cat}:=3, \text{deer}:=4, \text{dog}:=5, \text{frog}:=6, \text{horse}:=7, \text{ship}:=8, \text{truck}:=9 \}$.

\begin{wraptable}[11]{r}{0.54\textwidth}
\vspace{-9pt}
\centering
\caption{Models and real-data assignments. The same data for all models is used for stability; the heterogeneous assignment is used for diversity experiments.}
\label{tab:model_data}
\resizebox{0.54\textwidth}{!}{
\begin{tabular}{c c c c}
\toprule
Model & Model Type & Stability & Diversity \\
\midrule
A & VLB Diffusion  
    & \multirow{4}{*}{$\{0,\cdots,9\}$}
  & $\{0,\cdots,9\}$  \\
B & OT-CFM
  &
  & $\{3,4\}$  \\
C & iCFM  
  &
  & $\{5,6\}$   \\
D & Hybrid Diffusion 
  &
  & $\{7,8,9\}$  \\
\bottomrule
\end{tabular}}
\vspace{-12pt}
\end{wraptable}

\textbf{Models.}
For real-data experiments, we consider two representative classes of generative models: denoising diffusion probabilistic models (DDPM) \citep{ddpm} and continuous flow matching models (CFM) \citep{cfm}. We instantiate four generative models from two paradigms, with two variants per paradigm as shown in \Cref{tab:model_data}.
All models are adopted from public GitHub repositories and are used with their default hyperparameters. Each model is initialized from its released pretrained checkpoint, keeping architectures and optimization settings unchanged. For the synthetic experiments, we additionally include a GAN-based \citep{GAN} model; details are in Appx.~\ref{sec:synex}.

\textbf{Iterative Retraining.} At each iteration, each model generates $10,000$ samples and is fine-tuned for a fixed number of steps on its mixed dataset of size $10,000$. This dataset consists of real samples and synthetic samples generated by other models.
The contribution of synthetic data is governed by the interaction matrix $\mathbf{W}$ and the synthetic mixing ratio $\alpha_i=\frac{\lambda_i}{\lambda_i+1}$. The procedure is in Alg.~\ref{alg:retraining}.

\begin{figure}[htbp]
    \centering
     \includegraphics[width=\linewidth]{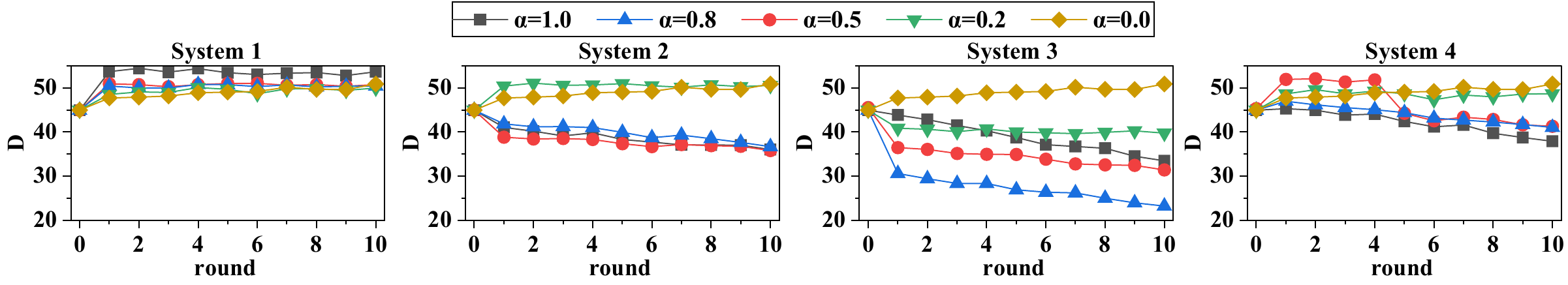}
    \caption{Diversity across models under heterogeneous data with different interaction graphs. \textit{Systems} are defined in Fig.~\ref{Fig:Homo}. Each figure reports the evolution of diversity $\mathcal{D}$ across retraining rounds under different synthetic mixing ratios $\alpha_i=\alpha$ shared across models.}
    \label{Fig:Hete}
    \vspace{-10pt}
\end{figure}

\textbf{Interaction Graphs.} We design four interaction graphs (Fig.~\ref{Fig:Homo}, left) to isolate the effects of graph structure, edge weights, and cycle configuration through controlled pairwise comparisons. \textbf{System~1} is an isolated-retraining baseline. \textbf{Systems~1} and~\textbf{2} share the same spectral radius $\|\mathbf{W}\| = 1$ but differ in graph structure. \textbf{Systems~2} and~\textbf{3} share the same graph structure but differ in edge weights. \textbf{System~4} has a shorter directed cycles. Additional graphs are explored in Appx.~\ref{subsec:synex_homo}.

\textbf{Evaluation Metrics.}
To assess distributional quality, we mainly compute the FID \citep{fidelity} between generated samples and real data using standard Inception feature embeddings. As a supplement, Appx.~\ref{sec:addrealex} reports precision and recall. To quantify diversity across models, we use 512-dimensional features from an ImageNet-pretrained \citep{ImageNet} ResNet-18 \citep{ResNet} with the final classification layer removed. 
System diversity $\mathcal{D}$ is computed following Def.~\ref{def:system_diversity}. Since a single run of iterative retraining is computationally expensive, we introduce stochasticity across rounds by varying the random seed. Specifically, in round $r$ for model $i$, we use a seed of the form $\mathrm{seed} + 10{,}000 \cdot r + 97 \cdot i$. For synthetic data experiments, all reported results are averaged over three independent runs with different seeds. We present the results of CIFAR-10 data here and the results of Gaussian data are in Appx.~\ref{sec:synex}.

\textbf{Stability experiments.} We first verify the stability analysis, training all models on the full CIFAR-10 ($p_i^{\mathrm{data}} \equiv p^{\mathrm{data}}$). The synthetic-data 
mixing ratio is shared across models ($\alpha_i \equiv \alpha$). Fig.~\ref{Fig:Homo} illustrates several interaction graph structures with varying patterns of synthetic data flow; additional graph configurations are reported in Appx.~\ref{subsec:synex_homo}. \textbf{System~1} corresponds to isolated retraining as a baseline. Across all systems, increasing $\alpha$ leads to higher FID, indicating degraded distributional quality. The four models show different sensitivities, reflecting architecture and training dynamics rather than data access. Crucially, neither the interaction norm $\|\mathbf{W}\|$ nor graph structure alone determines stability: \textbf{Systems~1} and \textbf{2} share $\|\mathbf{W}\|$ but exhibit different FID trends, while \textbf{Systems~2} and \textbf{3} share graph structure but differ in interaction strength. This is consistent with our spectral-radius criterion in Thm.~\ref{thm:stability}, which combines both ingredients. Comparing \textbf{Systems~3} and \textbf{4}, shorter interaction cycles can mitigate degradation when real data are present. This \textit{contrasts} with observations in the synthetic settings (Appx.~\ref{sec:synex}), where short cycles may instead accelerate instability, indicating that instability depends on which models participate in dominant feedback paths rather than cycle length itself. Additional synthetic-data experiments appear in Appx.~\ref{subsec:synex_homo}; precision and recall are in Appx.~\ref{sec:addrealex}.

\textbf{Diversity experiments.}
We next verify the diversity analysis as shown in Fig.~\ref{Fig:Hete}, training models on disjoint or partially overlapping subsets of CIFAR-10 (\Cref{tab:model_data}); this captures the practical 
scenario where different model providers have access to different data sources. Overall, increasing $\alpha$ and using denser interaction graphs both reduce diversity, consistent with stronger synthetic-data reliance encouraging alignment.
\textbf{System~1} (isolated retraining) is an exception: its diversity is driven by model instability rather than meaningful differentiation, and does not decrease with $\alpha$. \textbf{Systems~2} and \textbf{3} have similar topology but different interaction weights, where \textbf{System~3} exhibits a faster and more pronounced decline in $\mathcal{D}$, indicating that interaction strength matters beyond just the presence of $\alpha$. 
\textbf{Systems~2} and \textbf{3} also feature denser interactions than \textbf{System~4} and exhibit faster diversity decay across most $\alpha > 0$, aligning with our prediction (Thm.~\ref{thm:general_diversity}, Cor.~\ref{cor:symmetric_bound}) that higher graph connectivity accelerates homogenization. Appx.~\ref{subsec:synex_hete} reports  synthetic experiments with similar qualitative behavior. Note that increases in $\mathcal{D}$ may also result from instability, which is 
difficult to quantify since the fixed point is generally unknown; sample visualizations in Appx.~\ref{sec:addrealex} aid interpretation.

\section{Conclusion}
This work takes a first step toward studying networked self-consuming generative models through a graph-based dynamical framework, revealing how synthetic-data mixing and interaction topology jointly govern system stability and diversity. Our results show that cross-model feedback can induce topology-dependent amplification and diversity contraction, offering a principled understanding of the long-term behavior of multi-model generative ecosystems.

\begin{ack}
This material is based upon work supported by the U.S. National Science Foundation under Grant Nos. IIS-2202699, IIS-2416895, IIS-2301599, CMMI-2301601, DMS-2529302, and OIA-2616279. The authors declare no competing interests.

\end{ack}

\small

\bibliographystyle{plainnat}
\bibliography{ConsumingSys}

\newpage
\appendix
\startcontents[appendix]

\section*{Appendix Contents}
\hypersetup{linkcolor=black}
\printcontents[appendix]{}{1}{
    \setcounter{tocdepth}{2}
}
\hypersetup{linkcolor=blue}
\newpage
\section{Limitations and Future Directions}
\label{sec:limitations}

This work provides a first theoretical step toward understanding networked
self-consuming generative ecosystems. Our main results take a local stability
perspective around an arbitrary fixed point whenever such a fixed point exists.
While this view is standard in dynamical-systems analysis, a natural future
direction is to establish broader global existence and uniqueness guarantees for heterogeneous, nonconvex retraining systems.

Our analysis builds on the standard local regularity assumptions adopted in
prior single-model self-consuming \cite{Satbility,curation} analyses, and extends this style of local stability reasoning to networked retraining dynamics. These assumptions make the coupled system analytically tractable, but practical deep generative models involve finite-sample data, stochastic optimization, and nonconvex objectives. Extending the theory to approximate updates and weaker regularity conditions is an important direction for further bridging theory and large-scale generative-model training.

Finally, our diversity result gives an upper bound that reveals a graph-smoothing effect at equilibrium. Future work could derive sharper diversity characterizations, separate meaningful specialization from instability-induced diversity, and study adaptive or time-varying interaction graphs. The framework also suggests mitigation strategies, such as controlling high-gain feedback cycles or allocating real data to nodes with large local amplification.

\section{Broader Impact}
\label{sec:broader}

This work develops a theoretical framework for understanding stability and diversity in networked self-consuming generative ecosystems. The framework maynhelp model providers and platforms diagnose long-term risks such as distributional drift, instability, homogenization, and loss of diversity when synthetic data is reused across models. It may also inform mitigation strategies, such as limiting high-gain feedback cycles, preserving data diversity, or allocating real data to parts of the system that are most vulnerable to amplification.

At the same time, understanding how graph structure and synthetic-data mixing shape long-term model behavior could be misused to deliberately steer model ecosystems toward narrower or biased output distributions. For this reason,
practical use of such analyses should be paired with transparent data tracing, auditing of synthetic-data flows, and careful evaluation of downstream effects on different user groups and content domains.

\newpage
\section{Related Work}
\label{sec:relatedwork}

\subsection{Single-model Self-Consuming Loop}
High-quality synthetic data has been widely used for model pretraining, fine-tuning, and data augmentation. However, when models are repeatedly trained on their own or other models' generated outputs, the training distribution may drift away from the true data distribution. 

Through extensive experimentation, early studies such as \cite{GoMad, AICollapse,CollapseInevitable} demonstrated that this retraining loop leads to \textit{model collapse}, a degenerative process where models progressively lose their ability to represent the underlying data distribution.
Earlier work by \citet{amplification} showed that feedback loops in model-generated data can amplify existing dataset biases, laying the foundation for understanding recursive error reinforcement.
Specifically, \citet{AICollapse} identified that statistical approximation errors accumulate across generations, causing the model to lose information about the distribution's tails. Simultaneously, \citet{GoMad} and \citet{CollapseInevitable} observed that without sufficient fresh real data, the loop leads to a complete divergence from the ground truth.
Additional large-scale experiments by \citet{LLMDistributionShift} highlighted how properties of the training data, such as redundancy, diversity, and prompt type, modulate the degree of distribution shift in recursive LLM training.

Building on these empirical findings, subsequent work focused on theoretically characterizing the conditions under which collapse occurs. \citet{Satbility} provided conditions for stability, emphasizing the importance of retaining a non-negligible fraction of real data. Similarly, \citet{TheoreticalPerspective} modeled self-consuming loops as dynamical systems and showed that outputs lose variance and converge to trivial fixed points unless synthetic data is carefully scheduled. From a different angle, \citet{Tails} analyzed how recursive training alters scaling laws, leading to diminished tail coverage and faster error decay on rare examples.
Complementary to this, \citet{StatisticalAnalysisLM} provided a statistical treatment of model collapse, revealing monotonic increases in KL divergence and sharp variance decay during recursive fine-tuning. \citet{RateCollapse} further modeled the rate at which collapse emerges, quantifying error accumulation speed under recursive sampling and establishing bounds for early-stage divergence.

These works reveal how recursive training can cause models to drift, lose diversity, and collapse. They offer strong theoretical and empirical support. However, most of this work focuses on single-model behavior. Our framework subsumes these single-model settings as a special case. In particular, when $w_{ii} =1$ and $w_{ij}=0$ for all $i \neq j$, each model retrains exclusively on its own generated data, recovering the single-model self-consuming dynamics analyzed.

\subsection{Multi-Model Interactive Ecosystems}
\citet{AreWe} found that a large and growing fraction of online content is now generated by AI. This implies that training data used by one model often includes outputs from others, making cross-model data consumption common in practice. Motivated by this, recent work has begun to explore multi-model generative ecosystems, where models are explicitly trained on each other’s outputs.

\citet{CoEvolving} studied the stability of a co-evolving system involving one image model and one text model. Through simulations, they showed that when models learn from each other, small distribution shifts can compound and lead to joint instability. \citet{WhenandHow} further studied interacting self-consuming models under human curation, showing that cross-model interactions can substantially alter the long-term effect of preference-based curation. However, their analysis is limited to a two-model setup with fixed data domains, and may not generalize to larger or more diverse systems. \citet{Whathappens} approached the problem from a theoretical angle. They analyzed how one model behaves when trained on data that mixes outputs from multiple other models. Their theoretical results rely on simplifying assumptions, such as linearized dynamics assumptions. Their experiments also focus on just two interacting models, limiting the scope of their findings.
\citet{mutimodel} studied repeated fine-tuning of multi-modal generative systems composed of two coupled models using self-generated synthetic data. They showed that such closed feedback loops can lead to performance degradation over time, despite occasional short-term improvements in task metrics.
\citet{AdversarialCuration} studied a setup where two models curate data for each other with adversarial intent. Their work focuses on how harmful interactions between models affect training dynamics. 

While these works \citep{CoEvolving,Whathappens,mutimodel,WhenandHow,market} provide early insights into multi-model collapse, they are subject to several constraints. Most existing studies focus on small model pairs. The systems considered are relatively narrow in scope and not designed to represent general-purpose model ecosystems. In addition, many works assume predefined interaction structures rather than allowing interactions to emerge naturally.

In contrast, our work presents a more general framework that captures real-world interaction patterns. We explicitly model the flow of synthetic data across models and analyze the resulting network from a graph-theoretic perspective. This allows us to identify key factors that affect system stability, such as synthetic data ratio, topology density, and model heterogeneity. We further validate our analysis through simulations involving multiple models and diverse interaction structures.

\subsection{Mitigation Strategies for Model Collapse}
Model collapse is driven by loss of diversity and distributional drift. To counter this, recent research has explored several directions. 

One widely supported strategy is to keep real data in the training loop. \citet{GoMad} and \citet{Satbility} show that even a small reserve of real data can stabilize learning. Instead of discarding past data, cumulative synthetic samples  without replacing original data have proven more robust \cite{Whathappens,CollapseInevitable}.

Another approach introduces external selection. \citet{curation} and \citet{Co-EvolvingHuman} provide a theoretical analysis of human-curated training loops. They show that when models are retrained on their own outputs filtered by human feedback or a reward model, the process implicitly optimizes for those preferences. This curation mechanism can prevent collapse and guide models toward desirable behavior over time. Building on this, \citet{ScalingUp} emphasize that scaling up with synthetic data requires explicit verification mechanisms. They argue that without such checks, recursive training amplifies model biases and misalignment. \citet{Verification} propose using a verifier model to score generated samples, admitting only high-quality outputs into training.

Some defenses are built into the model loop itself. \citet{SIMS} proposed a method called SIMS for diffusion models. It adds a negative guidance signal during generation to reduce drift and improve stability. \citet{Self-Correcting} designed a training loop where generated samples are first corrected using expert rules before being reused. This helps keep the synthetic data high-quality and reduces the chance of reinforcing errors.  \citet{LatentSpaceFiltering} introduce latent space filtering to diffusion models. By removing low-entropy samples during retraining, their method preserves diversity and delays degradation.

Consistent with prior findings, our study also highlights the importance of real data during retraining. We observe that real data provides a stabilizing signal against distributional drift. Moreover, several existing mitigation strategies discussed above are conceptually orthogonal to our approach and could be incorporated into our framework.

\newpage
\section{Proofs}
\label{sec:proof}

Throughout this appendix, we work under Assumptions. For model $i$, recall that $\varepsilon_i := \max_{j: w_{ij} > 0}d_{\mathcal{W}}\bigl(p_{\bar{\boldsymbol{\theta}}_j},p_i^{\mathrm{data}}\bigr).$
The local objective of model $i$ is denoted
$$
\mathcal{H}_i(\boldsymbol{\theta}; \Theta) 
:= \mathbb{E}_{x \sim p_i^{\mathrm{data}}}[\log p_{\boldsymbol{\theta}}(x)] 
+ \lambda_i \sum_{j=1}^{K} w_{ij}\, 
\mathbb{E}_{x \sim p_{\boldsymbol{\theta}_j}}[\log p_{\boldsymbol{\theta}}(x)],
$$
so that $\mathcal{G}_i(\Theta) \in \arg\max_{\boldsymbol{\theta} \in \Theta_i} 
\mathcal{H}_i(\boldsymbol{\theta}; \Theta)$.

\subsection{Proofs of Characterization of Fixed Points}
\label{subsec:profs_chara_fp}
\begin{lemma}[The local maximum likelihood solution is unique]
\label{lem:well_posedness}
Under assumptions and $\alpha_i > L_i \varepsilon_i$, for each model $i$ there exists a neighborhood $\mathcal{U}$ of $\bar\Theta$ such that for any 
$\Theta \in \mathcal{U}$, the local objective $\boldsymbol{\theta}'_i \mapsto H_i(\boldsymbol{\theta}'_i; \Theta)$ admits a unique local maximizer on $U_i$. Consequently, $\mathcal{G}_i$ is single-valued on $\mathcal{U}$.
\end{lemma}

\begin{proof}
Fix a model $i$. The Hessian of $H_i$ with respect to the optimizing 
variable $\boldsymbol{\theta}'_i$ is
$$
\nabla_{\boldsymbol{\theta}'_i}^2 H_i(\boldsymbol{\theta}'_i; \Theta) 
= (1 + \lambda_i)\, 
\underbrace{\mathbb{E}_{p_i^{\mathrm{data}}}[\nabla^2 \log p_{\boldsymbol{\theta}'_i}]}
_{\preceq -\alpha_i I}
+ \lambda_i \sum_{j} w_{ij}\, 
\underbrace{\bigl(\mathbb{E}_{p_{\boldsymbol{\theta}_j}}[\nabla^2 \log p_{\boldsymbol{\theta}'_i}] 
- \mathbb{E}_{p_i^{\mathrm{data}}}[\nabla^2 \log p_{\boldsymbol{\theta}'_i}]\bigr)}
_{=:\, \Delta^i_j}.
$$
By Assumption~\ref{asp:hessian_stability}, $\|\Delta^i_j\|_2 \le L_i\, d_{\mathcal{W}}(p_{\boldsymbol{\theta}_j}, p_i^{\mathrm{data}})$. 
When evaluated at $\Theta = \bar\Theta$ and any $j$ with $w_{ij} > 0$, 
this is at most $L_i \varepsilon_i$ by definition. 

By continuity of $\boldsymbol{\theta}_j \mapsto d_{\mathcal{W}}(p_{\boldsymbol{\theta}_j}, p_i^{\mathrm{data}})$, the bound extends to a neighborhood of $\bar\Theta$ with a slightly relaxed constant $\varepsilon_i + \delta$ for any $\delta > 0$. 

Since $\alpha_i > L_i \varepsilon_i$ strictly, we may choose $\delta > 0$ small enough that $\alpha_i > L_i(\varepsilon_i + \delta)$. 

This gives a neighborhood $\mathcal{U}$ of $\bar\Theta$ on which, for all 
$\boldsymbol{\theta}'_i \in U_i$ and $\Theta \in \mathcal{U}$,
$$
\nabla_{\boldsymbol{\theta}'_i}^2 H_i(\boldsymbol{\theta}'_i; \Theta) 
\preceq -\bigl((1 + \lambda_i)\alpha_i - \lambda_i L_i(\varepsilon_i + \delta)\bigr) I 
\prec 0.
$$
Hence $\boldsymbol{\theta}'_i \mapsto H_i(\boldsymbol{\theta}'_i; \Theta)$ 
is strongly concave on $U_i$ uniformly in $\Theta \in \mathcal{U}$, so it 
admits a unique local maximizer and $\mathcal{G}_i$ is single-valued on 
$\mathcal{U}$.
\end{proof}

\begin{lemma}[Characterization of fixed points]
\label{sublem:stationarity_iff}
Under assumptions and $\alpha_i > L_i \varepsilon_i$, a global parameter vector $\bar\Theta$ is a fixed point of $\mathcal{G}$ if and only if, for every $i \in \mathcal{V}$,
$$
\nabla_{\boldsymbol{\theta}} \mathbb{E}_{x \sim p_i^{\mathrm{data}}}
[\log p_{\bar{\boldsymbol{\theta}}_i}(x)] 
+ \lambda_i \sum_{j=1}^{K} w_{ij}\, 
\nabla_{\boldsymbol{\theta}} \mathbb{E}_{x \sim p_{\bar{\boldsymbol{\theta}}_j}}
[\log p_{\bar{\boldsymbol{\theta}}_i}(x)] = 0.
$$
\end{lemma}

\begin{proof}
($\Rightarrow$) Suppose $\bar\Theta = \mathcal{G}(\bar\Theta)$. Then for 
each $i$, $\bar{\boldsymbol{\theta}}_i = \mathcal{G}_i(\bar\Theta) \in \arg\max_{\boldsymbol{\theta} \in \Theta_i} H_i(\boldsymbol{\theta}; \bar\Theta)$. 

Since $\mathcal{H}_i$ is twice continuously differentiable, and $\bar{\boldsymbol{\theta}}_i$ is an interior local maximizer (strong concavity on $U_i$ from Lemma~\ref{lem:well_posedness} places the maximum in the interior of $U_i$), Fermat's rule gives $\nabla_{\boldsymbol{\theta}} H_i(\bar{\boldsymbol{\theta}}_i; \bar\Theta) = 0$.

($\Leftarrow$) Suppose $\nabla_{\boldsymbol{\theta}} H_i(\bar{\boldsymbol{\theta}}_i; \bar\Theta) = 0$ holds for every $i$.
By Lemma~\ref{lem:well_posedness}, $\boldsymbol{\theta} \mapsto H_i(\boldsymbol{\theta}; \bar\Theta)$ is strongly concave on $U_i$, so any stationary point in $U_i$ is the unique local maximizer on $U_i$. 

Evaluating the tie-breaking rule at $\boldsymbol{\theta}_i^{t-1} = \bar{\boldsymbol{\theta}}_i$ gives distance zero, so $\mathcal{G}_i(\bar\Theta) = \bar{\boldsymbol{\theta}}_i$. Stacking over $i$: $\mathcal{G}(\bar\Theta) = \bar\Theta$.
\end{proof}

\subsection{Proof of the Uncoupled Solution as a Fixed Point}
\label{subsec:profs_uncoupled}

\begin{proposition}[Uncoupled solution as a fixed point]
\label{prop:uncoupled}
Suppose the real data distributions are homogeneous, 
$p_i^{\mathrm{data}} \equiv p^{\mathrm{data}}$ for all $i \in \mathcal{V}$, and there exists a distribution $p^\star$ such that $p_{\boldsymbol{\theta}_i^\star} \equiv p^\star$ for every $i \in \mathcal{V}$. Then $\Theta^\star = (\boldsymbol{\theta}_1^\star, \dots, \boldsymbol{\theta}_K^\star)$ is a fixed point of $\mathcal{G}$ for any $\lambda_i \ge 0$ and any row-stochastic $\mathbf{W}$.
\end{proposition}

\begin{proof}
Fix any $i \in \mathcal{V}$. The local objective of model $i$ evaluated at $\Theta = \Theta^\star$ is
\begin{align}
H_i(\boldsymbol{\theta}; \Theta^\star) 
&= \mathbb{E}_{x \sim p^{\mathrm{data}}}[\log p_{\boldsymbol{\theta}}(x)] 
+ \lambda_i \sum_{j=1}^{K} w_{ij}\, 
\mathbb{E}_{x \sim p_{\boldsymbol{\theta}_j^\star}}[\log p_{\boldsymbol{\theta}}(x)] 
\notag \\
&= \mathbb{E}_{x \sim p^{\mathrm{data}}}[\log p_{\boldsymbol{\theta}}(x)] 
+ \lambda_i \sum_{j=1}^{K} w_{ij}\, 
\mathbb{E}_{x \sim p^\star}[\log p_{\boldsymbol{\theta}}(x)] 
\notag \\
&= \mathbb{E}_{x \sim p^{\mathrm{data}}}[\log p_{\boldsymbol{\theta}}(x)] 
+ \lambda_i\, 
\mathbb{E}_{x \sim p^\star}[\log p_{\boldsymbol{\theta}}(x)], 
\label{eq:Hi_at_star}
\end{align}
where the second equality uses $p_{\boldsymbol{\theta}_j^\star} \equiv p^\star$ for all $j$, and the third uses row-stochasticity $\sum_j w_{ij} = 1$.

We show $\boldsymbol{\theta}_i^\star$ maximizes each term on the right-hand side of~\eqref{eq:Hi_at_star} over $\Theta_i$.

For the first term, by definition of $\boldsymbol{\theta}_i^\star$ as the MLE on $p^{\mathrm{data}}$:
\begin{equation}
\mathbb{E}_{x \sim p^{\mathrm{data}}}[\log p_{\boldsymbol{\theta}}(x)] 
\le \mathbb{E}_{x \sim p^{\mathrm{data}}}[\log p_{\boldsymbol{\theta}_i^\star}(x)], 
\qquad \forall \boldsymbol{\theta} \in \Theta_i.
\label{eq:first_term_opt}
\end{equation}
For the second term, since $p_{\boldsymbol{\theta}_i^\star} \equiv p^\star$, for any $\boldsymbol{\theta} \in \Theta_i$:
\begin{align}
\mathbb{E}_{x \sim p^\star}[\log p_{\boldsymbol{\theta}}(x)] 
&= \mathbb{E}_{x \sim p^\star}[\log p_{\boldsymbol{\theta}_i^\star}(x)] 
- D_{\mathrm{KL}}(p^\star \,\|\, p_{\boldsymbol{\theta}}) \notag \\
&\le \mathbb{E}_{x \sim p^\star}[\log p_{\boldsymbol{\theta}_i^\star}(x)], 
\label{eq:second_term_opt}
\end{align}
using non-negativity of KL divergence. 

Combining~\eqref{eq:Hi_at_star}--\eqref{eq:second_term_opt} yields
$$
H_i(\boldsymbol{\theta}; \Theta^\star) 
\le H_i(\boldsymbol{\theta}_i^\star; \Theta^\star), 
\qquad \forall \boldsymbol{\theta} \in \Theta_i.
$$
Hence $\boldsymbol{\theta}_i^\star$ is a maximizer of $H_i(\cdot; \Theta^\star)$.

By the nearest-local-maximizer tie-breaking rule, $\boldsymbol{\theta}_i^{t-1} = \boldsymbol{\theta}_i^\star$ gives distance zero, so $\mathcal{G}_i(\Theta^\star) = \boldsymbol{\theta}_i^\star$. 

Stacking over $i$: $\mathcal{G}(\Theta^\star) = \Theta^\star$.
\end{proof}

\subsection{Example: Heterogeneous Fixed Point Different from the Uncoupled Solution}
\label{subsec:example}
We present a simple two-model example illustrating that, under heterogeneous data, the coupled retraining dynamics can admit a stable fixed point distinct from the uncoupled solution $\Theta^\star$.

Consider two models $K=2$ with scalar parameters $\theta_1,\theta_2\in\mathbb R$.

Each model observes a different real-data distribution, let the $\log$-likelihood be given by
$$
\log p_{\boldsymbol{\theta}_i}(x) = -\tfrac12(\theta_i-a_i)^2
$$
so that in the absence of interactions, the uncoupled solution is $\theta_i^\star=a_i$, $
\Theta^\star = (a_1, a_2)$.

Introduce symmetric synthetic interactions $\lambda_i=\lambda_j=\lambda$ and $w_{ij}=w_{ji}=1$ by letting each model update according to
$$
H_i(\theta_i;\theta_j)
= \log p_{\boldsymbol{\theta}_i}(x) - \tfrac{\lambda}{2}(\theta_i-\theta_j)^2
$$

For each $i$, the objective $H_i$ is strictly concave in $\theta_i$, so the best-response map $\mathcal{G}_i$ is uniquely defined and smooth.

To compute the update map $\mathcal{G}_i$, we take the first-order condition with respect to $\theta_i$:
$$
\frac{\partial H_i}{\partial \theta_i}
= -(\theta_i - a_i) - \lambda(\theta_i - \theta_j)
$$
Setting this derivative to zero yields
$$
(1+\lambda)\theta_i = a_i + \lambda \theta_j
$$
and hence the unique best response
$$
\mathcal{G}_i(\theta_j) = \frac{a_i + \lambda \theta_j}{1+\lambda}
$$

A fixed point $\bar\Theta=(\bar\theta_1,\bar\theta_2)$ therefore satisfies
$$
\bar\theta_1=\frac{a_1+\lambda \bar\theta_2}{1+\lambda}
\qquad
\bar\theta_2=\frac{a_2+\lambda \bar\theta_1}{1+\lambda}
$$
So,
$$
\bar\theta_1=\frac{(1+\lambda)a_1+\lambda a_2}{1+2\lambda}
\qquad
\bar\theta_2=\frac{\lambda a_1+(1+\lambda)a_2}{1+2\lambda}
$$

For example, setting $a_1=1$, $a_2=-1$, and $\lambda=1$ gives the unique fixed point
$$
\bar\Theta=\bigl(\tfrac13,-\tfrac13\bigr)
$$
whereas the uncoupled solution is $\Theta^\star=(1,-1)$.

\paragraph{Local stability.}
The Jacobian of the update map $\mathcal{G}$ is
$$
\mathbf J =
\begin{pmatrix}
0 & \frac{\lambda}{1+\lambda} \\
\frac{\lambda}{1+\lambda} & 0
\end{pmatrix}
$$
Its spectral radius is $\rho(\mathbf J)=\frac{\lambda}{1+\lambda}<1$, confirming that the fixed point $\bar\Theta$ is locally stable.

Thus, under heterogeneous data, the fixed point of the coupled retraining dynamics exists and is locally stable, but differs from the uncoupled solution.

\subsection{Proofs of Theorem~\ref{thm:stability}}
\label{subsec:profs_stability}

\begin{proposition}[Block Jacobian bound]
\label{subprop:jac-block}
Let $\bar\Theta$ be a fixed point of $\mathcal{G}$. Under assumptions and $\alpha_i > L_i \varepsilon_i$, $\mathcal{G}$ is $C^1$ in a neighborhood of $\bar\Theta$, the Jacobian $\mathbf{J} := \nabla_\Theta \mathcal{G}(\bar\Theta)$ is well-defined, and its $(i,j)$-th block satisfies
\begin{equation}
\label{eq:block_bound_final}
\|J_{ij}\|_2 \le \beta_i w_{ij}, 
\qquad \beta_i := \frac{\lambda_i M_i}{\alpha_i + \lambda_i(\alpha_i - L_i \varepsilon_i)}.
\end{equation}
\end{proposition}

\begin{proof}
\emph{$C^1$ smoothness and Jacobian decomposition.}
Define the stationarity map 
$$
F_i(\boldsymbol{\theta}'_i; \Theta) := 
\nabla_{\boldsymbol{\theta}'_i} H_i(\boldsymbol{\theta}'_i; \Theta).
$$
By Lemma~\ref{sublem:stationarity_iff}, $F_i(\bar{\boldsymbol{\theta}}_i; \bar\Theta) = 0$. Moreover, by Lemma~\ref{lem:well_posedness}, 
$H_{ii} := \nabla_{\boldsymbol{\theta}'_i}^2 H_i(\bar{\boldsymbol{\theta}}_i; \bar\Theta) \prec 0$, 
hence invertible.

The Implicit Function Theorem \citep{ImplicitFunctionTheorem} applied to $F_i = 0$ yields a 
$C^1$ selection $\Theta \mapsto \mathcal{G}_i(\Theta)$ in a neighborhood 
of $\bar\Theta$ with $\mathcal{G}_i(\bar\Theta) = \bar{\boldsymbol{\theta}}_i$ 
and
$$
\nabla_\Theta \mathcal{G}_i(\bar\Theta) 
= -H_{ii}^{-1}\, \nabla_\Theta F_i(\bar{\boldsymbol{\theta}}_i; \bar\Theta).
$$

Stacking over $i$, $\mathcal{G}$ is $C^1$ and $\mathbf{J} := \nabla_\Theta \mathcal{G}(\bar\Theta)$ is well-defined.

Differentiating $F_i$ with respect to $\boldsymbol{\theta}_j$:
$$
\frac{\partial F_i}{\partial {\boldsymbol{\theta}}_j}(\bar{\boldsymbol{\theta}}_i)
=
\lambda_i w_{ij}
\frac{\partial}{\partial {\boldsymbol{\theta}}_j}
\Big(
\mathbb{E}_{p_{{\boldsymbol{\theta}}_j}}\big[\nabla_\theta \log p_{\bar{\boldsymbol{\theta}}_i}\big]
\Big)\Big|_{{\boldsymbol{\theta}}_j=\bar{\boldsymbol{\theta}}_j}
=
\lambda_i w_{ij}B_{ij}
$$

Hence the $(i,j)$-th block of $\mathbf{J}$ is
$$
J_{ij} = -H_{ii}^{-1}\, \lambda_i w_{ij}\, B_{ij}.
$$
Taking spectral norms,
\begin{equation}
\label{eq:Jij_norm_bound}
\|J_{ij}\|_2 \le \|H_{ii}^{-1}\|_2 \cdot \lambda_i w_{ij} \cdot \|B_{ij}\|_2.
\end{equation}

\emph{Bounding $\|H_{ii}^{-1}\|_2$.}  
From the proof of Lemma~\ref{lem:well_posedness}, $H_{ii} \preceq -\bigl((1+\lambda_i)\alpha_i - \lambda_i L_i \varepsilon_i\bigr) I$. 

Since $\alpha_i > L_i \varepsilon_i$, the bracketed quantity is positive, 
so
\begin{equation}
\label{eq:Hii_inv_bound}
\|H_{ii}^{-1}\|_2 
\le \frac{1}{(1+\lambda_i)\alpha_i - \lambda_i L_i \varepsilon_i}.
\end{equation}

\emph{Bounding $\|B_{ij}\|_2$.}
By assumption, $\|B_{ij}\|_2 \le M_i$.

Combining \eqref{eq:Jij_norm_bound}, \eqref{eq:Hii_inv_bound}, and 
$\|B_{ij}\|_2 \le M_i$:
$$
\|J_{ij}\|_2 \le 
\frac{\lambda_i w_{ij} M_i}{(1+\lambda_i)\alpha_i - \lambda_i L_i \varepsilon_i} 
= \frac{\lambda_i M_i}{\alpha_i + \lambda_i(\alpha_i - L_i \varepsilon_i)} \cdot w_{ij} 
= \beta_i w_{ij}.
$$
\end{proof}

\begin{lemma}[Comparison principle $\rho(\mathbf{J}) \le \rho(\mathbf{D})$]
\label{lem:J_vs_D}
Let $\mathbf{J} = \nabla_\Theta \mathcal{G}(\bar\Theta)$ with block decomposition satisfying $\|J_{ij}\|_2 \le \beta_i w_{ij}$, as established in Proposition~\ref{subprop:jac-block}, and let $\mathbf{D} := \mathrm{diag}(\beta_1, \dots, \beta_K) \mathbf{W}$. Then $\rho(\mathbf{J}) \le \rho(\mathbf{D}).$
\end{lemma}

\begin{proof}
We first show by induction that for all $t \ge 1$,
\begin{equation}
\label{eq:power_bound}
\|(\mathbf{J}^t)_{ij}\|_2 \le (\mathbf{D}^t)_{ij}, 
\qquad \forall i, j \in \mathcal{V}.
\end{equation}

\emph{Base case} ($t = 1$): $\|J_{ij}\|_2 \le \beta_i w_{ij} = D_{ij}$ 
by Proposition~\ref{subprop:jac-block}.

\emph{Inductive step}: assume \eqref{eq:power_bound} holds for some $t \ge 1$. Then for the $(i,j)$-th block of $\mathbf{J}^{t+1} = \mathbf{J} \cdot \mathbf{J}^t$,
\begin{align*}
\|(\mathbf{J}^{t+1})_{ij}\|_2 
&= \Bigl\|\sum_{k \in \mathcal{V}} J_{ik}\, (\mathbf{J}^t)_{kj}\Bigr\|_2 
\le \sum_k \|J_{ik}\|_2\, \|(\mathbf{J}^t)_{kj}\|_2 \\
&\le \sum_k D_{ik}\, (\mathbf{D}^t)_{kj} 
= (\mathbf{D}^{t+1})_{ij},
\end{align*}
where the first inequality uses triangle inequality and submultiplicativity 
of the spectral norm, and the seconduses the inductive hypothesis.

To convert the entrywise block bound to a spectral norm bound on 
$\mathbf{J}^t$, using the Frobenius norm $\|\cdot\|_F$:
$$
\|\mathbf{J}^t\|_2 \le \|\mathbf{J}^t\|_F 
= \sqrt{\sum_{i, j} \|(\mathbf{J}^t)_{ij}\|_F^2} 
$$
$d_{\max} := \max_i d_i$ accounts for block dimensions, using 
$\|M\|_F \le \sqrt{d}\, \|M\|_2$ for any $d \times d$ matrix $M$:
$$
\sqrt{\sum_{i, j} \|(\mathbf{J}^t)_{ij}\|_F^2} 
\le \sqrt{d_{\max}} \sqrt{\sum_{i, j} \|(\mathbf{J}^t)_{ij}\|_2^2} 
\le \sqrt{d_{\max}}\, \sqrt{\sum_{i, j} (\mathbf{D}^t)_{ij}^2} 
= \sqrt{d_{\max}}\, \|\mathbf{D}^t\|_F
$$

Combined with $\|\mathbf{D}^t\|_F \le \sqrt{K}\, \|\mathbf{D}^t\|_2$:
$$
\|\mathbf{J}^t\|_2 \le \sqrt{K\, d_{\max}}\, \|\mathbf{D}^t\|_2.
$$

Applying Gelfand's formula \citep{matrix}:
$$
\rho(\mathbf{J}) = \lim_{t \to \infty} \|\mathbf{J}^t\|_2^{1/t} 
\le \lim_{t \to \infty} (K\, d_{\max})^{1/(2t)}\, \|\mathbf{D}^t\|_2^{1/t} 
= \rho(\mathbf{D}),
$$
since $(K\, d_{\max})^{1/(2t)} \to 1$.
\end{proof}

\Stability*
\begin{proof}
By Proposition~\ref{subprop:jac-block}, $\mathcal{G}$ is $C^1$ in a neighborhood of $\bar\Theta$ and the Jacobian $\mathbf{J} := \nabla_\Theta \mathcal{G}(\bar\Theta)$ is well-defined. 

By Lemma~\ref{lem:J_vs_D}, $\rho(\mathbf{J}) \le \rho(\mathbf{D})$.

Combined with the hypothesis $\rho(\mathbf{D}) < 1$, we have $\rho(\mathbf{J}) < 1$, which implies that $\bar\Theta$ is a locally asymptotically stable fixed point.

\emph{Choice of contraction norm.}
For any $\delta > 0$ with $\rho(\mathbf{J}) + 2\delta < 1$, there exists a matrix norm $|\!|\!|\cdot|\!|\!|$ on $\prod_i \Theta_i$ such that $|\!|\!|\mathbf{J}|\!|\!| \le \rho(\mathbf{J}) + \delta < 1$. 
(This is a standard consequence of Gelfand's formula \citep{matrix}: such a norm exists because $\rho(\mathbf{J})$ is the infimum over all sub-multiplicative matrix norms of $\mathbf{J}$.)

\emph{Contraction neighborhood.}
Since $\nabla \mathcal{G}$ is continuous at $\bar\Theta$ with 
$\nabla \mathcal{G}(\bar\Theta) = \mathbf{J}$, there exists a 
neighborhood $\mathcal{U}$ of $\bar\Theta$ such that
\begin{equation}
\label{eq:continuity-slack}
|\!|\!|\nabla \mathcal{G}(\Theta) - \mathbf{J}|\!|\!| \le \delta, 
\qquad \forall \Theta \in \mathcal{U}.
\end{equation}
Combining $|\!|\!|\mathbf{J}|\!|\!| \le \rho(\mathbf{J}) + \delta$ and~\eqref{eq:continuity-slack} 
via the triangle inequality,
$$
|\!|\!|\nabla \mathcal{G}(\Theta)|\!|\!| 
\le |\!|\!|\mathbf{J}|\!|\!| + |\!|\!|\nabla \mathcal{G}(\Theta) - \mathbf{J}|\!|\!| 
\le \rho(\mathbf{J}) + 2\delta =: \rho_{\mathcal{G}} < 1, 
\qquad \forall \Theta \in \mathcal{U}.
$$

By the mean-value inequality, for all $\Theta, \Theta' \in \mathcal{U}$ 
with the segment $[\Theta, \Theta'] \subset \mathcal{U}$,
$$
|\!|\!|\mathcal{G}(\Theta) - \mathcal{G}(\Theta')|\!|\!| 
\le \rho_{\mathcal{G}} |\!|\!|\Theta - \Theta'|\!|\!|.
$$
Taking $\mathcal{U}$ to be a convex ball in the $|\!|\!|\cdot|\!|\!|$ 
metric, and shrinking if needed so 
$\mathcal{G}(\mathcal{U}) \subseteq \mathcal{U}$ (which is possible 
since $\mathcal{G}(\bar\Theta) = \bar\Theta$ and $\mathcal{G}$ is 
continuous), $\mathcal{G}$ is a contraction self-map on $\mathcal{U}$.

\emph{Convergence.}
By Banach's fixed-point theorem \citep{Banach}, for any $\Theta^0 \in \mathcal{U}$, the iterates $\Theta^{t+1} = \mathcal{G}(\Theta^t)$ converge to the unique fixed point of $\mathcal{G}$ in $\mathcal{U}$, which is $\bar\Theta$. 

Moreover,
$$
|\!|\!|\Theta^t - \bar\Theta|\!|\!| 
\le \rho_{\mathcal{G}}^t |\!|\!|\Theta^0 - \bar\Theta|\!|\!|, 
\quad \forall t \ge 0.
$$
Since all norms on finite-dimensional $\prod_i \Theta_i$ are equivalent, 
there exist constants $c_1, c_2 > 0$ with 
$c_1 \|\cdot\|_2 \le |\!|\!|\cdot|\!|\!| \le c_2 \|\cdot\|_2$. Hence,
$$
\|\Theta^t - \bar\Theta\|_2 
\le \frac{c_2}{c_1}\, \rho_{\mathcal{G}}^t \|\Theta^0 - \bar\Theta\|_2,
$$
establishing geometric convergence in $\|\cdot\|_2$ and local asymptotic 
stability of $\bar\Theta$.

\emph{Asymptotic linear rate.}
From the above bound,
$$
\|\Theta^t - \bar\Theta\|_2^{1/t} 
\le \frac{c_2}{c_1}^{1/t} \rho_{\mathcal{G}}.
$$
Taking $t \to \infty$, $\frac{c_2}{c_1}^{1/t} \to 1$, so
$$
\limsup_{t \to \infty} \|\Theta^t - \bar\Theta\|_2^{1/t} 
\le \rho_{\mathcal{G}} = \rho(\mathbf{J}) + 2\delta.
$$
Since $\delta > 0$ was arbitrary (subject to 
$\rho(\mathbf{J}) + 2\delta < 1$), letting $\delta \to 0$:
$$
\limsup_{t \to \infty} \|\Theta^t - \bar\Theta\|_2^{1/t} 
\le \rho(\mathbf{J}) \le \rho(\mathbf{D}),
$$
where the last inequality uses Lemma~\ref{lem:J_vs_D}.
\end{proof}

\subsection{Proofs of Corollary~\ref{cor:norm-certificate}}
\label{subsec:profs_norm-certificate}
\Normcertificate*
\begin{proof}
We prove two parts.

\emph{Part 1: $\beta_{\max} \|\mathbf{W}\|_2 < 1 \Rightarrow \rho(\mathbf{D}) < 1$.}
Since $\mathbf{D} = \mathrm{diag}(\beta_1, \dots, \beta_K) \mathbf{W}$ 
with $\beta_i \ge 0$,
$$
\|\mathbf{D}\|_2 
\le \|\mathrm{diag}(\beta_1, \dots, \beta_K)\|_2 \cdot \|\mathbf{W}\|_2 
= \beta_{\max} \|\mathbf{W}\|_2.
$$
Combined with $\rho(\mathbf{D}) \le \|\mathbf{D}\|_2$:
$$
\rho(\mathbf{D}) \le \|\mathbf{D}\|_2 \le \beta_{\max} \|\mathbf{W}\|_2 < 1.
$$
Local asymptotic stability of $\bar\Theta$ then follows from 
Theorem~\ref{thm:stability}.

\emph{Part 2: Equivalence with the $\lambda_i$ condition.}

The condition $\beta_{\max} \|\mathbf{W}\|_2 < 1$ is equivalent to 
$\beta_i \|\mathbf{W}\|_2 < 1$ for all $i \in \mathcal{V}$.

Substituting 
$\beta_i = \frac{\lambda_i M_i } {(\alpha_i + \lambda_i(\alpha_i - L_i \varepsilon_i))}$:
$$
\frac{\lambda_i M_i}
{\alpha_i + \lambda_i(\alpha_i - L_i \varepsilon_i)} \|\mathbf{W}\|_2 < 1.
$$

Since $\alpha_i > L_i \varepsilon_i$ implies the denominator is positive, 
cross-multiplying gives
$$
\lambda_i M_i \|\mathbf{W}\|_2 < \alpha_i + \lambda_i(\alpha_i - L_i \varepsilon_i).
$$
Isolating $\lambda_i$ on the left:
$$
\lambda_i \bigl[M_i \|\mathbf{W}\|_2 - (\alpha_i - L_i \varepsilon_i)\bigr] 
< \alpha_i.
$$
Since $M_i  > 0$ and $M_i\|\mathbf W\|_2 > \alpha_i-L_i\varepsilon_i$
$$
\lambda_i < \frac{\alpha_i}
{M_i \|\mathbf{W}\|_2 - (\alpha_i - L_i \varepsilon_i)}.
$$
\end{proof}

\begin{proposition}[Single-model recovery]
\label{prop:single_model}
Suppose $K = 1$, $w_{11} = 1$. Under assumptions with the regularity margin 
$\alpha_1 > L_1 \varepsilon_1$,  the normalized self-coupling in the single-model Jacobian calculation is bounded by $M_1 := \alpha_1 + L_1 \varepsilon_1$, and the single-model stability
condition reduces to $\lambda_1 < \frac{\alpha_1}{2 L_1 \varepsilon_1}$.
\end{proposition}
\begin{proof}
When $K=1$ and $w_{11}=1$, the networked retraining map reduces to the
standard single-model self-consuming update
\[
G_\lambda(\theta)\in \operatorname{local\text{-}argmax}_{\theta'}
\Bigl[
H_1(\theta')+\lambda H_2(\theta,\theta')
\Bigr],
\]
where
\[
H_1(\theta')=\mathbb E_{x\sim p^{\rm data}}[\log p_{\theta'}(x)],
\qquad
H_2(\theta,\theta')=\mathbb E_{x\sim p_\theta}[\log p_{\theta'}(x)].
\]
Let
\[
A:=\nabla^2_{\theta',\theta'}H_1(\bar\theta_1),
\qquad
B:=\nabla^2_{\theta,\theta'}H_2(\bar\theta_1,\bar\theta_1).
\]
By the probability-conservation identity used in single-model
self-consuming analyses,
\[
\nabla^2_{\theta',\theta}H_2(\bar\theta_1,\bar\theta_1)
=
-\nabla^2_{\theta',\theta'}H_2(\bar\theta_1,\bar\theta_1).
\]
Therefore, applying the implicit-function calculation for the local update
branch gives
\[
J_G(\bar\theta_1)
=
(I+\lambda_1 A^{-1}B)^{-1}\lambda_1 A^{-1}B .
\]

Next, Assumption~\ref{asp:strong_concave} gives
\[
\|A^{-1}\|_2\le \frac{1}{\alpha_1}.
\]
Moreover, by Assumption~\ref{asp:hessian_stability},
\[
\|B-A\|_2
\le
L_1 d_{\mathcal W}(p_{\bar\theta_1},p_1^{\rm data})
=
L_1\varepsilon_1 .
\]
Hence
\[
\|A^{-1}B-I\|_2
=
\|A^{-1}(B-A)\|_2
\le
\frac{L_1\varepsilon_1}{\alpha_1},
\]
and
\[
\|A^{-1}B\|_2
\le
1+\frac{L_1\varepsilon_1}{\alpha_1}.
\]
Equivalently, the normalized coupling can be bounded by
\[
\alpha_1\|A^{-1}B\|_2
\le
\alpha_1+L_1\varepsilon_1.
\]
Thus, in the single-model self-loop case, the coupling term appearing in the
normalized single-model Jacobian calculation is controlled by
$M_1:=\alpha_1+L_1\varepsilon_1$.

Since $\alpha_1 > L_1\varepsilon_1$, we have
$\|A^{-1}B-I\|_2 \le \frac{L_1\varepsilon_1}{\alpha_1} < 1$.
Moreover, $\frac{\lambda_1}{(1+\lambda_1)}<1$ for $\lambda_1\ge 0$. Hence
\[
\frac{\lambda_1}{1+\lambda_1}\|A^{-1}B-I\|_2 < 1.
\]

Using
\[
J_G(\bar\theta_1)
=
(I+\lambda_1 A^{-1}B)^{-1}\lambda_1 A^{-1}B,
\]
we bound
\[
\|J_G(\bar\theta_1)\|_2
\le
\|(I+\lambda_1 A^{-1}B)^{-1}\|_2
\cdot
\lambda_1\|A^{-1}B\|_2 .
\]
Since
\[
I+\lambda_1 A^{-1}B
=
(1+\lambda_1)I+\lambda_1(A^{-1}B-I),
\]

we have, whenever
\[
\frac{\lambda_1}{1+\lambda_1}
\|A^{-1}B-I\|_2<1,
\]
that
\[
\|(I+\lambda_1 A^{-1}B)^{-1}\|_2
\le
\frac{1}
{1+\lambda_1-\lambda_1\frac{L_1\varepsilon_1}{\alpha_1}}.
\]
Therefore,
\[
\|J_G(\bar\theta_1)\|_2
\le
\frac{
\lambda_1(1+\frac{L_1\varepsilon_1}{\alpha_1})
}
{
1+\lambda_1-\lambda_1\frac{L_1\varepsilon_1}{\alpha_1}
}.
\]
A sufficient condition for the right-hand side to be smaller than one is
\[
\lambda_1
\left(1+\frac{L_1\varepsilon_1}{\alpha_1}\right)
<
1+\lambda_1-\lambda_1\frac{L_1\varepsilon_1}{\alpha_1},
\]
which is equivalent to
\[
2\lambda_1\frac{L_1\varepsilon_1}{\alpha_1}<1.
\]
Thus
\[
\lambda_1<\frac{\alpha_1}{2L_1\varepsilon_1}.
\]
This recovers the single-model stability threshold.
\end{proof}

\subsection{Proofs of Proposition \ref{prop:cycle_gain} and Corollary~\ref{cor:necessary-cycle} }
\label{subsec:profs_cycle_gain}
\cyclegain*
\begin{proof}
Under the convention $w_{ij} = $ weight of edge $j \to i$ (i.e., row $i$ of $\mathbf{W}$ represents incoming signals to model $i$), we have $\mathbf{D}_{ij} = \beta_i w_{ij}$ associated with edge $j \to i$. 

For a nonnegative matrix $\mathbf{D}$, the $(a, b)$-th entry of 
$\mathbf{D}^t$ admits the path-sum interpretation
$$
(\mathbf{D}^t)_{ab} = \sum_{a \leftarrow k_1 \leftarrow \cdots \leftarrow k_{t-1} \leftarrow b} 
\prod \text{(edge weights along the path)},
$$
where the sum is over all directed paths from $b$ to $a$ of length $t$.

Consider a simple directed cycle 
$C = (i_\ell \to i_{\ell-1} \to \cdots \to i_1 \to i_\ell)$. Traversing $C$ exactly $m$ times yields a directed sequence $(i_\ell,  \dots, i_1, i_\ell, \dots, i_1, \dots, i_\ell)$ of length $m\ell + 1$, constituting a valid path from $i_1$ to $i_1$ of length $m\ell$. 

Since $\mathbf{D}$ is nonnegative, the entry $(\mathbf{D}^{m\ell})_{i_1, i_1}$ is lower-bounded by the weight of this specific path:
$$
(\mathbf{D}^{m\ell})_{i_1, i_1} 
\ge \left( \prod_{r=1}^{\ell} D_{i_r, i_{r+1}} \right)^m,
$$
where $i_{\ell+1} = i_1$.

The induced $\ell_\infty$ norm satisfies 
$\|\mathbf{D}^{m\ell}\|_\infty = \max_a \sum_b (\mathbf{D}^{m\ell})_{ab}$, 
which is at least any single entry:
$$
\|\mathbf{D}^{m\ell}\|_\infty 
\ge (\mathbf{D}^{m\ell})_{i_1, i_1} 
\ge \left( \prod_{r=1}^{\ell} D_{i_r, i_{r+1}} \right)^m.
$$
Taking the $(m\ell)$-th root:
$$
\|\mathbf{D}^{m\ell}\|_\infty^{1/(m\ell)} 
\ge \left( \prod_{r=1}^{\ell} D_{i_r, i_{r+1}} \right)^{1/\ell} 
= \gamma(C),
$$
using the definition of cycle gain.

By Gelfand's formula, $\rho(\mathbf{D}) = \lim_{t \to \infty} \|\mathbf{D}^t\|_\infty^{1/t}$. Taking the limit $m \to \infty$ along the subsequence $t = m\ell$:
$$
\rho(\mathbf{D}) = \lim_{m \to \infty} \|\mathbf{D}^{m\ell}\|_\infty^{1/(m\ell)} 
\ge \gamma(C).
$$
Since $C$ was arbitrary, $\rho(\mathbf{D}) \ge \max_{m \in [M]} \gamma(C_m)$.
\end{proof}

\Necessarycycle*

\begin{proof}
We prove the contrapositive. Suppose there exists a simple directed cycle
\[
C=(i_\ell\to i_{\ell-1}\to\cdots\to i_1\to i_\ell)
\]
such that
\[
\prod_{r=1}^{\ell}D_{i_r i_{r+1}}\ge 1,
\qquad i_{\ell+1}=i_1 .
\]
Equivalently, $\gamma(C)\ge 1$. By Proposition~\ref{prop:cycle_gain},
\[
\rho(\mathbf D)\ge \gamma(C)\ge 1.
\]
Therefore $\rho(\mathbf D)<1$ cannot hold. Taking the contrapositive, if
$\rho(\mathbf D)<1$, then every simple directed cycle $C$ must satisfy
\[
\prod_{r=1}^{\ell}D_{i_r i_{r+1}}<1.
\]
Using $D_{ij}=\beta_i w_{ij}$, this is
\[
\prod_{r=1}^{\ell}\beta_{i_r}w_{i_r i_{r+1}}<1.
\]
In particular, for a self-loop $C=(i\to i)$, the condition becomes
\[
\beta_i w_{ii}<1.
\]
\end{proof}

\subsection{Proofs of Proposition \ref{prop:global_solution}}
\label{subsec:proof_prop:global_solution}

\begin{proposition}[Fixed-point representation mixture identity]
\label{prop:fixed_point_mixture}
Suppose assumptions hold with $\alpha_i > L_i \varepsilon_i$ for all 
$i \in \mathcal{V}$. Then at any fixed point $\bar\Theta$ of $\mathcal{G}$, the representation $\bar{\boldsymbol{f}}_i$ admits the decomposition
\begin{equation}
\label{eq:fp-mixture}
\bar{\boldsymbol{f}}_i 
= \frac{1}{1+\lambda_i} \boldsymbol{f}_i^{\mathrm{data}} 
+ \frac{\lambda_i}{1+\lambda_i} \sum_{j \in \mathcal{V}} w_{ij}\, \bar{\boldsymbol{f}}_j 
+ \bar{\boldsymbol{r}}_i,
\end{equation}
where the residual $\bar{\boldsymbol{r}}_i \in \mathbb{R}^D$ satisfies
\begin{equation}
\label{eq:fp-residual-bound}
\|\bar{\boldsymbol{r}}_i\|_2 \le \frac{1 + 2\lambda_i}{1+\lambda_i}\, \varepsilon_i.
\end{equation}
\end{proposition}

\begin{proof}
\emph{Step 1: Mixture distribution and its representation.}
At the fixed point $\bar\Theta$, the synthetic-data training mixture for model $i$ is
$$
\bar{p}_i(x) := \frac{1}{1+\lambda_i} p_i^{\mathrm{data}}(x) 
+ \frac{\lambda_i}{1+\lambda_i} \sum_{j \in \mathcal{V}} w_{ij}\, 
p_{\bar{\boldsymbol{\theta}}_j}(x).
$$
Applying the operator $\int \phi(x) (\cdot) dx$:
\begin{equation}
\begin{aligned}
    \mathbb{E}_{x \sim \bar{p}_i}[\phi(x)]  =& \int \phi(x) \bar{p}_i(x) dx \notag= \int \phi(x) \left[ \frac{1}{1+\lambda_i} p^{data}_{i}(x) + \frac{\lambda_i}{1+\lambda_i} \sum_{j=1}^K w_{ij} p_{\bar{\boldsymbol{\theta}}_j}(x) \right] dx \notag \\
    &= \frac{1}{1+\lambda_i} \underbrace{\int \phi(x) p^{data}_{i}(x) dx}_{\boldsymbol{f}_{i}^{\mathrm{data}}} + \frac{\lambda_i}{1+\lambda_i} \sum_{j=1}^K w_{ij} \underbrace{\int \phi(x) p_{{\boldsymbol{\theta}}_j^{t-1}} dx}_{\bar{\boldsymbol{f}}_j} \notag
\end{aligned}
\end{equation}

So,
\begin{equation}
\label{eq:mixture-feature}
\mathbb{E}_{x \sim \bar{p}_i}[\phi(x)] 
= \frac{1}{1+\lambda_i} \boldsymbol{f}_i^{\mathrm{data}} 
+ \frac{\lambda_i}{1+\lambda_i} \sum_{j \in \mathcal{V}} w_{ij}\, \bar{\boldsymbol{f}}_j.
\end{equation}

\emph{Step 2: Residual definition.}
Define the representation residual
\begin{equation}
\label{eq:residual-def}
\bar{\boldsymbol{r}}_i := \bar{\boldsymbol{f}}_i 
- \mathbb{E}_{x \sim \bar{p}_i}[\phi(x)] 
= \mathbb{E}_{x \sim p_{\bar{\boldsymbol{\theta}}_i}}[\phi(x)] 
- \mathbb{E}_{x \sim \bar{p}_i}[\phi(x)].
\end{equation}
Substituting~\eqref{eq:mixture-feature} into~\eqref{eq:residual-def}:
$$
\bar{\boldsymbol{f}}_i 
= \frac{1}{1+\lambda_i} \boldsymbol{f}_i^{\mathrm{data}} 
+ \frac{\lambda_i}{1+\lambda_i} \sum_{j \in \mathcal{V}} w_{ij}\, \bar{\boldsymbol{f}}_j 
+ \bar{\boldsymbol{r}}_i
$$

\emph{Step 3: Bound the residual.}
For any 1-Lipschitz scalar function $g: \mathcal{X} \to \mathbb{R}$ and 
any pair of probability measures $p, q$ on $\mathcal{X}$, the 
Kantorovich--Rubinstein duality \citep{Villani} gives
$$
\bigl|\mathbb{E}_p[g] - \mathbb{E}_q[g]\bigr| \le d_W(p, q).
$$
Since $\phi: \mathcal{X} \to \mathbb{R}^D$ is 1-Lipschitz with respect 
to the Euclidean norm on $\mathbb{R}^D$, for any unit vector 
$u \in \mathbb{R}^D$, the scalar function 
$x \mapsto \langle u, \phi(x) \rangle$ is 1-Lipschitz on $\mathcal{X}$. 

Applying Kantorovich--Rubinstein with 
$p = p_{\bar{\boldsymbol{\theta}}_i}$ and $q = \bar{p}_i$:
\begin{align*}
\langle u, \bar{\boldsymbol{r}}_i \rangle 
&= \bigl\langle u,\ \mathbb{E}_{x \sim p_{\bar{\boldsymbol{\theta}}_i}}[\phi(x)] 
- \mathbb{E}_{x \sim \bar{p}_i}[\phi(x)] \bigr\rangle 
&& \text{(residual definition)} \\
&= \bigl\langle u,\ \mathbb{E}_{x \sim p_{\bar{\boldsymbol{\theta}}_i}}[\phi(x)] \bigr\rangle 
- \bigl\langle u,\ \mathbb{E}_{x \sim \bar{p}_i}[\phi(x)] \bigr\rangle 
&& \text{(linearity of inner product)} \\
&= \mathbb{E}_{x \sim p_{\bar{\boldsymbol{\theta}}_i}}[\langle u, \phi(x) \rangle] 
- \mathbb{E}_{x \sim \bar{p}_i}[\langle u, \phi(x) \rangle] 
&& \text{(swap inner product and expectation)} \\
&= \mathbb{E}_{x \sim p_{\bar{\boldsymbol{\theta}}_i}}[g_u(x)] 
- \mathbb{E}_{x \sim \bar{p}_i}[g_u(x)] 
&& \text{(definition of } g_u \text{)} \\
&\le d_W(p_{\bar{\boldsymbol{\theta}}_i}, \bar{p}_i) 
&& \text{(Kantorovich--Rubinstein)}.
\end{align*}

Taking $u = \bar{\boldsymbol{r}}_i / \|\bar{\boldsymbol{r}}_i\|$,
\begin{equation}
\label{eq:r_to_dW}
\|\bar{\boldsymbol{r}}_i\| \le d_W(p_{\bar{\boldsymbol{\theta}}_i}, \bar{p}_i).
\end{equation}

\emph{Step 4: Bound $d_W(p_{\bar{\boldsymbol{\theta}}_i}, \bar{p}_i)$.}
By the triangle inequality of Wasserstein distance:
\begin{equation}
\label{eq:dW-triangle}
d_W(p_{\bar{\boldsymbol{\theta}}_i}, \bar{p}_i) 
\le \underbrace{d_W(p_{\bar{\boldsymbol{\theta}}_i}, p_i^{\mathrm{data}})}_{=: \, A} 
+ \underbrace{d_W(p_i^{\mathrm{data}}, \bar{p}_i)}_{=: \, B}.
\end{equation}

\emph{Bounding $A$.} By the definition of $\varepsilon_i$ in 
Assumption~\ref{asp:hessian_stability}, $A = d_W(p_{\bar{\boldsymbol{\theta}}_i}, p_i^{\mathrm{data}}) \le \varepsilon_i.$

\emph{Bounding $B$.} Using the convexity of Wasserstein distance in each argument  and the definition of $\bar{p}_i$:
\begin{align*}
B &= d_W\!\left(p_i^{\mathrm{data}},\ 
\frac{1}{1+\lambda_i} p_i^{\mathrm{data}} 
+ \frac{\lambda_i}{1+\lambda_i} \sum_j w_{ij} p_{\bar{\boldsymbol{\theta}}_j}\right) \\
&\le \frac{1}{1+\lambda_i}\, d_W(p_i^{\mathrm{data}}, p_i^{\mathrm{data}}) 
+ \frac{\lambda_i}{1+\lambda_i} \sum_j w_{ij}\, 
d_W(p_i^{\mathrm{data}}, p_{\bar{\boldsymbol{\theta}}_j}) \\
&= \frac{\lambda_i}{1+\lambda_i} \sum_j w_{ij}\, 
d_W(p_i^{\mathrm{data}}, p_{\bar{\boldsymbol{\theta}}_j}),
\end{align*}
where the first term vanishes since $d_W(p_i^{\mathrm{data}}, p_i^{\mathrm{data}}) = 0$. 

By the definition of $\varepsilon_i = \max_{j: w_{ij}>0} 
d_W(p_{\bar{\boldsymbol{\theta}}_j}, p_i^{\mathrm{data}})$ and 
$\sum_j w_{ij} = 1$ (row-stochasticity of $\mathbf{W}$):
$$
B \le \frac{\lambda_i}{1+\lambda_i} \cdot \varepsilon_i.
$$

\emph{Combining.} Substituting the bounds for $A$ and $B$ 
into~\eqref{eq:dW-triangle}:
$$
d_W(p_{\bar{\boldsymbol{\theta}}_i}, \bar{p}_i) 
\le \varepsilon_i + \frac{\lambda_i}{1+\lambda_i}\, \varepsilon_i 
= \frac{1 + 2\lambda_i}{1+\lambda_i}\, \varepsilon_i.
$$
Combined with~\eqref{eq:r_to_dW},
$$
\|\bar{\boldsymbol{r}}_i\| \le \frac{1 + 2\lambda_i}{1+\lambda_i}\, \varepsilon_i.
$$
\end{proof}

\globalsolution*
\begin{proof}
By Proposition~\ref{prop:fixed_point_mixture}, for each 
$i \in \mathcal{V}$,
$$
\bar{\boldsymbol{f}}_i 
= \frac{1}{1+\lambda_i} \boldsymbol{f}_i^{\mathrm{data}} 
+ \frac{\lambda_i}{1+\lambda_i} \sum_{j \in \mathcal{V}} w_{ij}\, \bar{\boldsymbol{f}}_j 
+ \bar{\boldsymbol{r}}_i.
$$
Multiplying both sides by $(1+\lambda_i)$ and rearranging:
$$
(1+\lambda_i)\bar{\boldsymbol{f}}_i 
- \lambda_i \sum_{j \in \mathcal{V}} w_{ij}\, \bar{\boldsymbol{f}}_j 
= \boldsymbol{f}_i^{\mathrm{data}} + (1+\lambda_i)\bar{\boldsymbol{r}}_i.
$$
Stacking over $i \in \mathcal{V}$ in matrix form:
$$
(\mathbf{I} + \mathbf{\Lambda})\bar{\mathbf{F}} 
- \mathbf{\Lambda}\mathbf{W}\bar{\mathbf{F}} 
= \mathbf{F}_{\mathrm{data}} + (\mathbf{I} + \mathbf{\Lambda})\bar{\mathbf{R}}.
$$
Since $(\mathbf{I} + \mathbf{\Lambda}) - \mathbf{\Lambda}\mathbf{W} 
= \mathbf{I} + \mathbf{\Lambda}(\mathbf{I} - \mathbf{W}) 
= \mathbf{I} + \mathbf{\Lambda}\mathbf{L}$:
\begin{equation}
\label{eq:matrix_eqn}
(\mathbf{I} + \mathbf{\Lambda}\mathbf{L})\bar{\mathbf{F}} 
= \mathbf{F}_{\mathrm{data}} + (\mathbf{I} + \mathbf{\Lambda})\bar{\mathbf{R}}.
\end{equation}

Let $\mathbf{M} := \mathbf{I} + \mathbf{\Lambda}\mathbf{L} 
= (\mathbf{I} + \mathbf{\Lambda}) - \mathbf{\Lambda}\mathbf{W}$. The 
entries of $\mathbf{M}$ are
$$
M_{ii} = 1 + \lambda_i - \lambda_i w_{ii} = 1 + \lambda_i(1 - w_{ii}), 
\qquad 
M_{ij} = -\lambda_i w_{ij} \quad (j \neq i).
$$
The off-diagonal absolute row sum is
$$
\sum_{j \neq i} |M_{ij}| 
= \lambda_i \sum_{j \neq i} w_{ij} 
= \lambda_i (1 - w_{ii}),
$$
using $\sum_j w_{ij} = 1$ (row-stochasticity of $\mathbf{W}$). Comparing 
with the diagonal:
$$
M_{ii} = 1 + \lambda_i(1 - w_{ii}) > \lambda_i(1 - w_{ii}) = \sum_{j \neq i} |M_{ij}|,
$$
so $\mathbf{M}$ is strictly diagonally dominant, $\mathbf{M}$ is invertible.

Multiplying both sides of~\eqref{eq:matrix_eqn} by 
$\mathbf{M}^{-1} = (\mathbf{I} + \mathbf{\Lambda}\mathbf{L})^{-1}$:
$$
\bar{\mathbf{F}} = (\mathbf{I} + \mathbf{\Lambda}\mathbf{L})^{-1}
\bigl[\mathbf{F}_{\mathrm{data}} + (\mathbf{I} + \mathbf{\Lambda})\bar{\mathbf{R}}\bigr].
$$
\end{proof}

\subsection{Proofs of Theorem~\ref{thm:general_diversity} and Corollary~\ref{cor:symmetric_bound}}
\label{subsec:proof_thm:diversity_bound}

\generaldiversity*
\begin{proof}
Let $\mathbf{M} := \mathbf{I} + \mathbf{\Lambda}\mathbf{L}$ and 
$\mathbf{P} := \mathbf{I}_K - \frac{1}{|\mathcal{V}|}\mathbf{1}\mathbf{1}^\top$ 
be the centering matrix.

For $\mathbf{A} \in \mathbb{R}^{K \times D}$ with rows $\boldsymbol{a}_i$ and column-mean $\boldsymbol{a}_{\mathrm{avg}} = \frac{1}{|\mathcal{V}|}\sum_i \boldsymbol{a}_i$, the rows of $\mathbf{P}\mathbf{A}$ are $\boldsymbol{a}_i - \boldsymbol{a}_{\mathrm{avg}}$, so
$$
\mathcal{D}_{\mathrm{out}} = \frac{1}{|\mathcal{V}|} \|\mathbf{P}\bar{\mathbf{F}}\|_F^2, 
\qquad 
\mathcal{D}_{\mathrm{data}} = \frac{1}{|\mathcal{V}|} \|\mathbf{P}\mathbf{F}_{\mathrm{data}}\|_F^2.
$$

By Proposition~\ref{prop:global_solution},
$$
\bar{\mathbf{F}} = \mathbf{M}^{-1}\mathbf{F}_{\mathrm{data}} 
+ \mathbf{M}^{-1}(\mathbf{I} + \mathbf{\Lambda})\bar{\mathbf{R}}.
$$
Since $\mathbf{W}$ is row-stochastic, $\mathbf{M}\mathbf{1} = \mathbf{1}$, 
hence $\mathbf{M}^{-1}\mathbf{1} = \mathbf{1}$. 

Decomposing 
$\mathbf{F}_{\mathrm{data}} = \mathbf{P}\mathbf{F}_{\mathrm{data}} + 
\tfrac{1}{|\mathcal{V}|}\mathbf{1}\mathbf{1}^\top \mathbf{F}_{\mathrm{data}}$ 
and applying $\mathbf{M}^{-1}$:
\begin{align*}
\mathbf{P}\mathbf{M}^{-1}\mathbf{F}_{\mathrm{data}} 
&= \mathbf{P}\mathbf{M}^{-1}\mathbf{P}\mathbf{F}_{\mathrm{data}} 
+ \tfrac{1}{|\mathcal{V}|}\mathbf{P}\mathbf{M}^{-1}\mathbf{1}\mathbf{1}^\top \mathbf{F}_{\mathrm{data}} \\
&= \mathbf{P}\mathbf{M}^{-1}\mathbf{P}\mathbf{F}_{\mathrm{data}} 
+ \tfrac{1}{|\mathcal{V}|}\mathbf{P}\mathbf{1}\mathbf{1}^\top \mathbf{F}_{\mathrm{data}} 
&& \text{(}\mathbf{M}^{-1}\mathbf{1} = \mathbf{1}\text{)} \\
&= \mathbf{P}\mathbf{M}^{-1}\mathbf{P}\mathbf{F}_{\mathrm{data}} 
&& \text{(}\mathbf{P}\mathbf{1} = 0\text{)}.
\end{align*}
Hence
$$
\mathbf{P}\bar{\mathbf{F}} 
= \mathbf{P}\mathbf{M}^{-1}\mathbf{P}\mathbf{F}_{\mathrm{data}} 
+ \mathbf{P}\mathbf{M}^{-1}(\mathbf{I} + \mathbf{\Lambda})\bar{\mathbf{R}}.
$$

Apply $\|a + b\|_F^2 \le 2\|a\|_F^2 + 2\|b\|_F^2$.
$$
\|\mathbf{P}\bar{\mathbf{F}}\|_F^2 
\le 2\|\mathbf{P}\mathbf{M}^{-1}\mathbf{P}\mathbf{F}_{\mathrm{data}}\|_F^2 
+ 2\|\mathbf{P}\mathbf{M}^{-1}(\mathbf{I} + \mathbf{\Lambda})\bar{\mathbf{R}}\|_F^2.
$$

Using $\|\mathbf{A}\mathbf{B}\|_F \le \|\mathbf{A}\|_2 \|\mathbf{B}\|_F$ 
and $\|\mathbf{P}\|_2 = 1$:
$$
\|\mathbf{P}\mathbf{M}^{-1}\mathbf{P}\mathbf{F}_{\mathrm{data}}\|_F^2 
\le \|\mathbf{M}^{-1}\|_2^2 \cdot \|\mathbf{P}\mathbf{F}_{\mathrm{data}}\|_F^2 
= |\mathcal{V}|\, \|\mathbf{M}^{-1}\|_2^2 \cdot \mathcal{D}_{\mathrm{data}}.
$$

Similarly,
$$
\|\mathbf{P}\mathbf{M}^{-1}(\mathbf{I} + \mathbf{\Lambda})\bar{\mathbf{R}}\|_F^2 
\le \|\mathbf{M}^{-1}\|_2^2 \cdot \|(\mathbf{I} + \mathbf{\Lambda})\bar{\mathbf{R}}\|_F^2.
$$
Since $(\mathbf{I} + \mathbf{\Lambda})\bar{\mathbf{R}}$ has rows 
$(1+\lambda_i)\bar{\boldsymbol{r}}_i$,
\[
\|(\mathbf{I} + \mathbf{\Lambda})\bar{\mathbf{R}}\|_F^2 
= \sum_{i \in \mathcal{V}} (1+\lambda_i)^2 \|\bar{\boldsymbol{r}}_i\|_2^2.
\]
By the residual bound from Proposition~\ref{prop:fixed_point_mixture}, 
$\|\bar{\boldsymbol{r}}_i\|\le \frac{1+2\lambda_i}{1+\lambda_i}\varepsilon_i$ . Hence
$$
(1+\lambda_i)^2 \|\bar{\boldsymbol{r}}_i\|^2 
\le (1+\lambda_i)^2 \cdot \frac{(1+2\lambda_i)^2}{(1+\lambda_i)^2}\, \varepsilon_i^2 
= (1+2\lambda_i)^2\, \varepsilon_i^2.
$$
Therefore
$$
\|(\mathbf{I} + \mathbf{\Lambda})\bar{\mathbf{R}}\|_F^2 
\le \sum_{i \in \mathcal{V}} (1+2\lambda_i)^2\, \varepsilon_i^2.
$$

Substituting into Step 3 and dividing by $|\mathcal{V}|$:
$$
\mathcal{D}_{\mathrm{out}} 
= \frac{1}{|\mathcal{V}|} \|\mathbf{P}\bar{\mathbf{F}}\|_F^2 
\le 2\|\mathbf{M}^{-1}\|_2^2 \biggl[\mathcal{D}_{\mathrm{data}} 
+ \frac{1}{|\mathcal{V}|}\sum_{i \in \mathcal{V}} (1+2\lambda_i)^2 \varepsilon_i^2\biggr].
$$
\end{proof}

\symmetricbound*
\begin{proof}
Under the assumption $\bar{\boldsymbol{r}}_i = 0$, hence $\bar{\mathbf{R}} = \mathbf{0}$, and 
Proposition~\ref{prop:global_solution} reduces to
$$
\bar{\mathbf{F}} = \mathbf{M}^{-1} \mathbf{F}_{\mathrm{data}}, 
\qquad \mathbf{M} := \mathbf{I} + \lambda \mathbf{L}.
$$

\emph{Step 1: Spectral analysis of $\mathbf{M}$.}
Since $\mathbf{W}$ is symmetric, $\mathbf{L} = \mathbf{I} - \mathbf{W}$ is symmetric. 

Because $\mathbf{W}$ is row-stochastic and nonnegative, $\mathbf{L}$ is positive semi-definite with eigenvalues $0 = \mu_1 < \mu_2 \le \cdots \le \mu_K$ and orthonormal eigenvectors $\mathbf{u}_1 = \frac{1}{\sqrt{K}}\mathbf{1}, \mathbf{u}_2, \ldots, \mathbf{u}_K$.

Hence $\mathbf{M}^{-1}$ has the same eigenvectors with eigenvalues $\frac{1}{1+\lambda \mu_k}$. The function $g(\mu) = \frac{1}{1+\lambda\mu}$ is monotonically decreasing for $\lambda > 0$, so $\mathbf{M}^{-1}$'s largest eigenvalue is $\frac{1}{1+\lambda\mu_1} = 1$ (along $\mathbf{u}_1$), and its second largest is $\frac{1}{1+\lambda\mu_2}$ (along $\mathbf{u}_2$).

\emph{Step 2: Centering and restriction to $\mathbf{1}^\perp$.}
Let $\mathbf{P} := \mathbf{I}_K - \frac{1}{K}\mathbf{1}\mathbf{1}^\top$ 
be the centering matrix. 
\[
\mathcal{D}_{\mathrm{out}} = \frac{1}{K}\|\mathbf{P}\bar{\mathbf{F}}\|_F^2, 
\qquad 
\mathcal{D}_{\mathrm{data}} = \frac{1}{K}\|\mathbf{P}\mathbf{F}_{\mathrm{data}}\|_F^2.
\]
Since $\mathbf{M}^{-1}\mathbf{1} = \mathbf{1}$ (the eigenvector for 
$\mu_1 = 0$), $\mathbf{M}^{-1}$ commutes with $\mathbf{P}$ in this 
symmetric case, and $\mathbf{P}$ projects onto $\mathbf{1}^\perp$:
\[
\mathbf{P}\bar{\mathbf{F}} = \mathbf{P}\mathbf{M}^{-1}\mathbf{F}_{\mathrm{data}} 
= \mathbf{M}^{-1}\mathbf{P}\mathbf{F}_{\mathrm{data}}.
\]

\emph{Step 3: Bound via restricted spectral norm.}
On $\mathbf{1}^\perp$, $\mathbf{M}^{-1}$ has spectral norm
\[
\|\mathbf{M}^{-1}|_{\mathbf{1}^\perp}\|_2 
= \max_{k \ge 2} \frac{1}{1 + \lambda \mu_k} 
= \frac{1}{1 + \lambda \mu_2}.
\]
Hence
\[
\|\mathbf{P}\bar{\mathbf{F}}\|_F^2 
= \|\mathbf{M}^{-1}\mathbf{P}\mathbf{F}_{\mathrm{data}}\|_F^2 
\le \|\mathbf{M}^{-1}|_{\mathbf{1}^\perp}\|_2^2 \cdot \|\mathbf{P}\mathbf{F}_{\mathrm{data}}\|_F^2 
= \frac{1}{(1+\lambda\mu_2)^2} \|\mathbf{P}\mathbf{F}_{\mathrm{data}}\|_F^2.
\]
Dividing both sides by $K$:
\[
\mathcal{D}_{\mathrm{out}} \le \frac{1}{(1+\lambda\mu_2)^2}\, \mathcal{D}_{\mathrm{data}}.
\]
\end{proof}
\section{Closed-Form Spectral Validation of the Networked Dynamics}
\label{sec:controlled}

To complement the neural generative-model experiments, we consider a controlled networked dynamical system in which the fixed point, Jacobian, comparison matrix, and cycle gains are all available in closed form. While the experiments in Sec.~\ref{sec:exp} test whether the qualitative mechanisms predicted by our theory persist under finite data, imperfect optimization, and model misspecification, the purpose of this experiment is different: it provides a direct numerical validation.

\textbf{Setup.}We consider $K=4$ interacting nodes, denoted by $\{A,B,C,D\}$, with each node represented by a two-dimensional state $\theta_i\in\mathbb{R}^2$. We assign the heterogeneous reference centroids
\begin{equation*}
    a_A =
    \begin{bmatrix}
        1 \\ 0
    \end{bmatrix},
    \qquad
    a_B =
    \begin{bmatrix}
        -1 \\ 0
    \end{bmatrix},
    \qquad
    a_C =
    \begin{bmatrix}
        0 \\ 1
    \end{bmatrix},
    \qquad
    a_D =
    \begin{bmatrix}
        0 \\ -1
    \end{bmatrix}
\end{equation*}
and let
\begin{equation*}
    A =
    \begin{bmatrix}
        a_A^\top \\
        a_B^\top \\
        a_C^\top \\
        a_D^\top
    \end{bmatrix}
    \in \mathbb{R}^{4\times 2}
\end{equation*}

To introduce controlled model heterogeneity, we associate each source node with a sensitivity coefficient $c = (1,\,1,\,1.8,\,1)^\top$ and $\Gamma = \operatorname{diag}(c)$. Thus, node $C$ has a larger local sensitivity than the other nodes. Importantly, the same sensitivity vector $c$ is used for every interaction topology; no topology-specific parameters are tuned.

For an interaction matrix $W$, the network evolves according to the affine update
\begin{equation*}
    \Theta^{t+1}
    =
    \frac{
        A + \lambda W\Gamma\Theta^t
    }{1+\lambda}
\end{equation*}
where $\Theta^t=[(\theta_A^t)^\top,\ldots,(\theta_D^t)^\top]^\top$ collects the four node states and $\lambda\geq 0$ controls the strength of the interaction term. As in the main experiments, we report this strength using the synthetic-data ratio $s := \frac{\lambda}{1+\lambda}\in[0,1)$.

We evaluate the first four interaction systems used in the main paper, as shown in Fig.~\ref{Appfig:HomoToy}. System~1 corresponds to isolated self-consumption, Systems~2 and~3 share the same ring topology but use different interaction strengths, and System~4 is fully connected. These are the same four controlled topologies used in the main experiments.

\textbf{Closed-form fixed point and exact Jacobian.} Whenever $(1+\lambda)I-\lambda W\Gamma$ is nonsingular, Eq.~\eqref{eq:f} admits the unique fixed point $\bar{\Theta} =\left[(1+\lambda)I-\lambda W\Gamma \right]^{-1} A$. Unlike the realistic generative-model experiments, this controlled setting therefore allows us to directly measure the distance to the fixed point, $\|\Theta^t-\bar{\Theta}\|_F$. And its Jacobian is constant and available exactly:$J=s(W\Gamma)\otimes I_2=sW\Gamma$. Consequently, the error $E^t:=\Theta^t-\bar{\Theta}$ follows $E^{t+1}=J E^t$. Hence, for a generic perturbation with a nonzero component in the dominant eigenspace, the asymptotic convergence factor is $\rho(J)$. All graph-level quantities, including the comparison matrix $D$ and the cycle gain $\gamma(C)$, follow the definitions in Sec.~\ref{sec:HomoData}.

\textbf{Direct validation of the spectral criterion}
We first fix $\lambda=1$, corresponding to a synthetic-data ratio $s=0.5$. For each interaction topology, we initialize the dynamics in a small neighborhood of the exact fixed point $\Theta^0=\bar{\Theta}+10^{-3} Z$, where the entries of $Z$ are independently sampled from a standard Gaussian distribution. We run $250$ iterations and record $e_t := \|\Theta^t-\bar{\Theta}\|_F$ .
To estimate the asymptotic contraction rate, we fit
\begin{equation*}
    \log e_t
    \approx
    b + t\log r_{\mathrm{emp}}
\end{equation*}
over the latter half of the trajectory before numerical saturation, and report the resulting $r_{\mathrm{emp}}$.

Tbl.~\ref{subtab:closed_form_spectral} reports the exact and empirical spectral quantities. The empirical convergence rate agrees with $\rho(J)$ to three decimal places in every system. Moreover, $\rho(J)\leq\rho(D)$ holds throughout, consistent with the comparison principle in Thm.~\ref{thm:stability}.

\begin{table}[htbp]
    \centering
    \small
    \caption{
    Closed-form spectral validation at $\lambda=1$ ($s=0.5$).  The measured asymptotic rate $r_{\mathrm{emp}}$ matches the exact Jacobian spectral radius $\rho(J)$, while $\rho(D)$ upper-bounds $\rho(J)$. The last column reports the largest cycle gain and verifies the lower bound in Prop.~\ref{prop:cycle_gain}.
    }
    \label{subtab:closed_form_spectral}
    \begin{tabular}{lccccc}
        \toprule
        System
        & $\|W\|_2$
        & $\rho(J)$
        & $r_{\mathrm{emp}}$
        & $\rho(D)$
        & $\max_C \gamma(C)$
        \\
        \midrule
        Sys.~1 (isolated)
        & 1.00
        & 0.900
        & 0.900
        & 0.900
        & 0.900
        \\
        Sys.~2 (weak ring)
        & 1.00
        & 0.606
        & 0.606
        & 0.700
        & 0.450
        \\
        Sys.~3 (strong ring)
        & 1.00
        & 0.582
        & 0.582
        & 0.674
        & 0.604
        \\
        Sys.~4 (dense)
        & 1.00
        & 0.600
        & 0.600
        & 0.900
        & 0.225
        \\
        \bottomrule
    \end{tabular}
\end{table}

All four systems have the same spectral norm $\|W\|_2=1$, yet their exact convergence factors range from $0.582$ to $0.900$. Thus, the interaction norm alone cannot distinguish their dynamics. In contrast, the topology-aware spectral quantities retain information about how heterogeneous local sensitivities are propagated through the network. This illustrates why the norm-based condition in Cor.~\ref{cor:norm-certificate} can be strictly more conservative than the spectral criterion of Thm.~\ref{thm:stability}..In all systems,$\max_C \gamma(C)\leq\rho(D)$, as predicted by Prop.~\ref{prop:cycle_gain}. The bound can be tight, as in the isolated system, or substantially loose in graphs containing
many simultaneously interacting paths.

\textbf{Critical synthetic-data ratio} Because both $J$ and $D$ scale linearly with $s=\lambda/(1+\lambda)$, the two stability thresholds are available in closed form. For the systems considered here,
$s_c^J=\frac{1}{\rho(W\Gamma)}$, $s_c^D= \frac{1}{\rho(\operatorname{diag}(M_1,\ldots,M_K)W)}$. Table~\ref{subtab:closed_form_threshold} reports these two thresholds. For every system, $s_c^D\leq s_c^J$, showing quantitatively that the $D$-based criterion lies on the conservative side of the exact Jacobian threshold.

\begin{table}[htpb]
    \centering
    \caption{ Critical synthetic-data ratio $s=\lambda/(1+\lambda)$. The comparison-matrix threshold $s_c^D$ is a sufficient stability certificate and is therefore no larger than the exact Jacobian threshold $s_c^J$ in all tested systems.
    }
    \label{subtab:closed_form_threshold}
    \begin{tabular}{lcc}
        \toprule
        System
        & $s_c^D$
        & $s_c^J$
        \\
        \midrule
        Sys.~1 (isolated)
        & 0.556
        & 0.556
        \\
        Sys.~2 (weak ring)
        & 0.714
        & 0.825
        \\
        Sys.~3 (strong ring)
        & 0.742
        & 0.860
        \\
        Sys.~4 (dense)
        & 0.556
        & 0.833
        \\
        \bottomrule
    \end{tabular}
\end{table}

The gap between the two thresholds quantifies the conservativeness introduced by replacing heterogeneous block sensitivities with the nonnegative comparison system. This gap depends on topology. For example, the bound is tight in the isolated System~1, whereas in the dense System~4 the sufficient threshold is $0.556$, compared with the exact Jacobian threshold $0.833$. Thus, the comparison matrix provides a tractable topology-aware certificate without requiring it to exactly recover the full Jacobian dynamics.

\textbf{Summary.} This controlled experiment complements the neural generative-model experiments in two respects. First, because the fixed point and Jacobian
are available exactly, it directly verifies that $\rho(J)$ governs the asymptotic convergence rate. Second, it numerically illustrates both levels of graph-based comparison developed in the theory. Together with the realistic experiments in Sec.~\ref{sec:exp}, these results distinguish two complementary forms of validation: the controlled system directly verifies the spectral predictions under an exactly analyzable update map, while the neural-model experiments test whether the same qualitative mechanisms remain visible under finite data, imperfect optimization, and heterogeneous model architectures.

\newpage
\section{Experiments with Synthetic Gaussian Data}
\label{sec:synex}

In this section, we present synthetic data experiments to empirically validate the theoretical results of the main paper, providing a low-dimensional illustration that complements large-scale image experiments.

\begin{table}[thbp!]
\begin{center}
\small
\caption{Model architectures and data access.}
\label{subtab:model_summary}
\begin{tabular}{c c c c c c }
\toprule
$\#$ &
Model &
Architecture & Stability &
Diveristy & 
Retrain Steps\\
\midrule
A & DDPM & MLP (hidden dimensions 128) & $t \in \{0,\cdots,7\}$ & $t\in\{0,4\}$ & $100$\\
B & CFM  & MLP & $t \in \{0,\cdots,7\}$       & $t \in \{2,6\}$ & $50$ \\
C & GAN  & MLP  & $t \in \{0,\cdots,7\}$      & $t \in \{3,5\}$ & $30$ \\
D & DDPM & MLP (hidden dimensions 256) & $t \in \{0,\cdots,7\}$ & $t \in \{1,7\}$ & $100$ \\
\bottomrule
\end{tabular}
\end{center}
\end{table}

\paragraph{Dataset.} The synthetic dataset is a two-dimensional dataset constructed from mixtures of isotropic Gaussian distributions. The global data distribution follows a standard 8-mode Gaussian, where eight mode centers are uniformly distributed on a circle of radius $2$. The coordinates of these centers are:
$\mu_t = 2 \times \left(\cos(\frac{t\pi}{4}), \sin(\frac{t\pi}{4})\right)$ where $t = 0,1,2,\dots,7$. We use two real-data assignments to evaluate different aspects of the coupled system. The same data for all models is used for stability experiments, and the heterogeneous assignment is used for diversity experiments, where each model $i$ is associated with a subset of these Gaussian modes, as specified in~\Cref{subtab:model_summary}. Each data point is independently and uniformly sampled from one of the assigned mode centers, with isotropic Gaussian noise (mean of $0$ and standard deviation of $0.02$) added:
$x = \mu_t + \epsilon$, $\epsilon \sim \mathcal{N}(0, 0.02^2 I_2)$

\paragraph{Models.} Each model is implemented as a lightweight generative model operating in $\mathbb{R}^2$.
We consider three standard generative paradigms: DDPM \citep{ddpm}, CFM \citep{cfm}, GANs \citep{GAN}.
All models are implemented using multilayer perceptrons and are intentionally kept simple to minimize architectural inductive bias. For DDPM models, we use two network architectures with different capacities. All models are trained using identical optimization procedures within each paradigm.

Each model is first trained on its own real data distribution during a warm-up phase until convergence. This initialization ensures that all models start from a converged solution of their real-data objectives. A detailed summary of model architectures, data preferences is reported in~\Cref{subtab:model_summary}.

\paragraph{Training Settings.}
All experiments follow a fixed iterative retraining protocol.
In each retraining round, every model is trained on a mixture of real samples and synthetic samples generated in the previous round.
The total training set size is fixed to 10,000 samples, with the proportion of synthetic samples determined by the mixing ratio $\alpha_i= \frac{\lambda_i}{\lambda_i+1}$. In each round, $10,000$ samples are generated to calculate the metrics. All experiments are repeated with three different random seeds, and results are reported as the mean across runs.

\paragraph{Metrics.}
We measure the distributional discrepancy between generated samples and reference distributions using the 1-Wasserstein \citep{Wasserstein} distance $\mathcal{W}\mathcal{D}$.
To quantify diversity across models, we compute a system-level diversity  metric $\mathcal{D}$ defined as the average deviation of each model’s generated sample mean from the global mean of all models:
$
\mathcal{D} = \frac{1}{K} \sum_{i=1}^K \left\| \mu_i - \frac{1}{K}\sum_{j=1}^K \mu_j \right\|_2
$

\begin{algorithm}[t]
\caption{Iterative Self-Consuming Retraining}
\label{alg:retraining}
\small
\begin{algorithmic}[1]
\STATE {\bfseries Input:} $\{\mathcal{D}_i^{\mathrm{real}}\}_{i=1}^K$,  $\boldsymbol{\theta}_i^{0}$ trained on $\mathcal{D}_i^{\mathrm{real}}$, interaction weights $\mathbf{W}$, synthetic mixing ratio $\alpha_i$

\FOR{round $t=1$ to $T$}
    \FOR{each model $j=1,\dots,K$}
        \STATE Generate synthetic samples $\mathcal{D}_j^{t} \sim p_{\boldsymbol{\theta}_j^{t-1}}$
    \ENDFOR
    \FOR{each model $i=1,\dots,K$}
        \STATE Construct training set $\mathcal{D}_i^{t} = ( 1- \alpha_i) \mathcal{D}_i^{\mathrm{real}} \cup \alpha_i \bigcup_j w_{ij} \mathcal{D}_j^{(t)}$
        \STATE Update $\boldsymbol{\theta}_i^{t}$ by training on $\mathcal{D}_i^{t}$
    \ENDFOR
\ENDFOR

\STATE {\bfseries Output:} $\{\boldsymbol{\theta}_i^{T}\}_{i=1}^K$
\end{algorithmic}
\end{algorithm}

\begin{figure}[htbp]
    \centering
    \begin{subfigure}[c]{0.1\linewidth}
        \centering
        \includegraphics[width=\linewidth]{Figure/W/W1.png}
    \end{subfigure}
    \hfill
    \begin{subfigure}[c]{0.87\linewidth}
        \centering
        \includegraphics[width=\linewidth]{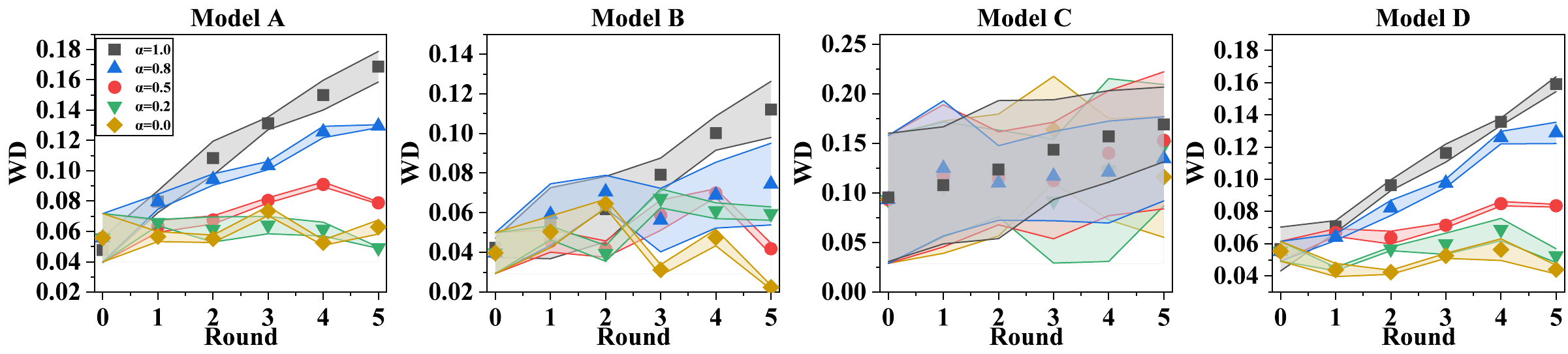}
    \end{subfigure}

    \begin{subfigure}[c]{0.1\linewidth}
        \centering
        \includegraphics[width=\linewidth]{Figure/W/W2.png}
    \end{subfigure}
    \hfill
    \begin{subfigure}[c]{0.87\linewidth}
        \centering
        \includegraphics[width=\linewidth]{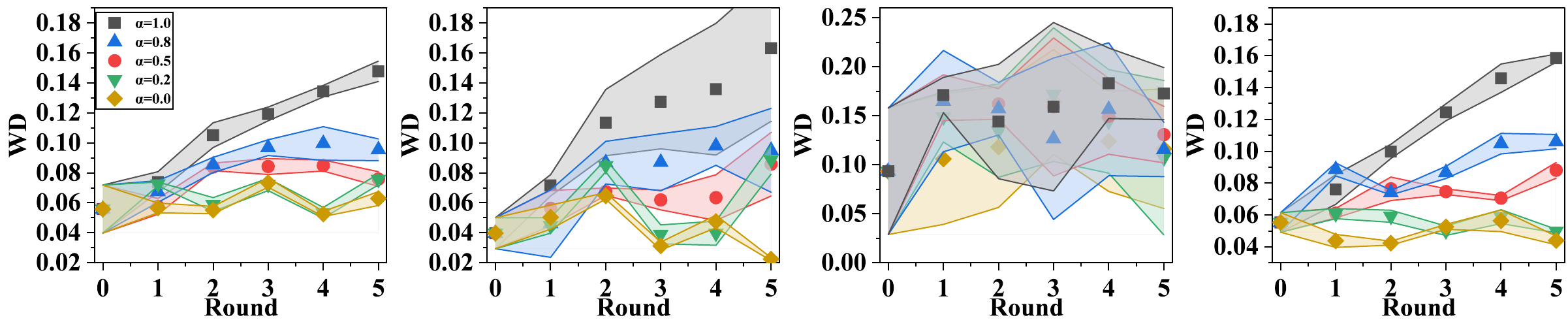}
    \end{subfigure}

    \begin{subfigure}[c]{0.1\linewidth}
        \centering
        \includegraphics[width=\linewidth]{Figure/W/W3.png}
    \end{subfigure}
    \hfill
    \begin{subfigure}[c]{0.87\linewidth}
        \centering
        \includegraphics[width=\linewidth]{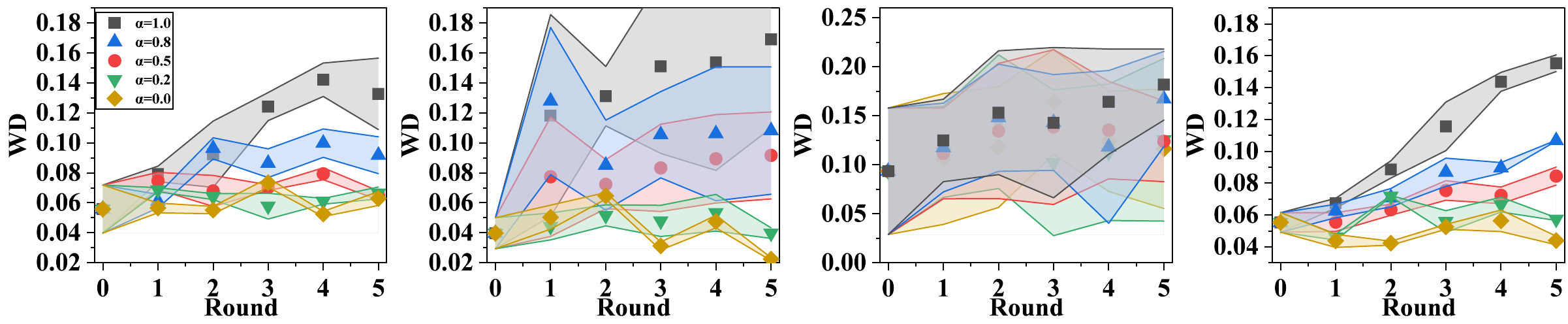}
    \end{subfigure}

    \begin{subfigure}[c]{0.1\linewidth}
        \centering
        \includegraphics[width=\linewidth]{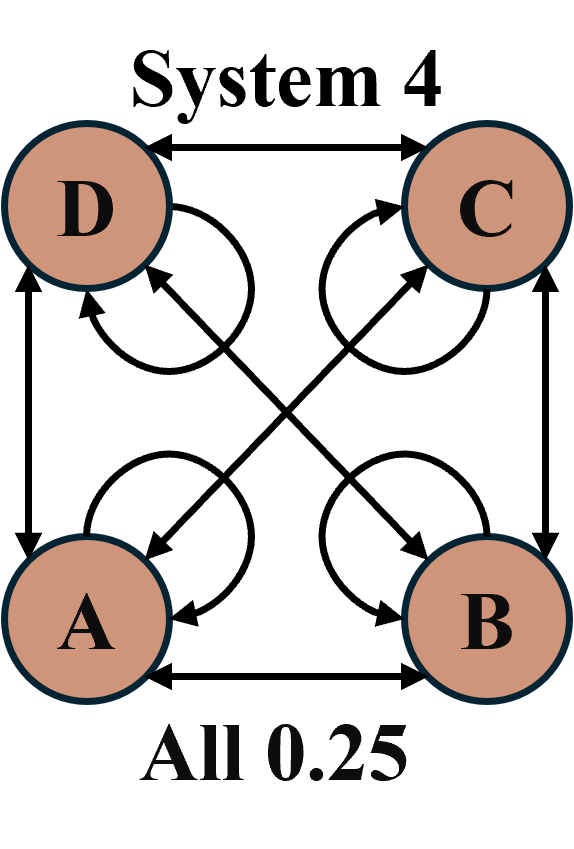}
    \end{subfigure}
    \hfill
    \begin{subfigure}[c]{0.87\linewidth}
        \centering
        \includegraphics[width=\linewidth]{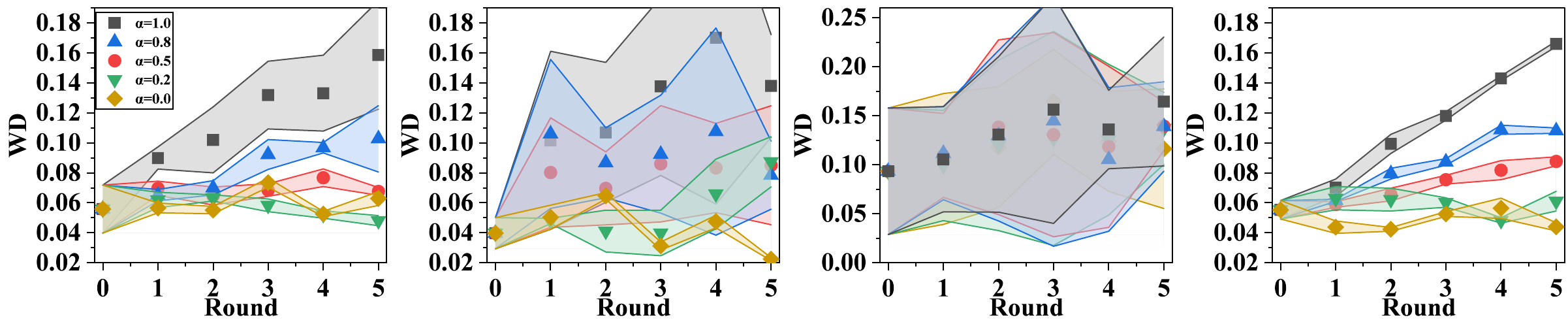}
    \end{subfigure}

    \begin{subfigure}[c]{0.1\linewidth}
        \centering
        \includegraphics[width=\linewidth]{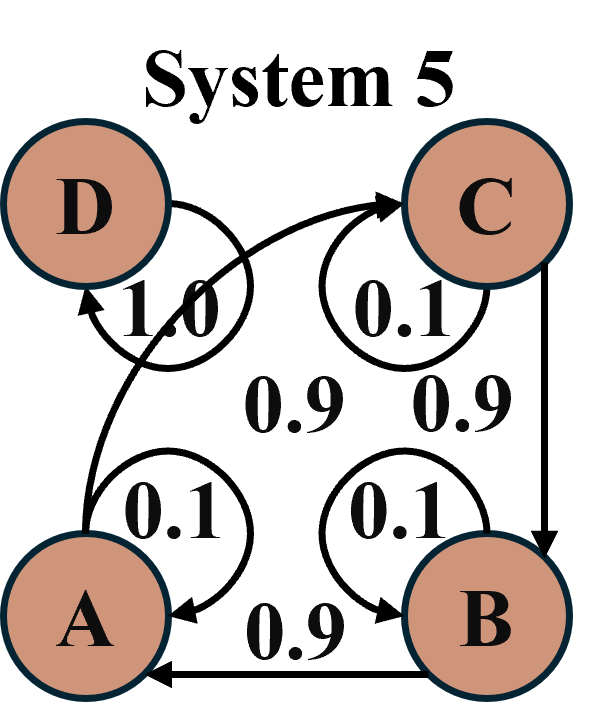}
    \end{subfigure}
    \hfill
    \begin{subfigure}[c]{0.87\linewidth}
        \centering
        \includegraphics[width=\linewidth]{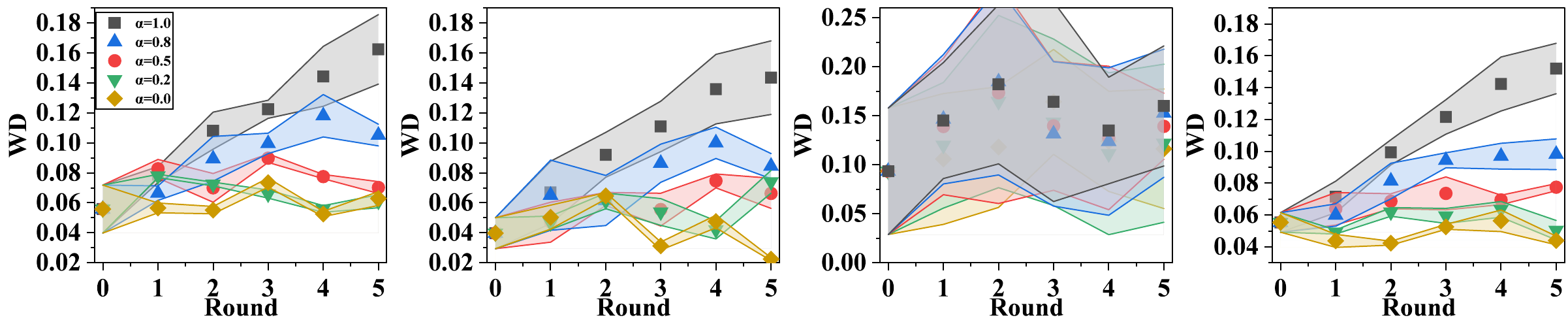}
    \end{subfigure}

    \begin{subfigure}[c]{0.1\linewidth}
        \centering
        \includegraphics[width=\linewidth]{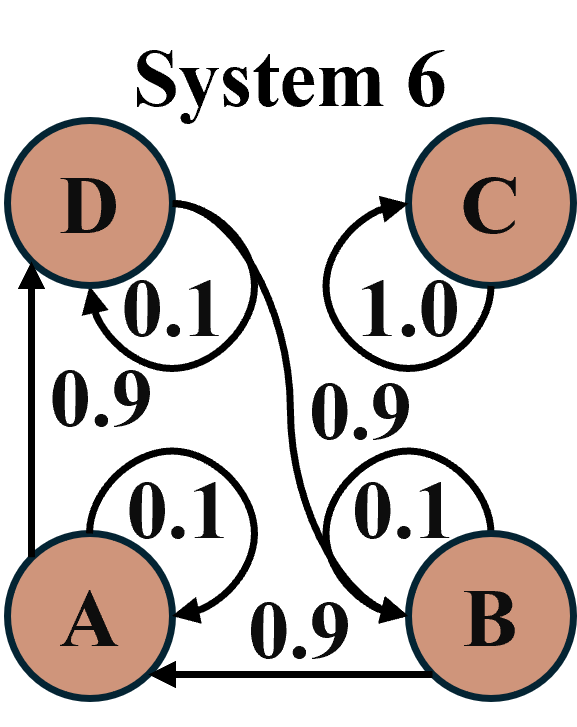}
    \end{subfigure}
    \hfill
    \begin{subfigure}[c]{0.87\linewidth}
        \centering
        \includegraphics[width=\linewidth]{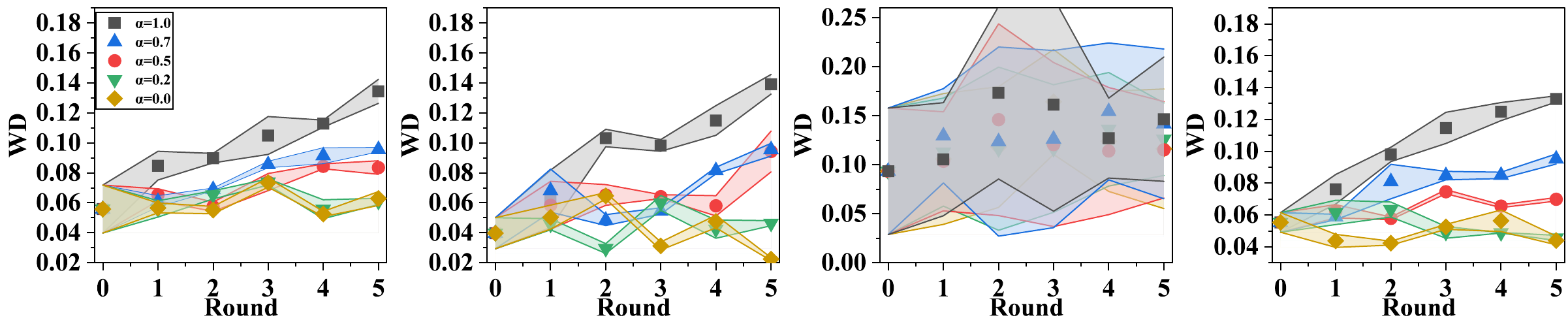}
    \end{subfigure}
    
    \begin{subfigure}[c]{0.1\linewidth}
        \centering
        \includegraphics[width=\linewidth]{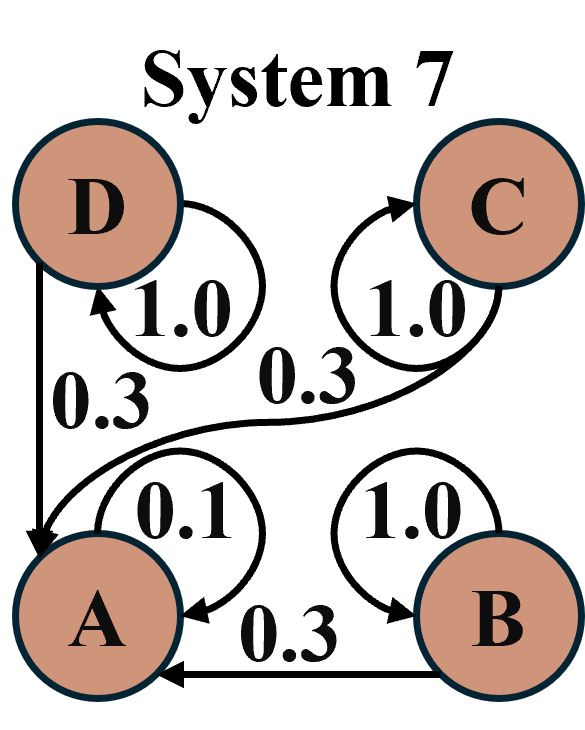}
    \end{subfigure}
    \hfill
    \begin{subfigure}[c]{0.87\linewidth}
        \centering
        \includegraphics[width=\linewidth]{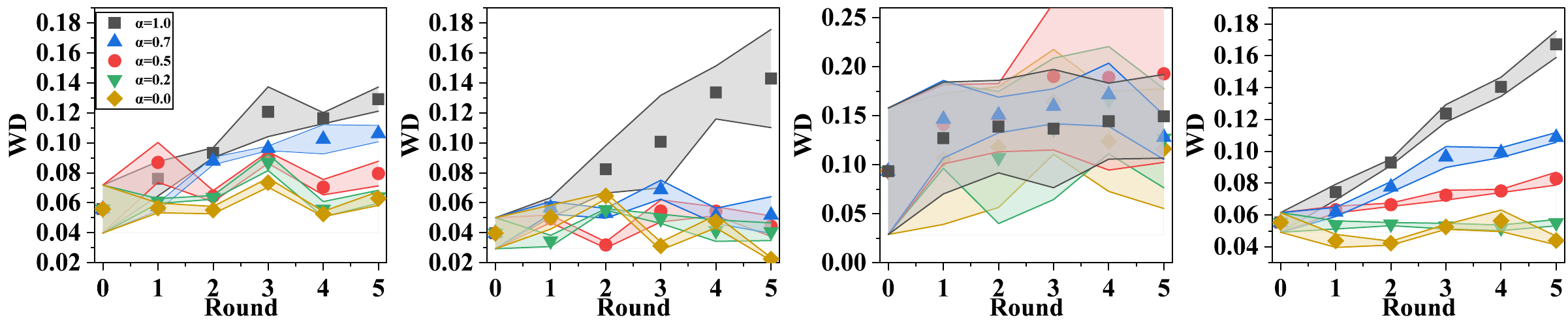}
    \end{subfigure}
    \caption{Stability experiments when all models trained on full dataset. Each row corresponds to a different interaction graph. For each row, the left figure illustrates the direction and strength of synthetic data flow, while the right figures report the Wasserstein distance ($\mathcal{W}\mathcal{D}$) between generated samples and real data for each model. Results are shown under different synthetic mixing strengths, with $\alpha_i=\alpha$ shared across models.}
    \label{Appfig:HomoToy}
\end{figure}

\begin{figure}[htbp]
    \centering
    \begin{subfigure}[c]{0.1\linewidth}
        \centering
        \includegraphics[width=\linewidth]{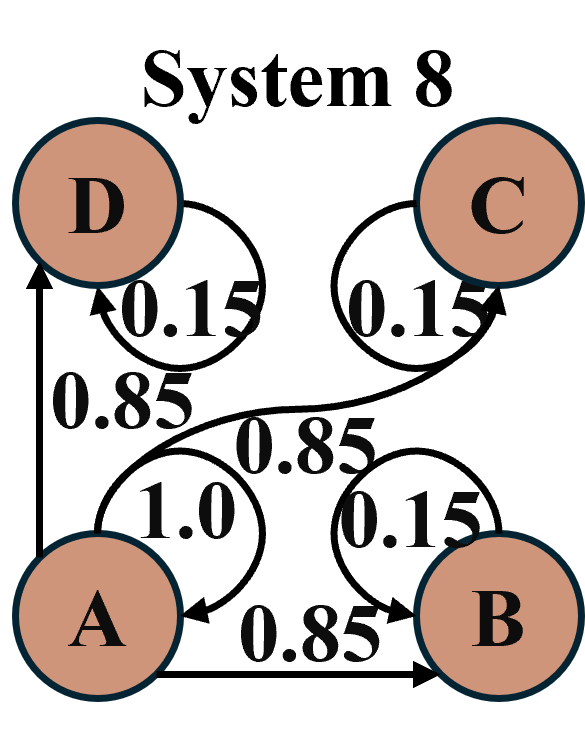}
    \end{subfigure}
    \hfill
    \begin{subfigure}[c]{0.87\linewidth}
        \centering
        \includegraphics[width=\linewidth]{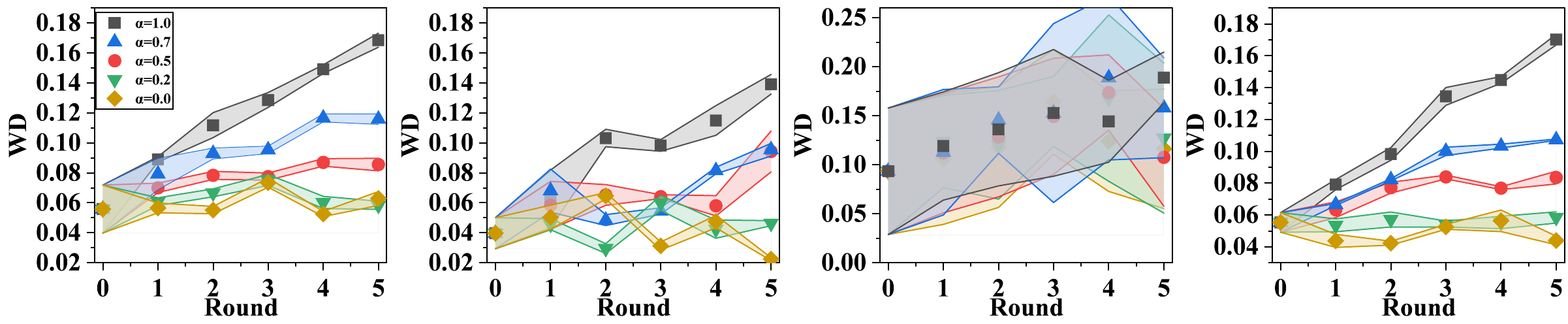}
    \end{subfigure}
    \caption{Stability experiments when all models trained on full dataset (continued).}
    \vspace{-10pt}
    \label{Appfig:HomoToy2}
\end{figure}

\subsection{Stability Experimental Results}
\label{subsec:synex_homo}
We first present the results of stability, where all models access to the same real data distribution $t \in \{0,\cdots,7\}$.
We fix the synthetic mixing ratio to be identical $\alpha_i = \alpha_j = \alpha$.
Fig.~\ref{Appfig:HomoToy} (and its continuation Fig.~\ref{Appfig:HomoToy2}) illustrates different interaction graph structures with varying synthetic data flows and reports the Wasserstein distance $\mathcal{W}\mathcal{D}$ across three  independent runs with different seeds between generated and real samples.

We analyze eight interaction graph structures across four models under the homogeneous data setting. Since all models have the same real-data target in this experiment, persistent growth in $\mathcal{W}\mathcal{D}$ indicates instability of the retraining dynamics rather than disagreement in the desired data distribution.
Although all models are trained on the same real data distribution and use an identical synthetic mixing strength $\alpha$, their retraining behaviors differ substantially due to both model-specific stability and interaction topology.

\textit{Model~C} adopts a GAN-based architecture and exhibits weaker baseline stability compared to the other models. This instability is reflected throughout the experiments. Across nearly all graph structures, Model C shows larger WD fluctuations across retraining rounds.

\textbf{System~1} serves as a baseline that approximates isolated self-consuming retraining, where models primarily train on their own generated samples.
In this setting, increasing $\alpha$ consistently leads to higher $\mathcal{W}\mathcal{D}$ for all models, with the effect being most pronounced for \textit{Model C}. This behavior aligns with the theoretical expectation that stronger synthetic mixing amplifies local instability even in the absence of cross-model interactions.

We next examine the effect of cycles by comparing \textbf{System~1}, \textbf{System~3}, and \textbf{System~5}, which differ in their maximum cycle lengths. Under the same synthetic mixing strength, longer cycles do not necessarily result in more severe degradation. In particular, \textbf{System~5} exhibits the highest $\mathcal{W}\mathcal{D}$ mean and variance across models, despite not having the longest cycle.
This observation is consistent with the theoretical result that stability is governed by the dominant cycle gain rather than cycle length alone.

The influence of interaction strength is illustrated by comparing \textbf{System~2} and \textbf{System 3}. \textbf{System~2} satisfies the theoretical stability condition based on the interaction norm, while \textbf{System~3} exceeds this threshold. Empirically, \textbf{System~3} displays faster $\mathcal{W}\mathcal{D}$ growth and increased variance across models under the same $\alpha$. These results support the theoretical prediction that stronger interactions reduce the stability margin of self-consuming retraining. At the same time, differences among systems with similar interaction norms (\textbf{System~1}, \textbf{System~2} and \textbf{System~4}) indicate that norm-based conditions are conservative and do not fully capture structural effects.

Finally, comparing \textbf{System~5} and \textbf{System~6} highlights the role of model participation within the interaction graph. Although the two systems share similar topological structures, \textbf{System~5} involves the unstable \textit{Model~C} more directly in high-gain cycles.
As a result, all models in \textbf{System~5} experience higher $\mathcal{W}\mathcal{D}$ and larger variance, whereas the corresponding effect is attenuated in \textbf{System~6}. This comparison suggests that instability introduced by a single weak model can propagate through the network when it participates in dominant interaction cycles, thereby affecting the entire system.

Overall, the homogeneous data experiments confirm the main theoretical predictions. Stability depends jointly on synthetic mixing strength, interaction structure, and the presence of unstable models. Moreover, graph topology shapes how local instabilities are amplified or contained, beyond what can be explained by interaction norms alone.

We further visualize the generated samples in Fig.~\ref{Appfig:HomoToy_samples_1.0} ($\alpha=1.0$) and Fig.~\ref{Appfig:HomoToy_samples_0.2} ($\alpha=0.2$).

\begin{figure}[htbp]
    \centering
        \begin{subfigure}[c]{0.1\linewidth}
        \centering \textbf{Initiation}
    \end{subfigure}
    \hfill
    \begin{subfigure}[c]{0.87\linewidth}
        \centering
        \includegraphics[width=\linewidth]{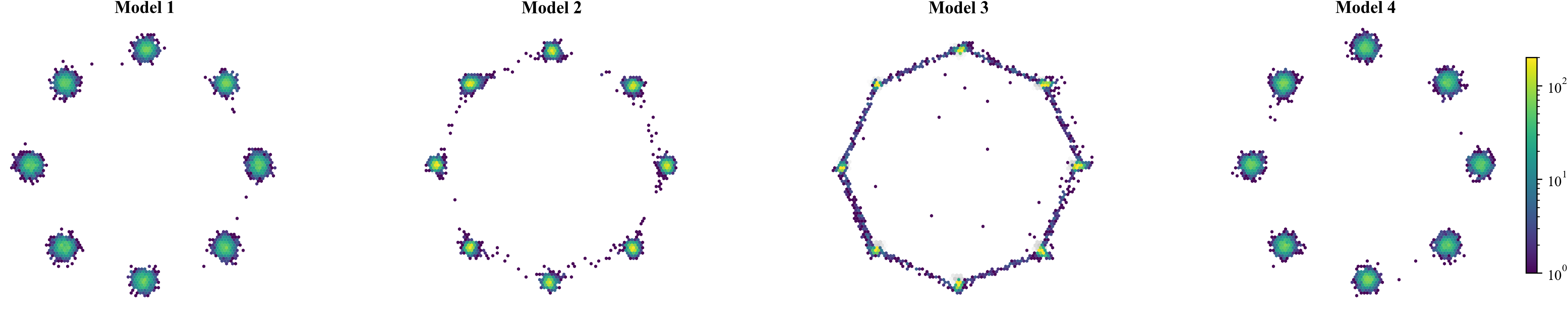}
    \end{subfigure}
    
    \begin{subfigure}[c]{0.1\linewidth}
        \centering
        \includegraphics[width=\linewidth]{Figure/W/W1.png}
    \end{subfigure}  
    \hfill
    \begin{subfigure}[c]{0.87\linewidth}
        \centering
        \includegraphics[width=\linewidth]{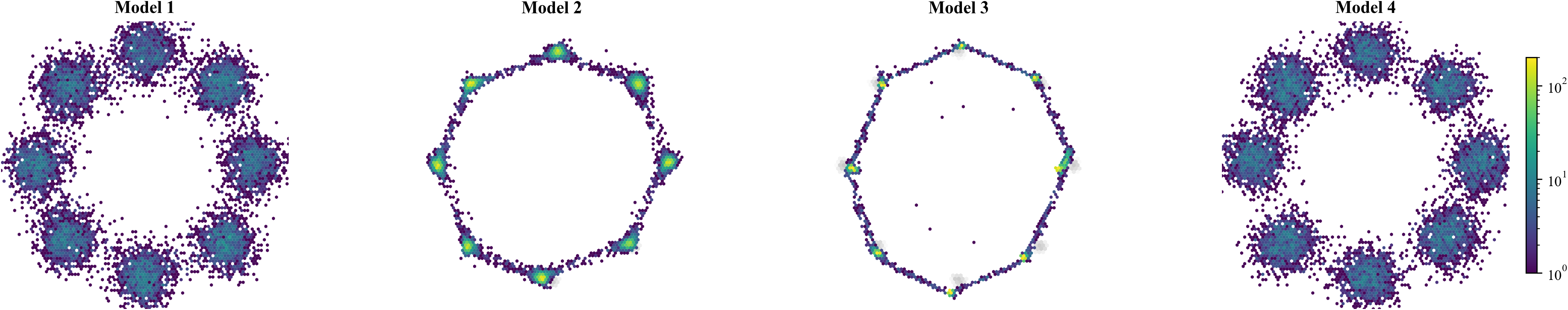}
    \end{subfigure}

    \begin{subfigure}[c]{0.1\linewidth}
        \centering
        \includegraphics[width=\linewidth]{Figure/W/W2.png}
    \end{subfigure}
    \hfill
    \begin{subfigure}[c]{0.87\linewidth}
        \centering
        \includegraphics[width=\linewidth]{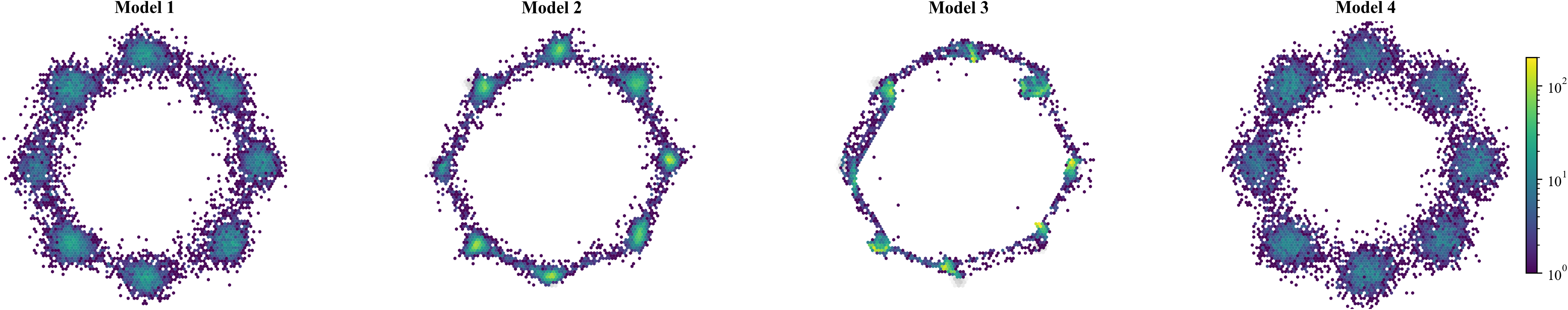}
    \end{subfigure}

    \begin{subfigure}[c]{0.1\linewidth}
        \centering
        \includegraphics[width=\linewidth]{Figure/W/W3.png}
    \end{subfigure}
    \hfill
    \begin{subfigure}[c]{0.87\linewidth}
        \centering
        \includegraphics[width=\linewidth]{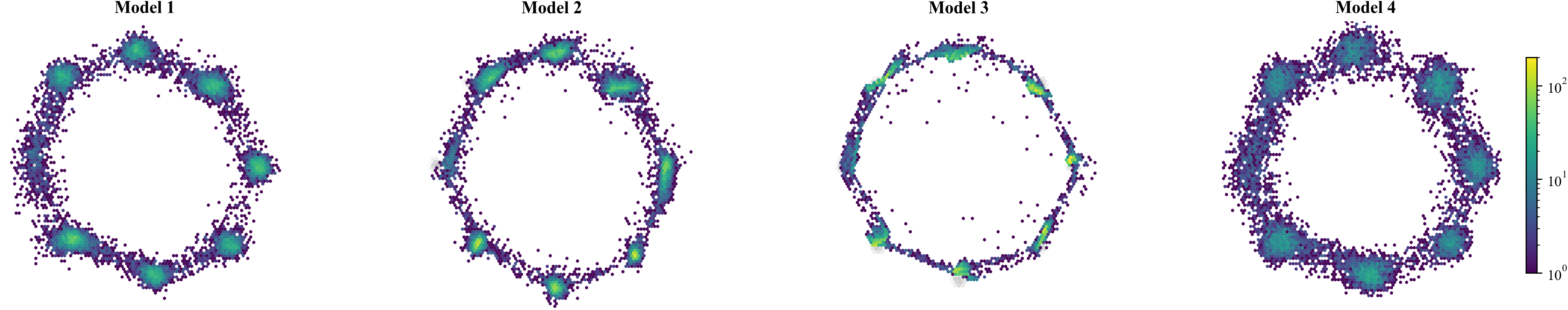}
    \end{subfigure}

    \begin{subfigure}[c]{0.1\linewidth}
        \centering
        \includegraphics[width=\linewidth]{Figure/W/W4.png}
    \end{subfigure}
    \hfill
    \begin{subfigure}[c]{0.87\linewidth}
        \centering
        \includegraphics[width=\linewidth]{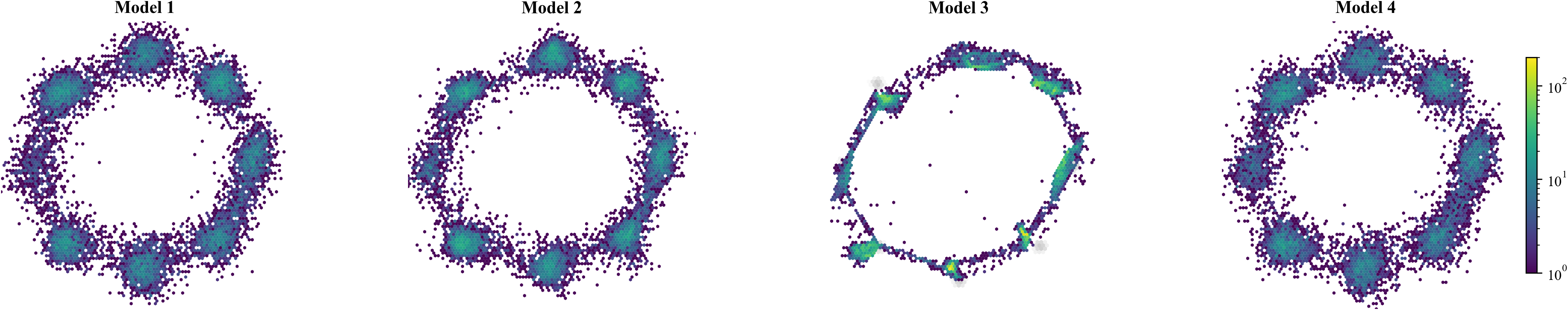}
    \end{subfigure}

    \begin{subfigure}[c]{0.1\linewidth}
        \centering
        \includegraphics[width=\linewidth]{Figure/W/W5.png}
    \end{subfigure}
    \hfill
    \begin{subfigure}[c]{0.87\linewidth}
        \centering
        \includegraphics[width=\linewidth]{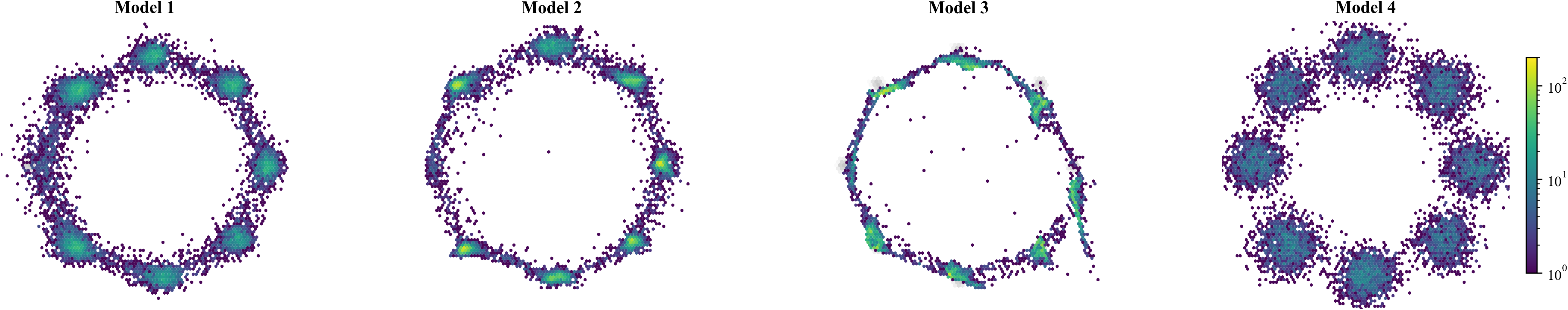}
    \end{subfigure}

    \begin{subfigure}[c]{0.1\linewidth}
        \centering
        \includegraphics[width=\linewidth]{Figure/W/W8.png}
    \end{subfigure}
    \hfill
    \begin{subfigure}[c]{0.87\linewidth}
        \centering
        \includegraphics[width=\linewidth]{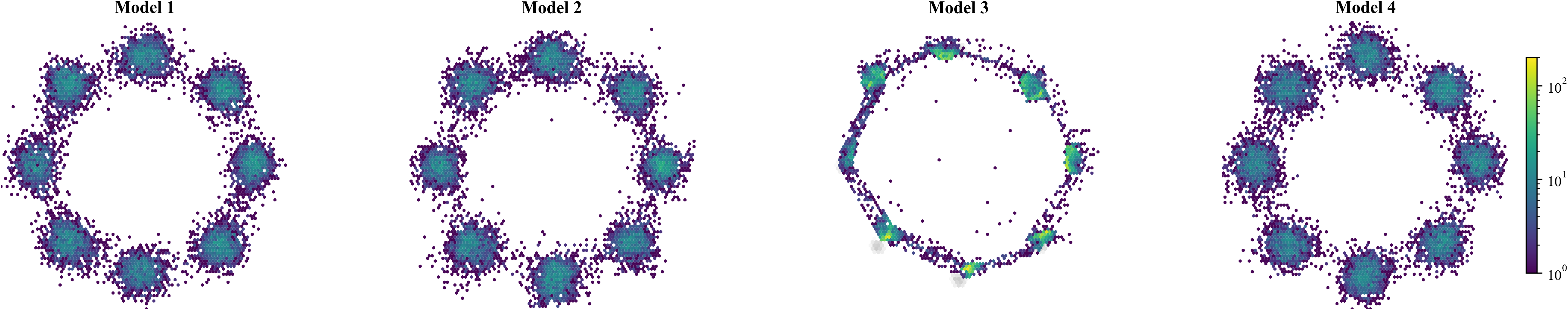}
    \end{subfigure}
    \caption{Generated samples when all models trained on full dataset with different interaction graphs. All models use $\alpha_i=\alpha=1.0$.}
    \label{Appfig:HomoToy_samples_1.0}
\end{figure}

\begin{figure}[htbp]
    \centering
        \begin{subfigure}[c]{0.1\linewidth}
        \centering \textbf{Initiation}
    \end{subfigure}
    \hfill
    \begin{subfigure}[c]{0.87\linewidth}
        \centering
        \includegraphics[width=\linewidth]{Figure/HomoToy/Sample/0.png}
    \end{subfigure}
    
    \begin{subfigure}[c]{0.1\linewidth}
        \centering
        \includegraphics[width=\linewidth]{Figure/W/W1.png}
    \end{subfigure}  
    \hfill
    \begin{subfigure}[c]{0.87\linewidth}
        \centering
        \includegraphics[width=\linewidth]{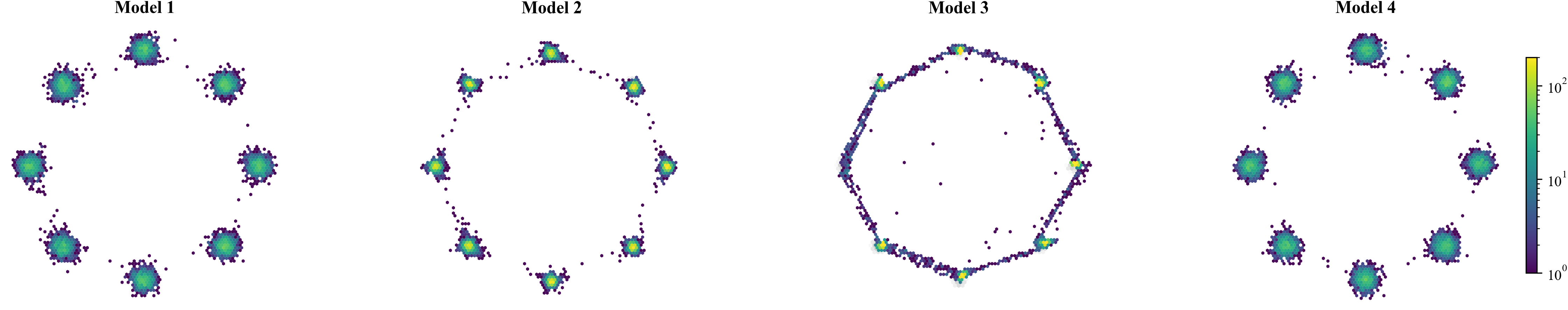}
    \end{subfigure}

    \begin{subfigure}[c]{0.1\linewidth}
        \centering
        \includegraphics[width=\linewidth]{Figure/W/W2.png}
    \end{subfigure}
    \hfill
    \begin{subfigure}[c]{0.87\linewidth}
        \centering
        \includegraphics[width=\linewidth]{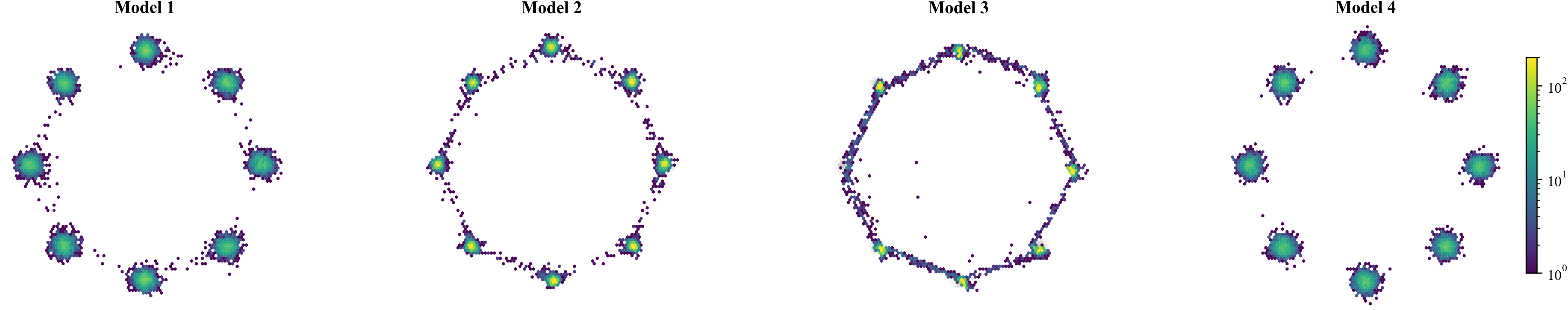}
    \end{subfigure}

    \begin{subfigure}[c]{0.1\linewidth}
        \centering
        \includegraphics[width=\linewidth]{Figure/W/W3.png}
    \end{subfigure}
    \hfill
    \begin{subfigure}[c]{0.87\linewidth}
        \centering
        \includegraphics[width=\linewidth]{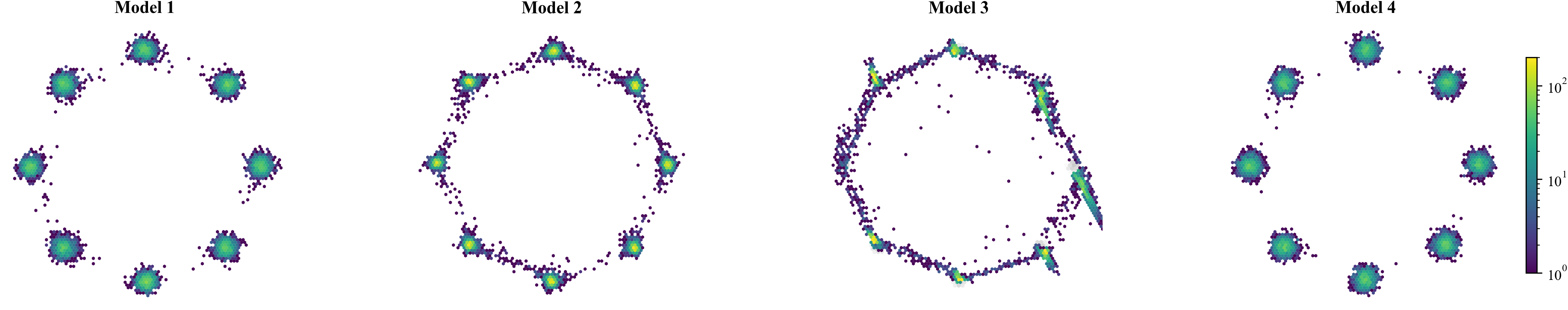}
    \end{subfigure}

    \begin{subfigure}[c]{0.1\linewidth}
        \centering
        \includegraphics[width=\linewidth]{Figure/W/W4.png}
    \end{subfigure}
    \hfill
    \begin{subfigure}[c]{0.87\linewidth}
        \centering
        \includegraphics[width=\linewidth]{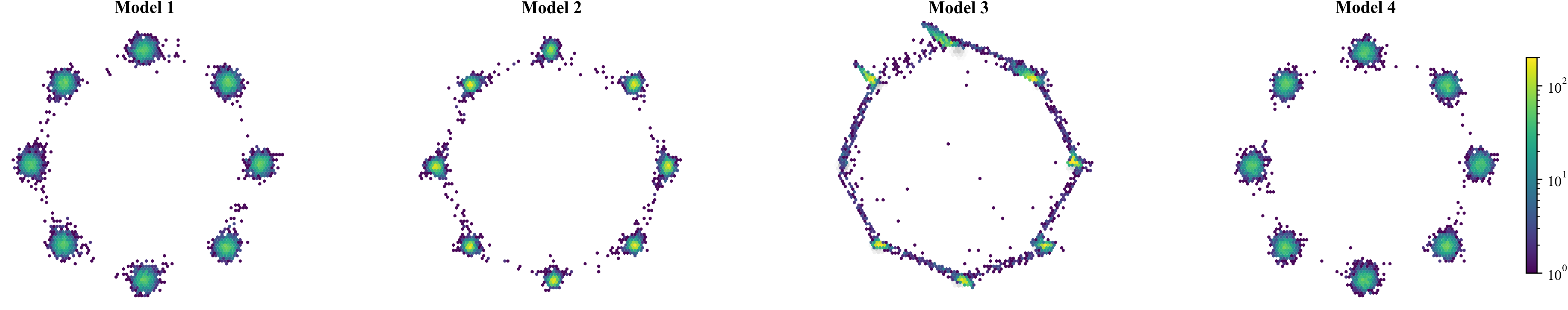}
    \end{subfigure}

    \begin{subfigure}[c]{0.1\linewidth}
        \centering
        \includegraphics[width=\linewidth]{Figure/W/W5.png}
    \end{subfigure}
    \hfill
    \begin{subfigure}[c]{0.87\linewidth}
        \centering
        \includegraphics[width=\linewidth]{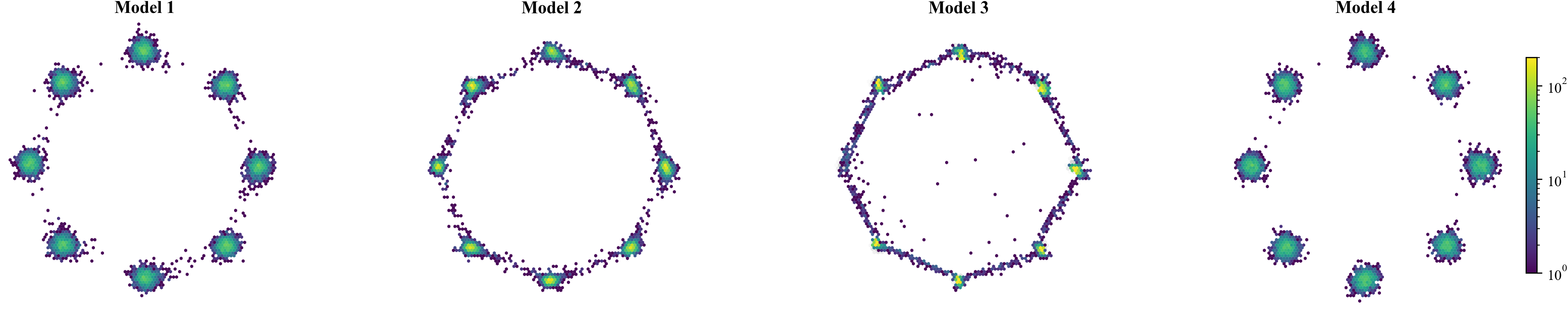}
    \end{subfigure}

    \begin{subfigure}[c]{0.1\linewidth}
        \centering
        \includegraphics[width=\linewidth]{Figure/W/W8.png}
    \end{subfigure}
    \hfill
    \begin{subfigure}[c]{0.87\linewidth}
        \centering
        \includegraphics[width=\linewidth]{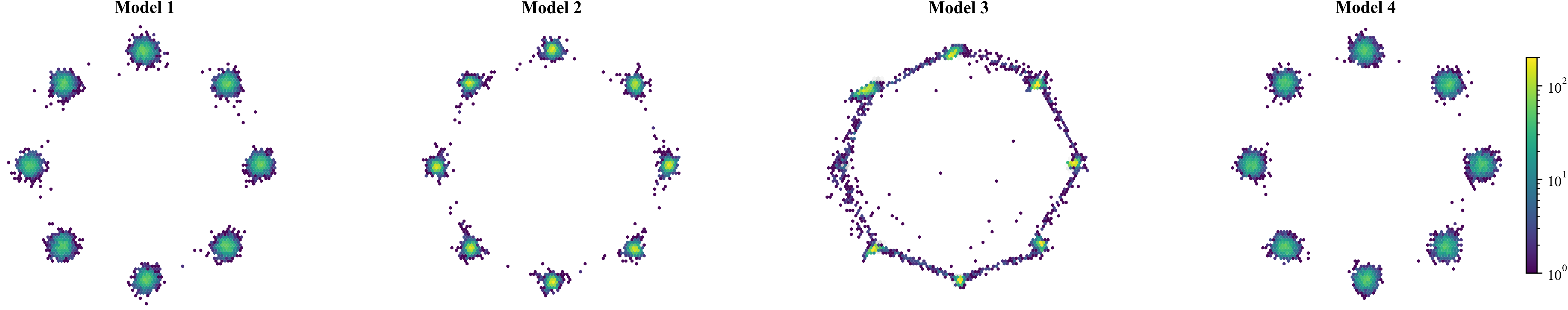}
    \end{subfigure}
    \caption{Generated samples when all models trained on full dataset. All models use $\alpha_i=\alpha=0.2$.}
    \label{Appfig:HomoToy_samples_0.2}
\end{figure}

\subsection{Diversity Experimental Results}
\label{subsec:synex_hete}
We next study the diversity, where models are trained on different subsets of the underlying modes as shown in~\Cref{subtab:model_summary}.  In this setting, we focus on a system-level diversity metric $\mathcal{D}$, which captures the degree of differentiation among models without requiring a shared reference distribution as shown in Fig.~\ref{Appfig:HeteToy}. We further visualize the generated samples to qualitatively illustrate the stability and diversity of the system as shown in Fig.~\ref{Appfig:HeteToy_samples_1.0} and Fig.~\ref{Appfig:HeteToy_samples_0.2}.

\begin{figure}[htbp]
    \centering
     \begin{subfigure}[c]{0.23\linewidth}
        \centering
        \includegraphics[width=0.5\linewidth]{Figure/W/W1.png}
    \end{subfigure}
    \hfill
     \begin{subfigure}[c]{0.23\linewidth}
        \centering
        \includegraphics[width=0.5\linewidth]{Figure/W/W2.png}
    \end{subfigure}
    \hfill
     \begin{subfigure}[c]{0.23\linewidth}
        \centering
        \includegraphics[width=0.5\linewidth]{Figure/W/W3.png}
    \end{subfigure}
    \hfill
     \begin{subfigure}[c]{0.23\linewidth}
        \centering
        \includegraphics[width=0.5\linewidth]{Figure/W/W4.png}
    \end{subfigure}
    
    \begin{subfigure}[c]{0.23\linewidth}
        \centering
        \includegraphics[width=\linewidth]{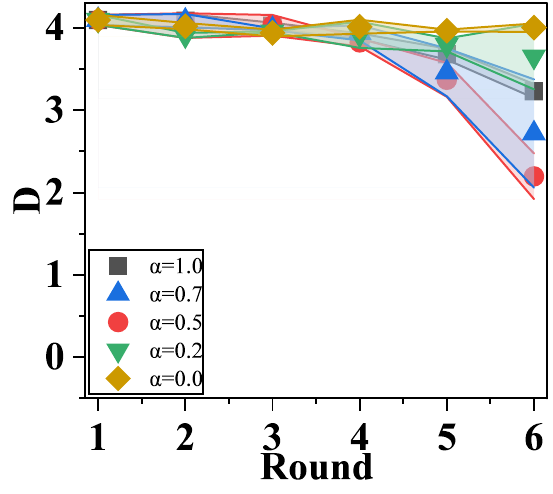}
    \end{subfigure}
    \hfill
    \begin{subfigure}[c]{0.23\linewidth}
        \centering
        \includegraphics[width=\linewidth]{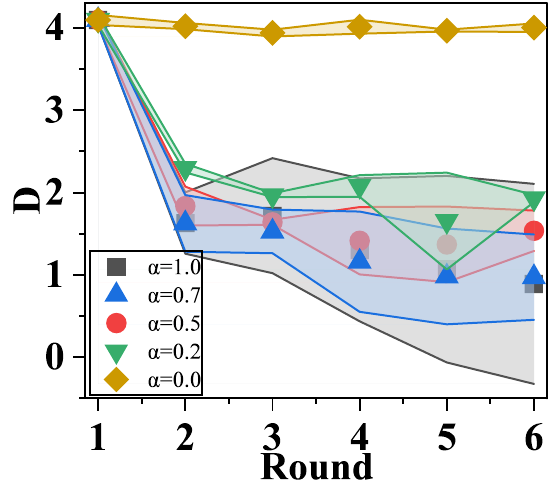}
    \end{subfigure}
    \hfill
    \begin{subfigure}[c]{0.23\linewidth}
        \centering
        \includegraphics[width=\linewidth]{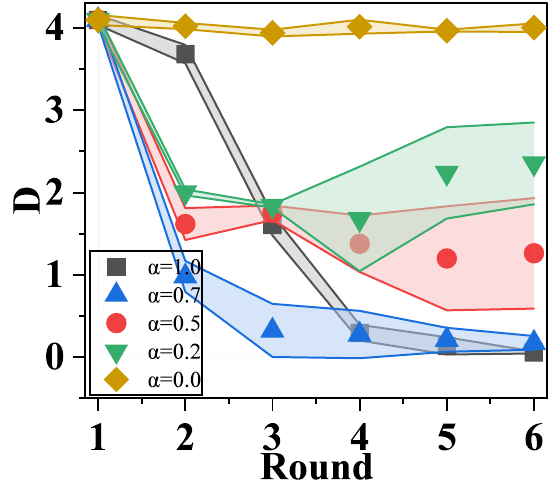}
    \end{subfigure}
    \hfill
    \begin{subfigure}[c]{0.23\linewidth}
        \centering
        \includegraphics[width=\linewidth]{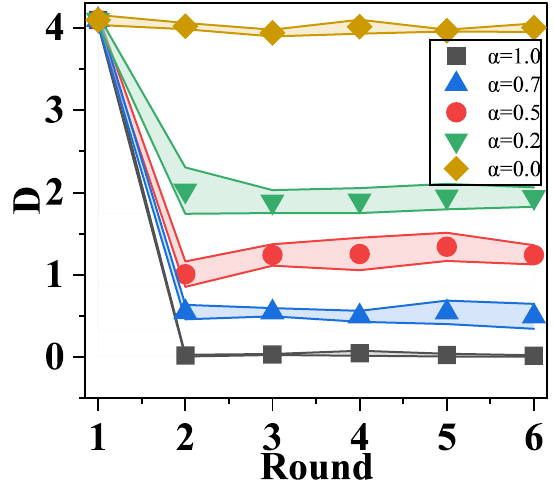}
    \end{subfigure}

    \begin{subfigure}[c]{0.23\linewidth}
        \centering
        \includegraphics[width=0.5\linewidth]{Figure/W/W5.png}
    \end{subfigure}
    \hfill
     \begin{subfigure}[c]{0.23\linewidth}
        \centering
        \includegraphics[width=0.5\linewidth]{Figure/W/W8.png}
    \end{subfigure}
    \hfill
     \begin{subfigure}[c]{0.23\linewidth}
        \centering
        \includegraphics[width=0.5\linewidth]{Figure/W/W6.png}
    \end{subfigure}
    \hfill
     \begin{subfigure}[c]{0.23\linewidth}
        \centering
        \includegraphics[width=0.5\linewidth]{Figure/W/W7.png}
    \end{subfigure}
    
    \begin{subfigure}[c]{0.23\linewidth}
        \centering
        \includegraphics[width=\linewidth]{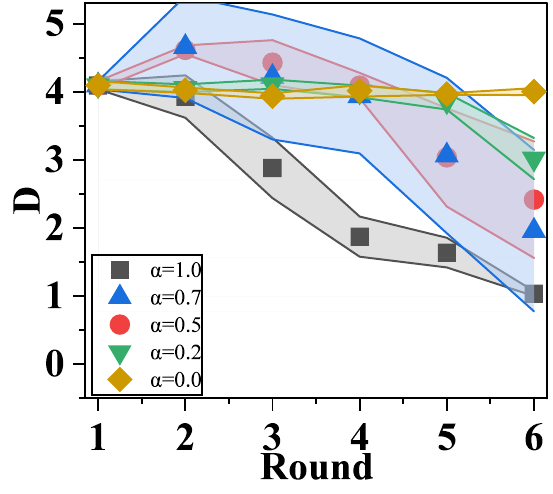}
    \end{subfigure}
    \hfill
    \begin{subfigure}[c]{0.23\linewidth}
        \centering
        \includegraphics[width=\linewidth]{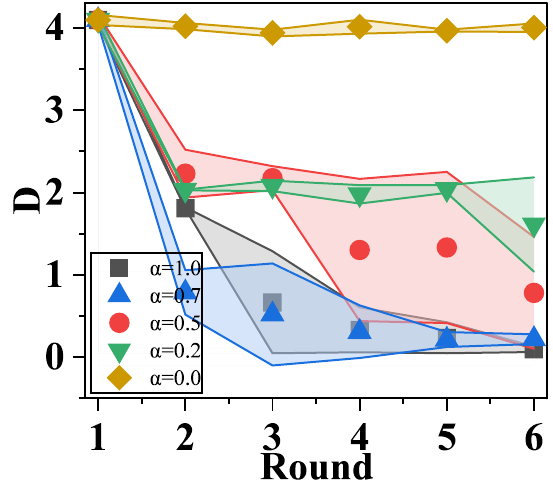}
    \end{subfigure}
    \hfill
    \begin{subfigure}[c]{0.23\linewidth}
        \centering
        \includegraphics[width=\linewidth]{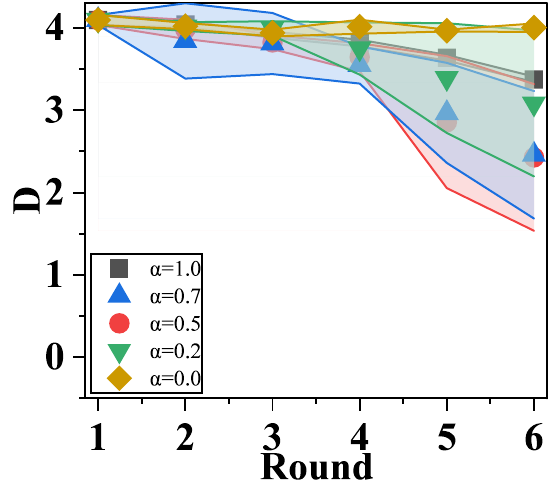}
    \end{subfigure}
    \hfill
    \begin{subfigure}[c]{0.23\linewidth}
        \centering
        \includegraphics[width=\linewidth]{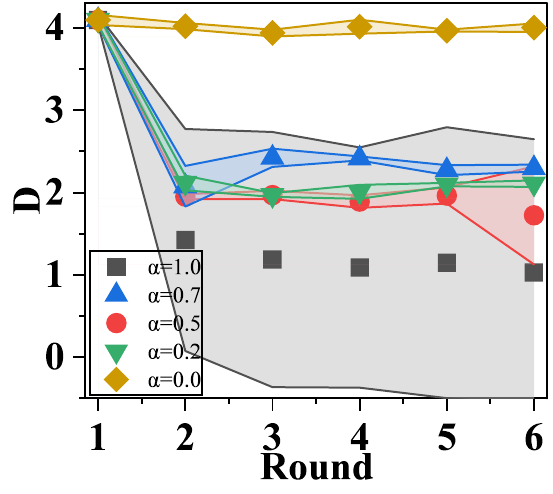}
    \end{subfigure}
    \caption{Diversity across models under heterogeneous data with different interaction graphs. Each figure reports the system-level diversity metric $\mathcal{D}$ across retraining rounds under different synthetic mixing strengths $\alpha$. }
    \label{Appfig:HeteToy}
    \vspace{-10pt}
\end{figure}

\begin{figure}[htbp]
    \centering
        \begin{subfigure}[c]{0.1\linewidth}
        \centering \textbf{Initiation}
    \end{subfigure}
    \hfill
    \begin{subfigure}[c]{0.87\linewidth}
        \centering
        \includegraphics[width=\linewidth]{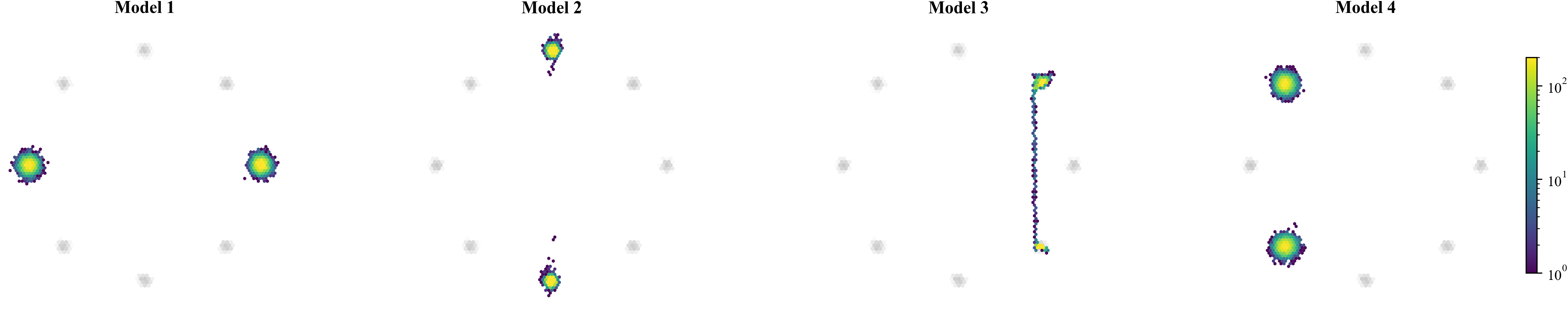}
    \end{subfigure}
    
    \begin{subfigure}[c]{0.1\linewidth}
        \centering
        \includegraphics[width=\linewidth]{Figure/W/W1.png}
    \end{subfigure}  
    \hfill
    \begin{subfigure}[c]{0.87\linewidth}
        \centering
        \includegraphics[width=\linewidth]{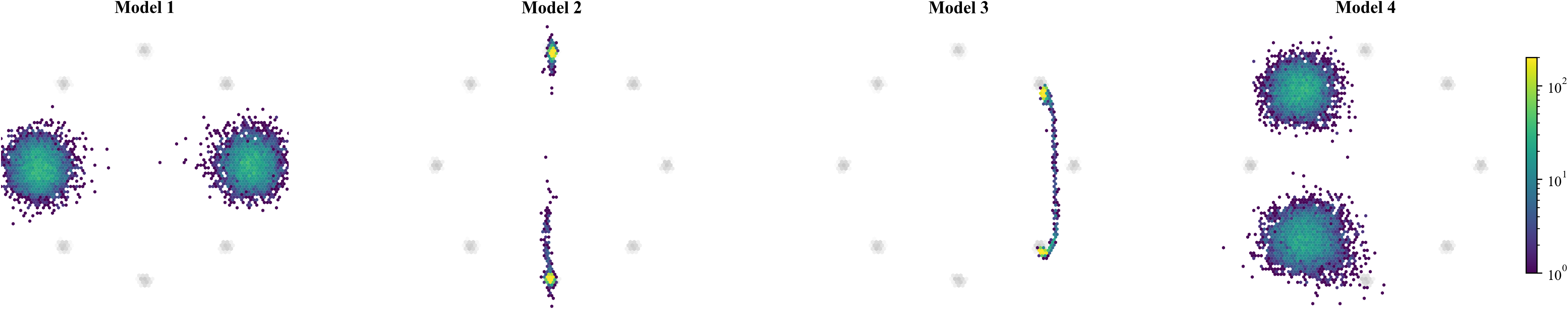}
    \end{subfigure}

    \begin{subfigure}[c]{0.1\linewidth}
        \centering
        \includegraphics[width=\linewidth]{Figure/W/W2.png}
    \end{subfigure}
    \hfill
    \begin{subfigure}[c]{0.87\linewidth}
        \centering
        \includegraphics[width=\linewidth]{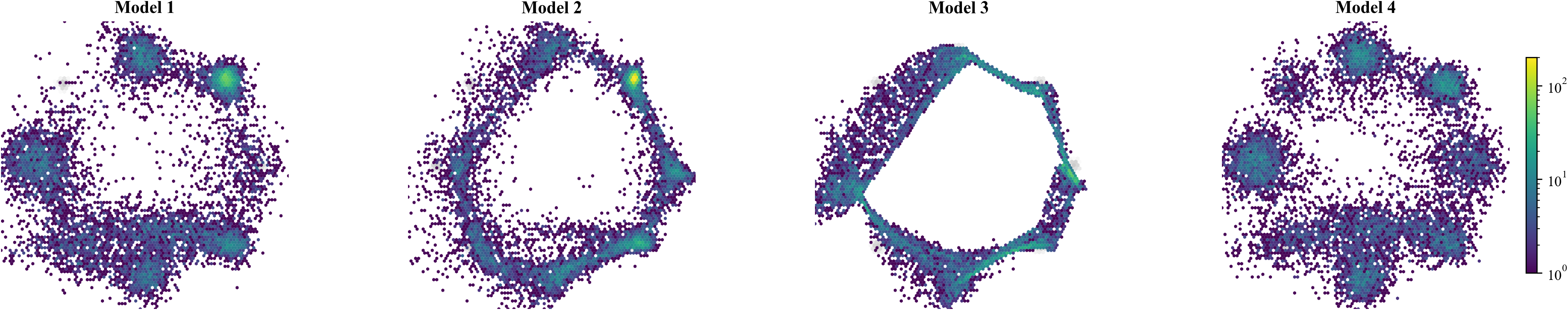}
    \end{subfigure}

    \begin{subfigure}[c]{0.1\linewidth}
        \centering
        \includegraphics[width=\linewidth]{Figure/W/W3.png}
    \end{subfigure}
    \hfill
    \begin{subfigure}[c]{0.87\linewidth}
        \centering
        \includegraphics[width=\linewidth]{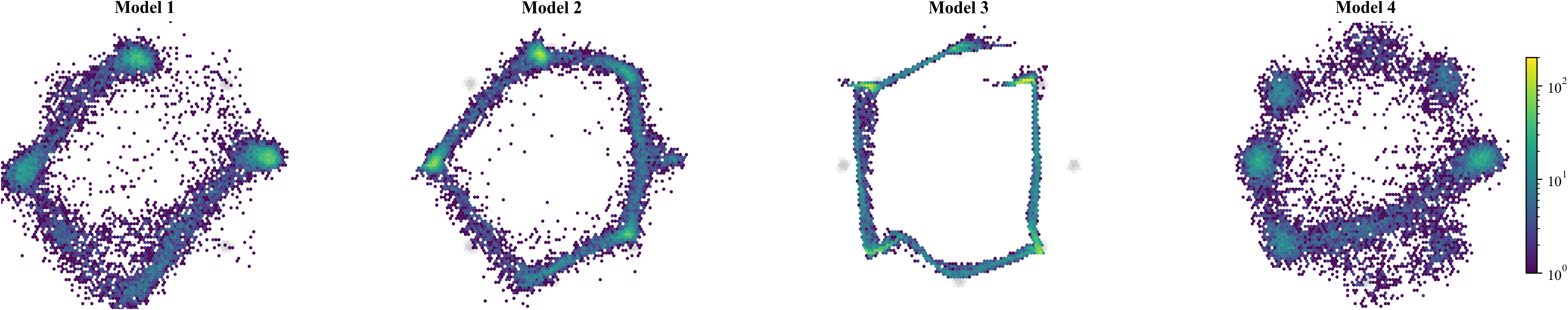}
    \end{subfigure}

    \begin{subfigure}[c]{0.1\linewidth}
        \centering
        \includegraphics[width=\linewidth]{Figure/W/W4.png}
    \end{subfigure}
    \hfill
    \begin{subfigure}[c]{0.87\linewidth}
        \centering
        \includegraphics[width=\linewidth]{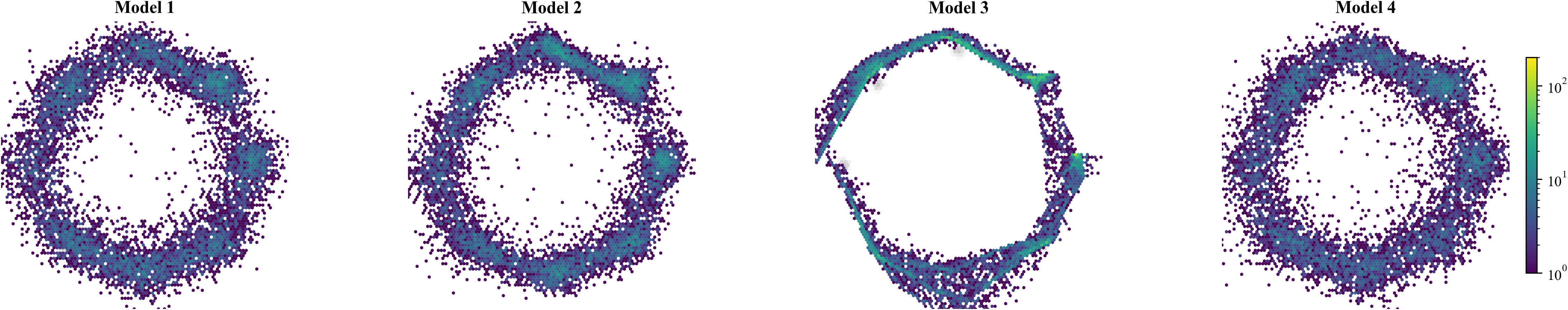}
    \end{subfigure}

    \begin{subfigure}[c]{0.1\linewidth}
        \centering
        \includegraphics[width=\linewidth]{Figure/W/W5.png}
    \end{subfigure}
    \hfill
    \begin{subfigure}[c]{0.87\linewidth}
        \centering
        \includegraphics[width=\linewidth]{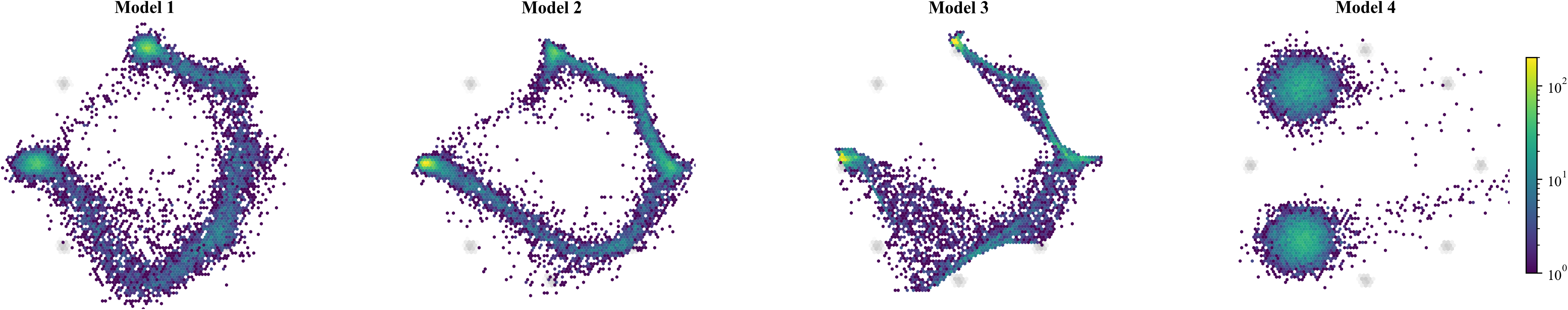}
    \end{subfigure}

    \begin{subfigure}[c]{0.1\linewidth}
        \centering
        \includegraphics[width=\linewidth]{Figure/W/W8.png}
    \end{subfigure}
    \hfill
    \begin{subfigure}[c]{0.87\linewidth}
        \centering
        \includegraphics[width=\linewidth]{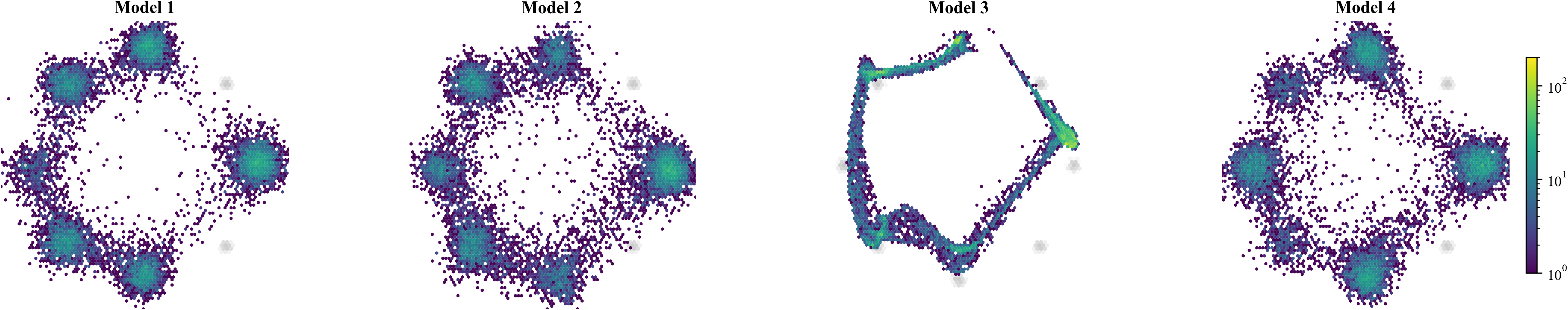}
    \end{subfigure}
    \caption{Generated samples under the heterogeneous data with different interaction graphs.
All models use $\alpha_i=\alpha=1.0$.}
    \label{Appfig:HeteToy_samples_1.0}
\end{figure}

\begin{figure}[htbp]
    \centering
        \begin{subfigure}[c]{0.1\linewidth}
        \centering \textbf{Initiation}
    \end{subfigure}
    \hfill
    \begin{subfigure}[c]{0.87\linewidth}
        \centering
        \includegraphics[width=\linewidth]{Figure/HeteToy/Sample/0.png}
    \end{subfigure}
    
    \begin{subfigure}[c]{0.1\linewidth}
        \centering
        \includegraphics[width=\linewidth]{Figure/W/W1.png}
    \end{subfigure}  
    \hfill
    \begin{subfigure}[c]{0.87\linewidth}
        \centering
        \includegraphics[width=\linewidth]{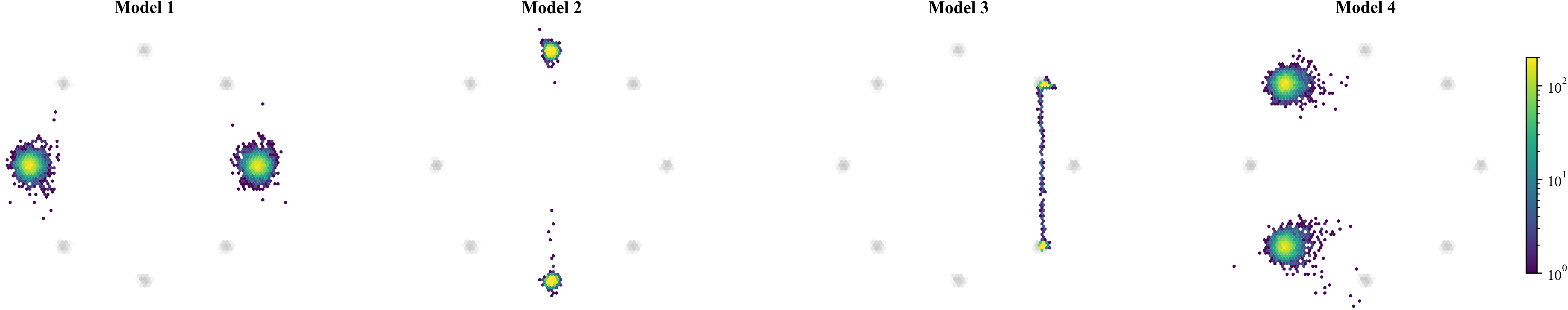}
    \end{subfigure}

    \begin{subfigure}[c]{0.1\linewidth}
        \centering
        \includegraphics[width=\linewidth]{Figure/W/W2.png}
    \end{subfigure}
    \hfill
    \begin{subfigure}[c]{0.87\linewidth}
        \centering
        \includegraphics[width=\linewidth]{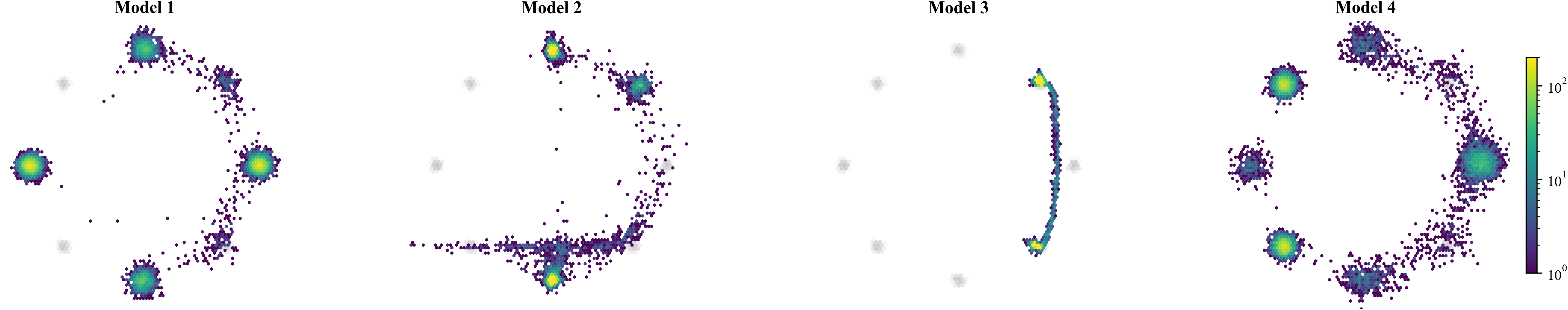}
    \end{subfigure}

    \begin{subfigure}[c]{0.1\linewidth}
        \centering
        \includegraphics[width=\linewidth]{Figure/W/W3.png}
    \end{subfigure}
    \hfill
    \begin{subfigure}[c]{0.87\linewidth}
        \centering
        \includegraphics[width=\linewidth]{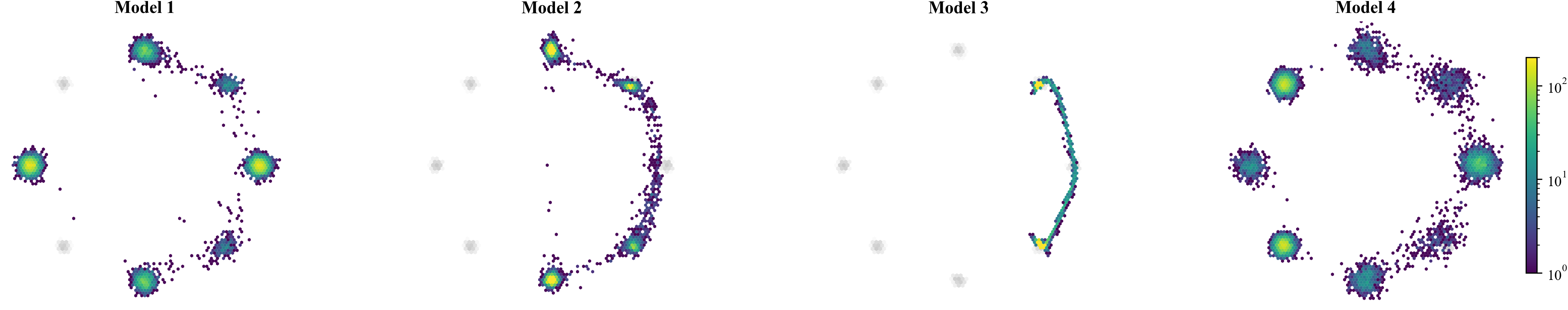}
    \end{subfigure}

    \begin{subfigure}[c]{0.1\linewidth}
        \centering
        \includegraphics[width=\linewidth]{Figure/W/W4.png}
    \end{subfigure}
    \hfill
    \begin{subfigure}[c]{0.87\linewidth}
        \centering
        \includegraphics[width=\linewidth]{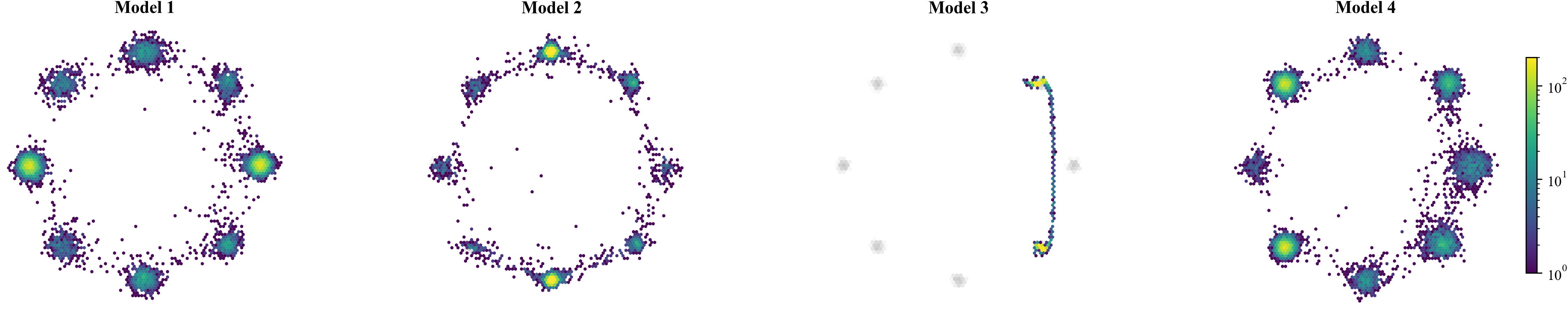}
    \end{subfigure}

    \begin{subfigure}[c]{0.1\linewidth}
        \centering
        \includegraphics[width=\linewidth]{Figure/W/W5.png}
    \end{subfigure}
    \hfill
    \begin{subfigure}[c]{0.87\linewidth}
        \centering
        \includegraphics[width=\linewidth]{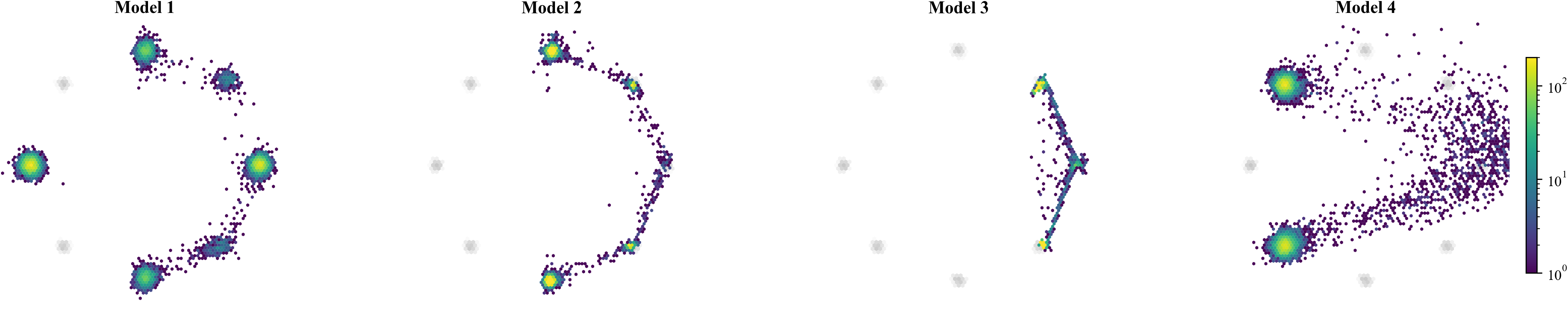}
    \end{subfigure}

    \begin{subfigure}[c]{0.1\linewidth}
        \centering
        \includegraphics[width=\linewidth]{Figure/W/W8.png}
    \end{subfigure}
    \hfill
    \begin{subfigure}[c]{0.87\linewidth}
        \centering
        \includegraphics[width=\linewidth]{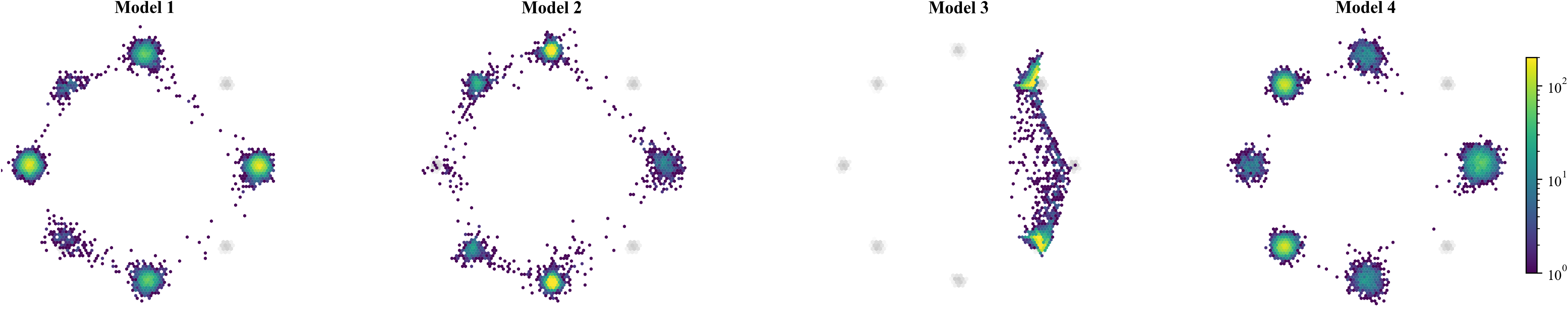}
    \end{subfigure}
    \caption{Generated samples under the heterogeneous data with different interaction graphs.
All models use $\alpha_i=\alpha=0.2$.}
    \label{Appfig:HeteToy_samples_0.2}
\end{figure}

Across all interaction graphs, increasing the synthetic mixing strength $\alpha$ generally leads to a reduction in diversity. When synthetic data dominates training, models increasingly align with one another despite heterogeneous real data access, resulting in homogenization. 

\textbf{System~1} provides a baseline corresponding to nearly isolated self-consuming retraining. In this case, large $\alpha$ can lead to elevated diversity. However, this increase does not reflect meaningful differentiation. Instead, it arises from independent model collapse, where unstable retraining causes different models to drift toward distinct degenerate distributions. This interpretation is supported by the large variance observed across runs and is consistent with the theory, which characterizes stability rather than the semantic quality of diversity.

Interaction topology also plays a decisive role in shaping heterogeneous results. This trend is strongest in densely connected systems such as \textbf{System~4}, where diversity collapses rapidly even for moderate $\alpha$. This behavior is consistent with the theoretical prediction. On the other side, systems with comparable interaction magnitudes can exhibit qualitatively different behaviors. For example, \textbf{System~2} and \textbf{System~4} share similar interaction norms, yet \textbf{System~4} experiences rapid diversity loss. This contrast reflects structural effects that go beyond conditions, which are sufficient but conservative.

The influence of unstable models is most clearly illustrated in \textbf{System~5} and \textbf{System~6}. \textit{Model~C}, which adopts a GAN-based architecture and exhibits weaker baseline stability, participates directly in cycles in \textbf{System~5}. As a result, synthetic errors generated by \textit{Model~C} are repeatedly amplified and propagated, leading to elevated diversity accompanied by large variance. This behavior reflects system-level instability rather than stable heterogeneity. In contrast, systems where the same model is weakly coupled or excluded from dominant cycles show significantly attenuated effects. These observations align with the theoretical prediction that instability introduced by a single node can dominate global behavior when embedded in dominant interaction cycles.
\newpage
\section{Additional Results on CIFAR-10 Datasets}
We conduct all the experiments on a server which has two Intel Xeon 6326 CPU and 6 Nvidia A6000 GPU. 

While FID summarizes overall distributional similarity, it does not distinguish whether degradation comes primarily from reduced sample fidelity or from loss of coverage. 
We therefore report precision and recall, together with qualitative samples, to better understand how networked self-consuming retraining affects both image quality and distributional support.

\label{sec:addrealex}
\begin{table}[thbp!]
\vspace{-5pt}
\begin{center}
\small
\caption{Model architectures and data access.}
\label{subtab:model_summary_real}
\begin{tabular}{c c c c c c }
\toprule
$\#$ &
Model &
Variants &  Stability  &
Diversity & 
Retrain Steps\\
\midrule
A & DDPM & VLB & class $\{0,\cdots,9\}$ & class $\{0,\cdots,9\}$ & $100$\\
B & CFM  & OT-CFM & class $\{0,\cdots,9\}$       & class $\{3,4\}$ & $600$ \\
C & CFM  & iCFM  & class $\{0,\cdots,9\}$     & class $\{5,6\}$ & $600$ \\
D & DDPM & hybrid & class $\{0,\cdots,9\}$ & class $\{7,8,9\}$ & $100$ \\
\bottomrule
\end{tabular}
\end{center}
\end{table}

\textbf{Details in models.}
For real-data experiments, we use four generative models from two paradigms, with two variants per paradigm. Using multiple variants within the same broad paradigm allows us to separate paradigm-level effects from objective-level effects.  For example, two diffusion models may share the same denoising backbone but respond differently to synthetic-data accumulation because their training objectives impose different likelihood and regression biases. 

For \textbf{DDPM variants},
we include two diffusion models based on the improved diffusion framework with a UNet backbone.
The first variant is a \emph{VLB} diffusion model, which is trained using a variational lower bound objective with a learned noise schedule and likelihood-based parameterization.
The second variant is a \emph{Hybrid} diffusion model, which combines the likelihood-style objective with a simplified regression loss used in practical diffusion training.
While both variants share the same diffusion forward process and network backbone, they differ in their training objectives and, consequently, in their optimization behavior and sample quality trade-offs. 

For \textbf{CFM variants},
we include two flow-matching models implemented with neural ODE parameterizations.
The first variant follows \emph{OT-CFM}, where the training pairs are constructed using an optimal-transport-inspired coupling, yielding flow trajectories that encourage coherent mass transport between source and target distributions.
The second variant follows \emph{iCFM}, which uses an implicit coupling and a simplified construction of flow-matching targets, leading to a different inductive bias on the learned transport dynamics.
Both variants optimize flow-matching objectives but differ in how the intermediate training targets are defined and how probability mass is coupled across time.

We initialize all models from publicly released pretrained checkpoints. To ensure a fair comparison, we follow the original training configurations for each model and modify only the number of training steps per round. Dataset specifications and step counts are provided in~\Cref{subtab:model_summary_real}.

\paragraph{Additional Results for Stability.}  In addition to FID, we report precision as shown in Fig.~\ref{AppFig:Homo_pre} and recall curves as shown in Fig.~\ref{AppFig:Homo_recall} to provide a more fine-grained characterization of generation quality. These two metrics are particularly useful for diagnosing different failure modes.  A drop in precision suggests that generated images become less realistic or less aligned with the data manifold, whereas a drop in recall indicates that the model covers fewer modes or classes of the target distribution. 

\begin{figure}[thbp!]
    \centering
    \begin{subfigure}[c]{0.08\linewidth}
        \centering
    \includegraphics[width=\linewidth]{Figure/W/W1.png}
    \end{subfigure}
    \hfill
    \begin{subfigure}[c]{0.89\linewidth}
        \centering
        \includegraphics[width=\linewidth]{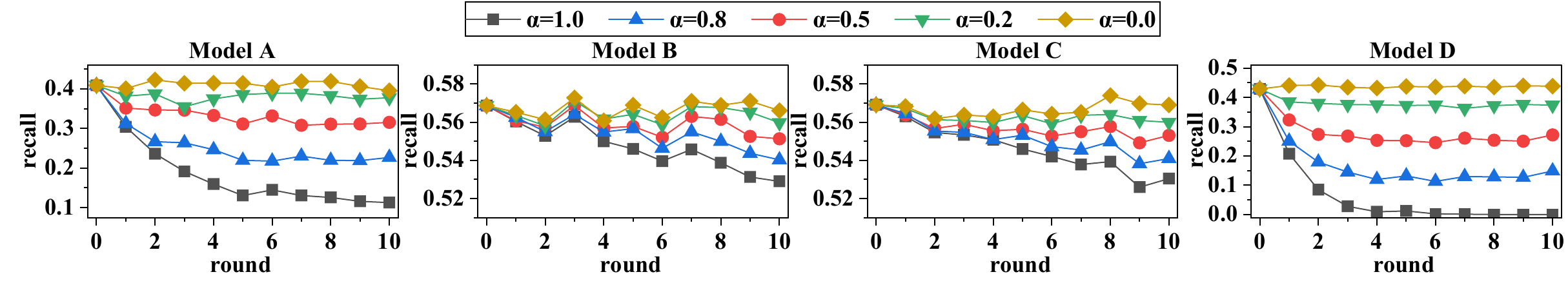}
    \end{subfigure}

    \begin{subfigure}[c]{0.08\linewidth}
        \centering
    \includegraphics[width=\linewidth]{Figure/W/W2.png}
    \end{subfigure}
    \hfill
    \begin{subfigure}[c]{0.89\linewidth}
        \centering
    \includegraphics[width=\linewidth]{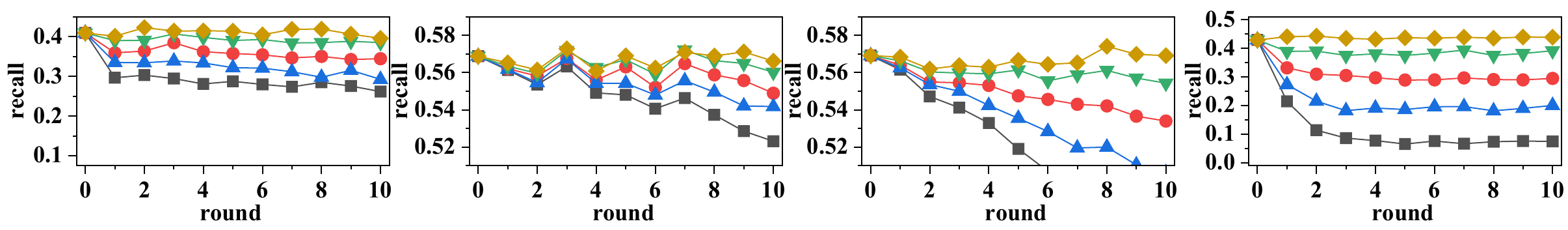}
    \end{subfigure}

    \begin{subfigure}[c]{0.08\linewidth}
        \centering
    \includegraphics[width=\linewidth]{Figure/W/W3.png}
    \end{subfigure}
    \hfill
    \begin{subfigure}[c]{0.89\linewidth}
        \centering
    \includegraphics[width=\linewidth]{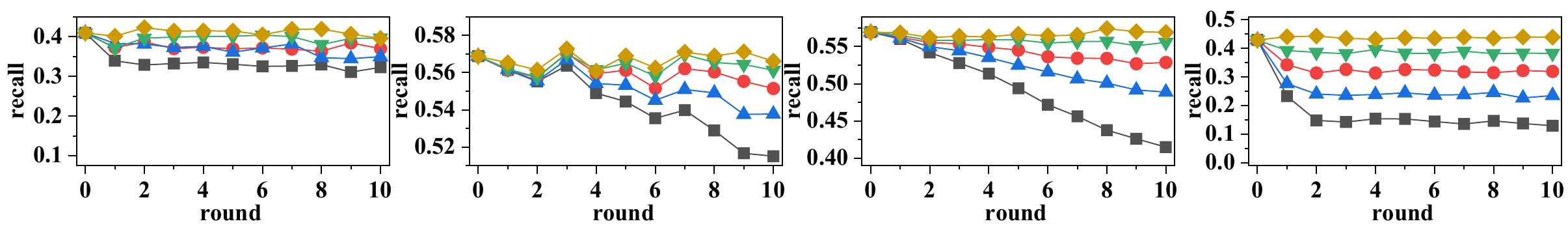}
    \end{subfigure}

    \begin{subfigure}[c]{0.08\linewidth}
        \centering
        \includegraphics[width=\linewidth]{Figure/W/HW4.png}
    \end{subfigure}
    \hfill
    \begin{subfigure}[c]{0.89\linewidth}
        \centering
        \includegraphics[width=\linewidth]{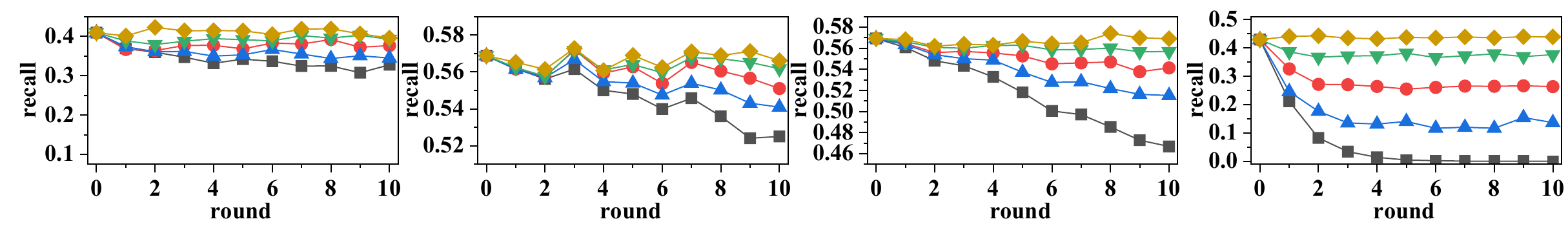}
    \end{subfigure}

    \caption{Stability experiments when all models trained on full dataset. Each row corresponds to a different interaction graph. For each row, the left figure illustrates the direction and strength of synthetic data flow, while the right figures report the recall between generated samples and real data for each model. Results are shown under different synthetic mixing strengths, with $\alpha_i=\alpha$ shared across models.}
    \label{AppFig:Homo_recall}
\end{figure}

\begin{figure}[htbp]
    \centering
    \begin{subfigure}[c]{0.08\linewidth}
        \centering
    \includegraphics[width=\linewidth]{Figure/W/W1.png}
    \end{subfigure}
    \hfill
    \begin{subfigure}[c]{0.89\linewidth}
        \centering
        \includegraphics[width=\linewidth]{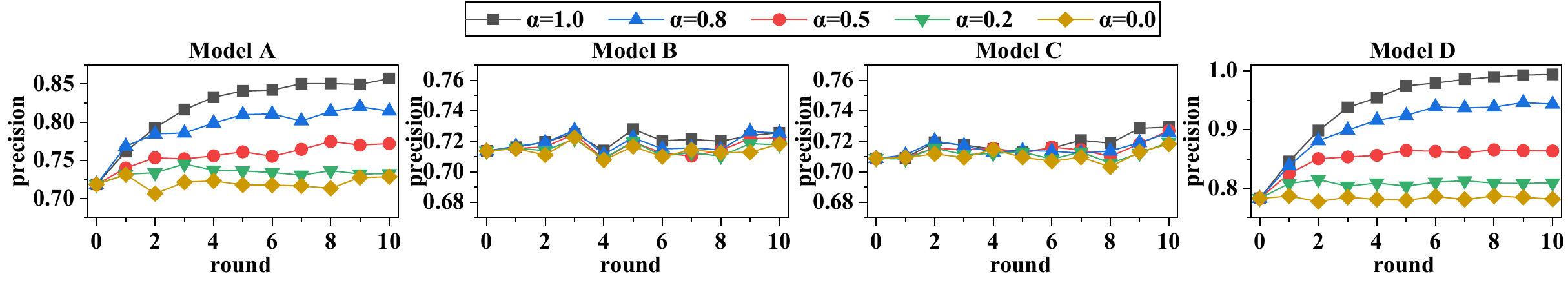}
    \end{subfigure}

    \begin{subfigure}[c]{0.08\linewidth}
        \centering
    \includegraphics[width=\linewidth]{Figure/W/W2.png}
    \end{subfigure}
    \hfill
    \begin{subfigure}[c]{0.89\linewidth}
        \centering
    \includegraphics[width=\linewidth]{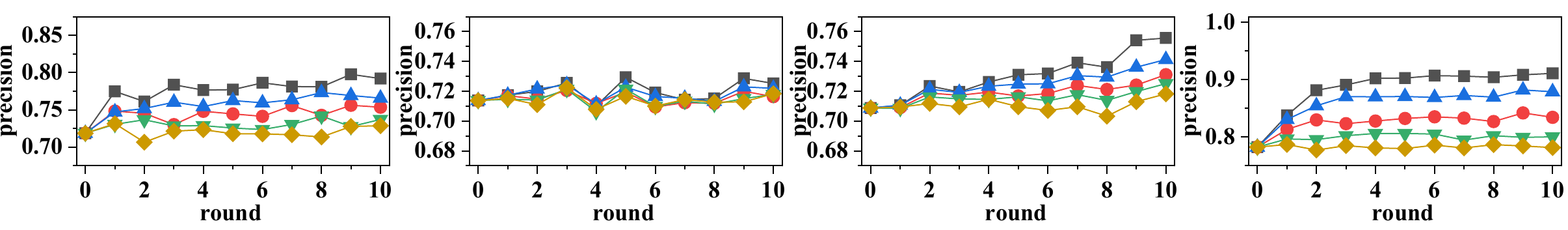}
    \end{subfigure}

    \begin{subfigure}[c]{0.08\linewidth}
        \centering
    \includegraphics[width=\linewidth]{Figure/W/W3.png}
    \end{subfigure}
    \hfill
    \begin{subfigure}[c]{0.89\linewidth}
        \centering
    \includegraphics[width=\linewidth]{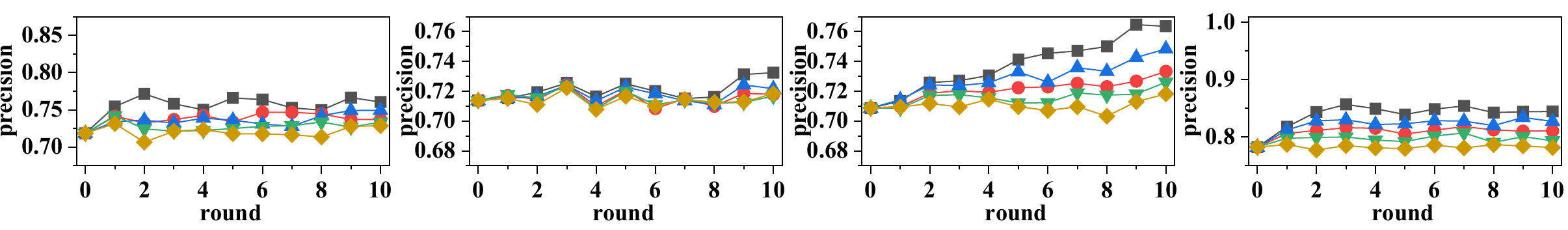}
    \end{subfigure}

    \begin{subfigure}[c]{0.08\linewidth}
        \centering
        \includegraphics[width=\linewidth]{Figure/W/HW4.png}
    \end{subfigure}
    \hfill
    \begin{subfigure}[c]{0.89\linewidth}
        \centering
        \includegraphics[width=\linewidth]{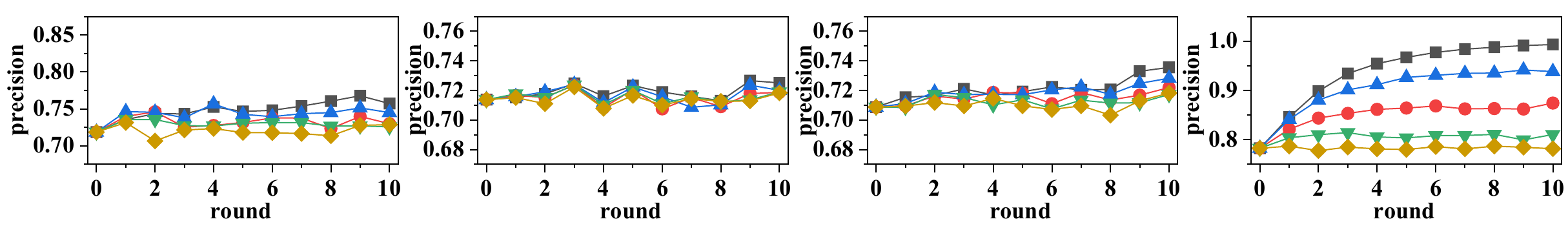}
    \end{subfigure}

    \caption{Stability experiments when all models trained on full dataset. Each row corresponds to a different interaction graph. For each row, the left figure illustrates the direction and strength of synthetic data flow, while the right figures report the precision between generated samples and real data for each model. Results are shown under different synthetic mixing strengths, with $\alpha_i=\alpha$ shared across models.}
    \label{AppFig:Homo_pre}
\end{figure}

\paragraph{Additional Results Diveristy.}
We additionally visualize generated samples at $r=10$ under $\alpha=1.0$ as shown in Fig.~\ref{Appfig:Hete_samples_1.0} and $\alpha=0.2$ as shown in Fig.~\ref{Appfig:Hete_samples_0.2}.
These qualitative results illustrate how different levels of synthetic mixing affect sample quality in practice.

\begin{figure}[htbp]
    \centering
     \begin{subfigure}[c]{0.1\linewidth}
        \centering
    \end{subfigure}
    \hfill
    \begin{subfigure}[c]{0.21\linewidth}
        \centering \textbf{Model~A}
    \end{subfigure}
    \begin{subfigure}[c]{0.21\linewidth}
        \centering \textbf{Model~B}
    \end{subfigure}
    \begin{subfigure}[c]{0.21\linewidth}
        \centering \textbf{Model~C}
    \end{subfigure}
    \begin{subfigure}[c]{0.21\linewidth}
        \centering \textbf{Model~D}
    \end{subfigure}

    \begin{subfigure}[c]{0.1\linewidth}
        \centering \textbf{Initiation}
    \end{subfigure}
    \hfill
    \begin{subfigure}[c]{0.21\linewidth}
        \centering
        \includegraphics[width=\linewidth]{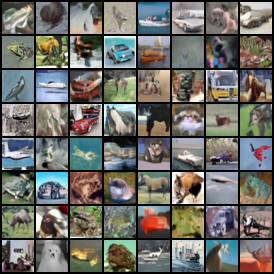}
    \end{subfigure}
    \begin{subfigure}[c]{0.21\linewidth}
        \centering
        \includegraphics[width=\linewidth]{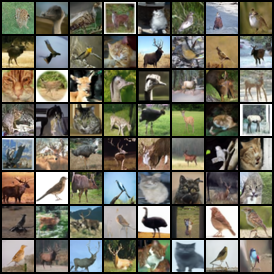}
    \end{subfigure}
    \begin{subfigure}[c]{0.21\linewidth}
        \centering
        \includegraphics[width=\linewidth]{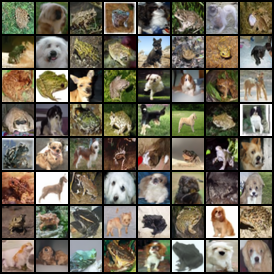}
    \end{subfigure}
    \begin{subfigure}[c]{0.21\linewidth}
        \centering
        \includegraphics[width=\linewidth]{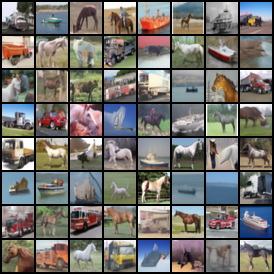}
    \end{subfigure}
    
    \begin{subfigure}[c]{0.1\linewidth}
        \centering
        \includegraphics[width=\linewidth]{Figure/W/W1.png}
    \end{subfigure}  
    \hfill
    \begin{subfigure}[c]{0.21\linewidth}
        \centering
        \includegraphics[width=\linewidth]{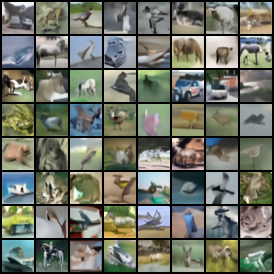}
    \end{subfigure}
    \begin{subfigure}[c]{0.21\linewidth}
        \centering
        \includegraphics[width=\linewidth]{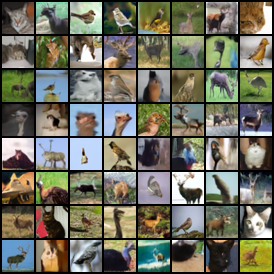}
    \end{subfigure}
    \begin{subfigure}[c]{0.21\linewidth}
        \centering
        \includegraphics[width=\linewidth]{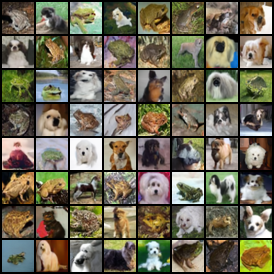}
    \end{subfigure}
    \begin{subfigure}[c]{0.21\linewidth}
        \centering
        \includegraphics[width=\linewidth]{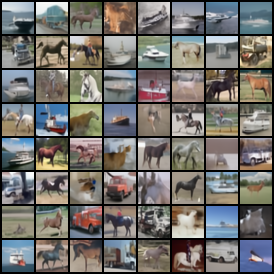}
    \end{subfigure}

    \begin{subfigure}[c]{0.1\linewidth}
        \centering
        \includegraphics[width=\linewidth]{Figure/W/W2.png}
    \end{subfigure}
    \hfill
        \begin{subfigure}[c]{0.21\linewidth}
        \centering
        \includegraphics[width=\linewidth]{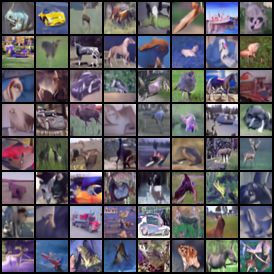}
    \end{subfigure}
    \begin{subfigure}[c]{0.21\linewidth}
        \centering
        \includegraphics[width=\linewidth]{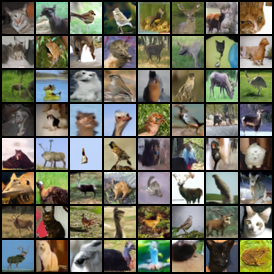}
    \end{subfigure}
    \begin{subfigure}[c]{0.21\linewidth}
        \centering
        \includegraphics[width=\linewidth]{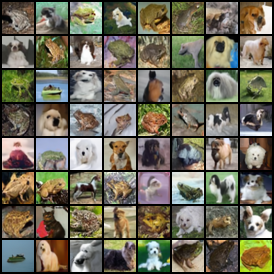}
    \end{subfigure}
    \begin{subfigure}[c]{0.21\linewidth}
        \centering
        \includegraphics[width=\linewidth]{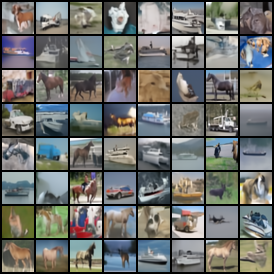}
    \end{subfigure}

    \begin{subfigure}[c]{0.1\linewidth}
        \centering
        \includegraphics[width=\linewidth]{Figure/W/W3.png}
    \end{subfigure}
    \hfill
        \begin{subfigure}[c]{0.21\linewidth}
        \centering
        \includegraphics[width=\linewidth]{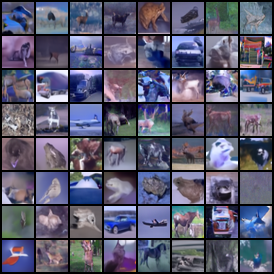}
    \end{subfigure}
    \begin{subfigure}[c]{0.21\linewidth}
        \centering
        \includegraphics[width=\linewidth]{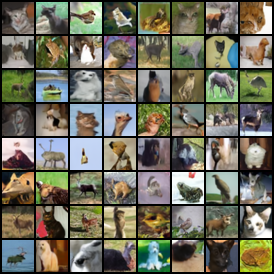}
    \end{subfigure}
    \begin{subfigure}[c]{0.21\linewidth}
        \centering
        \includegraphics[width=\linewidth]{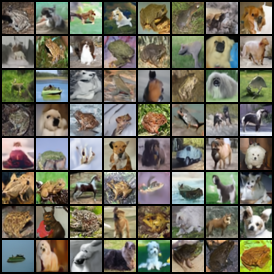}
    \end{subfigure}
    \begin{subfigure}[c]{0.21\linewidth}
        \centering
        \includegraphics[width=\linewidth]{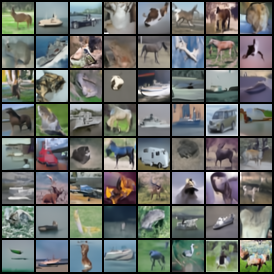}
    \end{subfigure}

    \begin{subfigure}[c]{0.1\linewidth}
        \centering
        \includegraphics[width=\linewidth]{Figure/W/HW4.png}
    \end{subfigure}
    \hfill
        \begin{subfigure}[c]{0.21\linewidth}
        \centering
        \includegraphics[width=\linewidth]{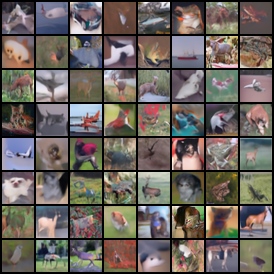}
    \end{subfigure}
    \begin{subfigure}[c]{0.21\linewidth}
        \centering
        \includegraphics[width=\linewidth]{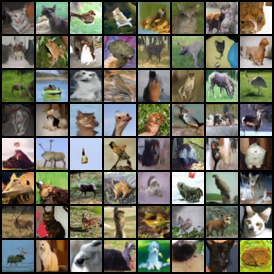}
    \end{subfigure}
    \begin{subfigure}[c]{0.21\linewidth}
        \centering
        \includegraphics[width=\linewidth]{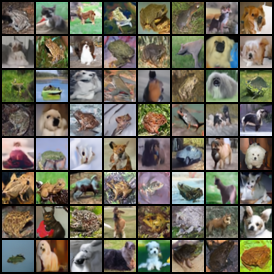}
    \end{subfigure}
    \begin{subfigure}[c]{0.21\linewidth}
        \centering
        \includegraphics[width=\linewidth]{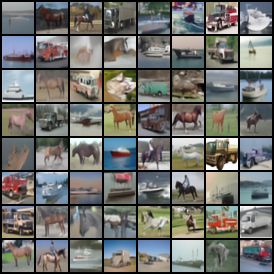}
    \end{subfigure}

    \caption{Generated samples under the heterogeneous real data setting with different interaction graphs.
All models use $\alpha_i=\alpha=1.0$. }
    \label{Appfig:Hete_samples_1.0}
\end{figure}

\begin{figure}[htbp]
    \centering
     \begin{subfigure}[c]{0.1\linewidth}
        \centering
    \end{subfigure}
    \hfill
    \begin{subfigure}[c]{0.21\linewidth}
        \centering \textbf{Model~A}
    \end{subfigure}
    \begin{subfigure}[c]{0.21\linewidth}
        \centering \textbf{Model~B}
    \end{subfigure}
    \begin{subfigure}[c]{0.21\linewidth}
        \centering \textbf{Model~C}
    \end{subfigure}
    \begin{subfigure}[c]{0.21\linewidth}
        \centering \textbf{Model~D}
    \end{subfigure}
    
        \begin{subfigure}[c]{0.1\linewidth}
        \centering \textbf{Initiation}
    \end{subfigure}
    \hfill
    \begin{subfigure}[c]{0.21\linewidth}
        \centering
        \includegraphics[width=\linewidth]{Figure/Homo/Sample/0_1.png}
    \end{subfigure}
    \begin{subfigure}[c]{0.21\linewidth}
        \centering
        \includegraphics[width=\linewidth]{Figure/Homo/Sample/0_2.png}
    \end{subfigure}
    \begin{subfigure}[c]{0.21\linewidth}
        \centering
        \includegraphics[width=\linewidth]{Figure/Homo/Sample/0_3.png}
    \end{subfigure}
    \begin{subfigure}[c]{0.21\linewidth}
        \centering
        \includegraphics[width=\linewidth]{Figure/Homo/Sample/0_4.png}
    \end{subfigure}
    
    \begin{subfigure}[c]{0.1\linewidth}
        \centering
        \includegraphics[width=\linewidth]{Figure/W/W1.png}
    \end{subfigure}  
    \hfill
    \begin{subfigure}[c]{0.21\linewidth}
        \centering
        \includegraphics[width=\linewidth]{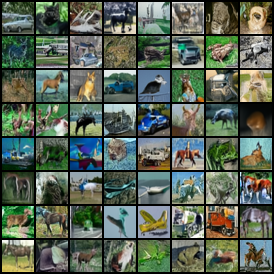}
    \end{subfigure}
    \begin{subfigure}[c]{0.21\linewidth}
        \centering
        \includegraphics[width=\linewidth]{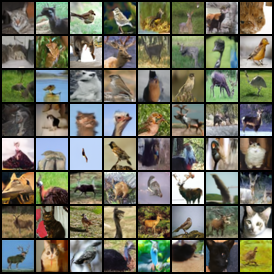}
    \end{subfigure}
    \begin{subfigure}[c]{0.21\linewidth}
        \centering
        \includegraphics[width=\linewidth]{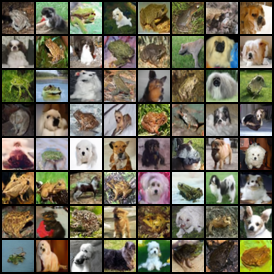}
    \end{subfigure}
    \begin{subfigure}[c]{0.21\linewidth}
        \centering
        \includegraphics[width=\linewidth]{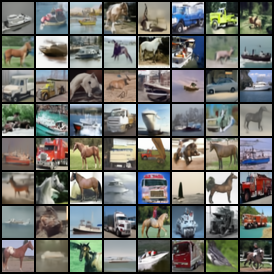}
    \end{subfigure}

    \begin{subfigure}[c]{0.1\linewidth}
        \centering
        \includegraphics[width=\linewidth]{Figure/W/W2.png}
    \end{subfigure}
    \hfill
        \begin{subfigure}[c]{0.21\linewidth}
        \centering
        \includegraphics[width=\linewidth]{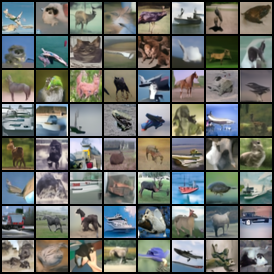}
    \end{subfigure}
    \begin{subfigure}[c]{0.21\linewidth}
        \centering
        \includegraphics[width=\linewidth]{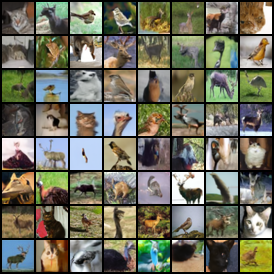}
    \end{subfigure}
    \begin{subfigure}[c]{0.21\linewidth}
        \centering
        \includegraphics[width=\linewidth]{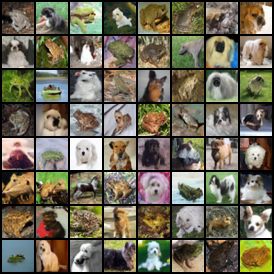}
    \end{subfigure}
    \begin{subfigure}[c]{0.21\linewidth}
        \centering
        \includegraphics[width=\linewidth]{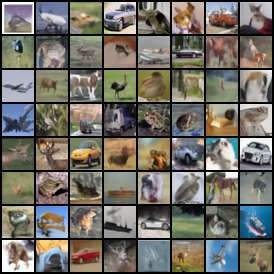}
    \end{subfigure}

    \begin{subfigure}[c]{0.1\linewidth}
        \centering
        \includegraphics[width=\linewidth]{Figure/W/W3.png}
    \end{subfigure}
    \hfill
        \begin{subfigure}[c]{0.21\linewidth}
        \centering
        \includegraphics[width=\linewidth]{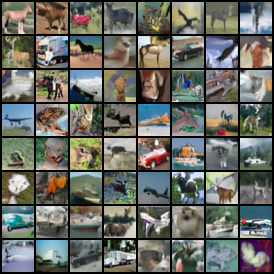}
    \end{subfigure}
    \begin{subfigure}[c]{0.21\linewidth}
        \centering
        \includegraphics[width=\linewidth]{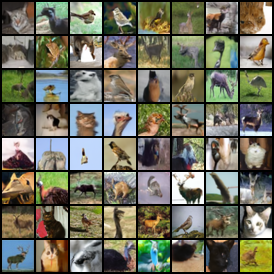}
    \end{subfigure}
    \begin{subfigure}[c]{0.21\linewidth}
        \centering
        \includegraphics[width=\linewidth]{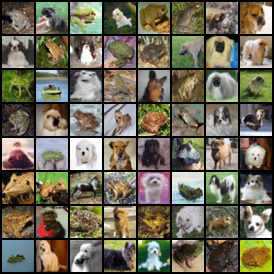}
    \end{subfigure}
    \begin{subfigure}[c]{0.21\linewidth}
        \centering
        \includegraphics[width=\linewidth]{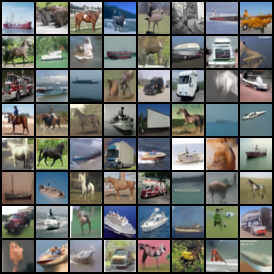}
    \end{subfigure}

    \begin{subfigure}[c]{0.1\linewidth}
        \centering
        \includegraphics[width=\linewidth]{Figure/W/HW4.png}
    \end{subfigure}
    \hfill
        \begin{subfigure}[c]{0.21\linewidth}
        \centering
        \includegraphics[width=\linewidth]{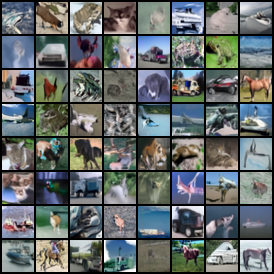}
    \end{subfigure}
    \begin{subfigure}[c]{0.21\linewidth}
        \centering
        \includegraphics[width=\linewidth]{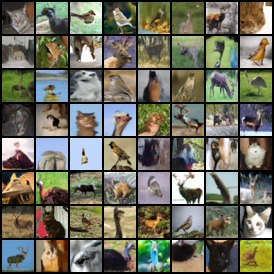}
    \end{subfigure}
    \begin{subfigure}[c]{0.21\linewidth}
        \centering
        \includegraphics[width=\linewidth]{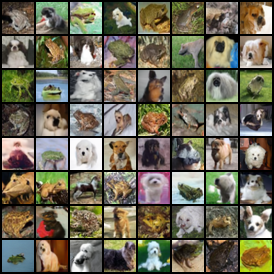}
    \end{subfigure}
    \begin{subfigure}[c]{0.21\linewidth}
        \centering
        \includegraphics[width=\linewidth]{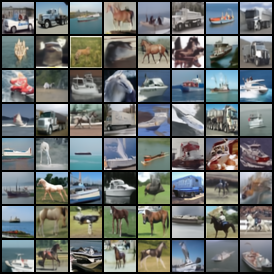}
    \end{subfigure}

    \caption{Generated samples under the heterogeneous real data setting with different interaction graphs.
All models use $\alpha_i=\alpha=0.2$.}
    \label{Appfig:Hete_samples_0.2}
\end{figure}



\end{document}